\documentclass{tlp}

\usepackage[leftbars]{changebar}
	{\cbend\end{list}\bigskip}%

\usepackage{amsmath,amssymb}
\usepackage{braket}
\usepackage{url}
\usepackage{srcltx}
\usepackage{times}

\usepackage{comment}

\newenvironment{normal}{}{}

\newenvironment{extended}{}{}

\excludecomment{reviews}
\excludecomment{normal}
\newtheorem{theorem}{Theorem}
\newtheorem{proposition}[theorem]{Proposition}
\newtheorem{lemma}[theorem]{Lemma}
\newtheorem{corollary}[theorem]{Corollary}
\newtheorem{definition}[theorem]{Definition}
\newtheorem{remark}[theorem]{Remark}
\newtheorem{example}[theorem]{Example}

\newenvironment{corollary*}[2][]
	{\medskip \par \noindent \textit{Corollary \ref{#2} (#1)} \par\noindent\it}
	{\medskip}

\newenvironment{proposition*}[2][]
	{\medskip \par \noindent \textit{Proposition \ref{#2} (#1)} \par\noindent\it}
	{\medskip}

\newenvironment{theorem*}[2][]
	{\medskip \par \noindent \textit{Theorem \ref{#2} (#1)} \par\noindent\it}
	{\medskip}

\newcommand{\qedhere}{\hspace*{1em}\hbox{\proofbox}}

\newcommand{\br}[1]{\left( #1 \right)}
\newcommand{\an}[1]{\left\langle #1 \right\rangle}
\newcommand{\mlthen}{\Longrightarrow}
\newcommand{\mlequiv}{\Longleftrightarrow}

\newcommand{\lpre}{\mathbf{P}}
\newcommand{\preds}[1]{\mathsf{pr}(#1)}
\newcommand{\predsP}[1]{\mathsf{pr}\left(#1\right)}

\newcommand{\bigland}{\bigwedge}
\newcommand{\lif}{\subset}
\newcommand{\mymodels}{\mathrel\mid\joinrel=}
\newcommand{\ent}{\mymodels}
\newcommand{\nent}{\not\ent}

\newcommand{\smod}{\mathsf{mod}}

\newcommand{\mk}[1][\,]{\mathbf{K}#1}
\newcommand{\mnot}[1][\,]{\mathbf{not}#1}
\newcommand{\mstr}[1][I, M, N]{\left\langle #1 \right\rangle}
\newcommand{\restr}[2][U]{#2^{\left[#1\right]}}
\newcommand{\restrc}[2][U]{#2^{\left[\lpre \setminus #1\right]}}

\newcommand{\satf}{\sigma}
\newcommand{\sat}[2][U]{\satf(#2, #1)}
\newcommand{\satpar}[2][U]{\satf\left(#2, #1\right)}
\newcommand{\satc}[2][U]{\satf(#2, \lpre \setminus #1)}

\newcommand{\foint}{\mathcal{I}}
\newcommand{\mint}{\mathcal{M}}

\newcommand{\kb}[1][K]{\mathcal{#1}}
\newcommand{\thr}[1][T]{\mathcal{#1}}
\newcommand{\upd}[1][U]{\mathcal{#1}}
\newcommand{\ont}[1][O]{\mathcal{#1}}

\newcommand{\lpnot}[1][\!\!]{\sim#1}
\newcommand{\lpif}{\leftarrow}
\newcommand{\prog}[1][\defprog]{\mathcal{#1}}

\newcommand{\least}[1]{\mathit{least}(#1)}
\newcommand{\dynall}[1][\dynprog]{\rho(#1)}
\newcommand{\dynrej}[2][]{\mathit{Rej}_{#1}(#2)}
\newcommand{\dyndef}[2][]{\mathit{Def}_{#1}(#2)}

\newcommand{\SE}[1]{\textsf{SE}\protect\nobreakdash#1\hspace{0pt}}
\newcommand{\dlfo}[1]{\zeta\br{#1}}
\newcommand{\longline}{\noindent\rule{\textwidth}{.01in}}

\begin{document}

\title[Splitting and Updating Hybrid Knowledge Bases (Extended
Version)]{Splitting and Updating
Hybrid Knowledge Bases (Extended Version)}

\author[M. Slota and J. Leite and T. Swift]
{MARTIN SLOTA$^{\footnotemark[1]}$ and JO{\~A}O LEITE$^{\footnotemark[2]}$ and TERRANCE SWIFT \\
CENTRIA \& Departamento de Inform{\'a}tica \\
Universidade Nova de Lisboa \\
2829-516 Caparica, Portugal}

\renewcommand{\thefootnote}{\fnsymbol{footnote}}
\footnotetext[1]{Supported by FCT Scholarship SFRH~/~BD~/~38214~/~2007}
\footnotetext[2]{Partially supported by FCT Project ASPEN
PTDC~/~EIA-CCO~/~110921~/~2009}
\setcounter{footnote}{0}
\renewcommand{\thefootnote}{\arabic{footnote}}

\maketitle

\begin{abstract}
	Over the years, nonmonotonic rules have proven to be a very expressive and
	useful knowledge representation paradigm. They have recently been used to
	complement the expressive power of Description Logics (DLs), leading to the
	study of integrative formal frameworks, generally referred to as
	\emph{hybrid knowledge bases}, where both DL axioms and rules can be used to
	represent knowledge. The need to use these hybrid knowledge bases in dynamic
	domains has called for the development of update operators, which, given the
	substantially different way Description Logics and rules are usually
	updated, has turned out to be an extremely difficult task.

	In \cite{Slota2010a}, a first step towards addressing this problem was
	taken, and an update operator for hybrid knowledge bases was proposed.
	Despite its significance -- not only for being the first update operator for
	hybrid knowledge bases in the literature, but also because it has some
	applications -- this operator was defined for a restricted class of problems
	where only the ABox was allowed to change, which considerably diminished its
	applicability. Many applications that use hybrid knowledge bases in dynamic
	scenarios require both DL axioms and rules to be updated.

	In this paper, motivated by real world applications, we introduce an update
	operator for a large class of hybrid knowledge bases where both the DL
	component as well as the rule component are allowed to dynamically change.
	We introduce splitting sequences and splitting theorem for hybrid knowledge
	bases, use them to define a modular update semantics, investigate its basic
	properties, and illustrate its use on a realistic example about cargo
	imports.
\end{abstract}


\section{Introduction}

Increasingly many real world applications need to intelligently access and
reason with large amounts of dynamically changing, structured and highly
interconnected information. The family of Description Logics (DLs)
\cite{Baader2003}, generally characterised as decidable fragments of
first-order logic, have become the established standard for representing
\emph{ontologies}, i.e. for specifying concepts relevant to a particular
domain of interest. DLs can be seen as some of the most expressive formalisms
based on Classical Logic for which decidable reasoning procedures still exist.

On the other hand, nonmonotonic rules have also proven to be a very useful
tool for knowledge representation. They complement the expressive power of
DLs, adding the possibility to reason with incomplete information using
default negation, and offering natural ways of expressing exceptions,
integrity constraints and complex queries. Their formal underpinning lies with
declarative, well-understood semantics, the stable model semantics
\cite{Gelfond1988} and its tractable approximation, the well-founded semantics
\cite{Gelder1991}, being the most prominent and widely accepted.

This has led to the need to integrate these distinct knowledge representation
paradigms. Over the last decade, there have been many proposals for
integrating DLs with nonmonotonic rules (see \cite{Hitzler2009} for a survey).
One of the more mature proposals is Hybrid MKNF Knowledge Bases
\cite{Motik2007} that allow predicates to be defined concurrently in both an
ontology and a set of rules, while enjoying several important properties. A
tractable variant of this formalism, based on the well-founded semantics,
allows for a top-down querying procedure \cite{Knorr2011}, making the approach
amenable to practical applications that need to deal with large knowledge
bases.

While such formalisms make it possible to seamlessly combine rules and
ontologies in a single unified framework, they do not take into account the
highly dynamic character of application areas where they are to be used. In
\cite{Slota2010a} we made a first step towards a solution to this problem,
addressing \emph{updates} by defining a change operation on a knowledge base
to record a change that occurred in the modelled world.

Update operators have first been studied in the context of action theories and
relational databases with NULL values \cite{Winslett1988,Winslett1990}. The
basic intuition behind these operators is that the models of a knowledge base
represent possible states of the world and when a change in the world needs to
be recorded, each of these possible worlds should be modified as little as
possible in order to arrive at a representation of the world after the update.
This means that, in each possible world, each propositional atom retains its
truth value as long as there is no update that directly requires it to change.
In other words, \emph{inertia} is applied to the atoms of the underlying
language. Later, these operators were successfully applied to partially
address updates of DL ontologies \cite{Liu2006,Giacomo2006}.

% Typically, such an ontology is divided in two parts, the
% definition of domain terminology, dubbed the \emph{TBox}, and assertions about
% individuals with respect to the defined terminology, dubbed the \emph{ABox}.
% the very same update operators have successfully been applied to DL
% ontologies, under the assumption that axioms defining concepts of the domain
% (so called TBox) stay static, and assertions about individuals
 
But when updates were studied in the context of rules, most authors found atom
inertia unsatisfactory. One of the main reasons for this is the clash between
atom inertia and the property of \emph{support} \cite{Apt1988,Dix1995a}, which
lies at the heart of most logic programming semantics. For instance, when
updating a logic program $\prog = \Set{p \lpif q., q.}$ by $\prog[Q] =
\Set{\lpnot q.}$,\footnote{The symbol $\lpnot[]$ denotes default negation.}
atom inertia dictates that $p$ must stay true after the update because the
update itself does not in any way directly affect the truth value of $p$.
Example 7 in \cite{Giacomo2006}, where a similar update is performed on an
analogical DL ontology, shows that such a behaviour may be desirable. However,
from a logic programming point of view, one expects $p$ to become false after
the update. This is because when $q$ ceases being true, the reason for $p$ to
be true disappears as it is no longer supported by any rule.

These intuitions, together with a battery of intuitive examples
\cite{Leite1997,Alferes2000}, led to the introduction of the \emph{causal
rejection principle} \cite{Leite1997} and subsequently to several approaches
to rule updates \cite{Alferes2000,Eiter2002,Leite2003,Alferes2005} that are
fundamentally different from classical update operators. The basic unit of
information to which inertia is applied is no longer an atom, but a rule. This
means that a rule stays in effect as long as it does not directly contradict a
newer rule. The truth values of atoms are not directly subject to inertia, but
are used to determine the set of rules that are overridden by newer rules, and
are themselves determined by the remaining rules.

However, the dichotomy between classical and rule updates goes far beyond the
different units of information to which inertia is applied. While classical
updates are performed on the models of a knowledge base, which renders them
syntax-independent, the property of support, being syntactic in its essence,
forces rule update methods to refer to the syntactic structure of underlying
programs -- the individual rules they contain and, in many cases, also the
heads and bodies of these rules. As we have shown in \cite{Slota2010b}, even
when classical updates are applied to \SE-models
\cite{Lifschitz2001,Turner2003}, a monotonic semantics for logic programs that
is more expressive than stable models, the property of support is lost. On the
other hand, applying rule updates to DL ontologies leads to a range of
technical difficulties. Some of them are caused by the fact that rule update
methods are specifically tailored towards identifying and resolving conflicts
between pairs of rules. DL axioms do not have a rule-like structure, and a
group of pairwise consistent axioms may enter in a conflict. Other
difficulties stem from the fact that such a syntactic approach can hardly
exhibit behaviour similar to that of classical updates, where reasoning by
cases is inherent in updating each model of a knowledge base independently of
all others. Thus, no single method seems suitable for updating hybrid
knowledge bases. A general update operator for hybrid knowledge must somehow
integrate these apparently irreconcilable approaches to dealing with evolving
knowledge.

In \cite{Slota2010a} we simplified this hard task by keeping rules static and
allowing the ontology component of a hybrid knowledge base to evolve. Despite
the importance of this first step, the applicability of the operator is
considerably diminished since typically all parts of a knowledge base are
subject to change. As an example, consider the following scenario where both
ontologies and rules are needed to assess the risk of imported cargo.

\begin{example}[A Hybrid Knowledge Base for Cargo Imports]
	\label{ex:import:description}
	% In order to illustrate the complexity of a decision support system used by a
	% large enterprise, we consider a system for assessing tariffs and possible
	% risks of imported cargo
	The Customs service for any developed country assesses imported cargo for a
	variety of risk factors including terrorism, narcotics, food and consumer
	safety, pest infestation, tariff violations, and intellectual property
	rights.\footnote{The system described here is not intended to reflect the
	policies of any country or agency.} Assessing this risk, even at a
	preliminary level, involves extensive knowledge about commodities, business
	entities, trade patterns, government policies and trade agreements. Some of
	this knowledge may be external to a given customs agency: for instance the
	broad classification of commodities according to the international
	Harmonized Tariff System (HTS), or international trade agreements. Other
	knowledge may be internal to a customs agency, such as lists of suspected
	violators or of importers who have a history of good compliance with
	regulations. While some of this knowledge is relatively stable, much of it
	changes rapidly. Changes are made not only at a specific level, such as
	knowledge about the expected arrival date of a shipment; but at a more
	general level as well. For instance, while the broad HTS code for tomatoes
	(0702) does not change, the full classification and tariffs for cherry
	tomatoes for import into the US changes seasonally.

	% These considerations suggest ways in which that knowledge about imports
	% may be represented.  Ontologies can be useful to represent knowledge that
	% arises from diverse sources because the frame-based visualization of
	% hierarchies and binary relations is easily understood by many people, and
	% because the FOL basis leads to powerful mechanisms to check consistency.
	% At the same time, rules are useful for when default negation is needed,
	% when procedural knowledge is required, (e.g. knowledge about probability)
	% or when large-arity relations need to be represented.  
	%
	% In addition, enterprise-scale applications must have large modules that
	% are low-complexity: for instance the US, has on the order of 100 million
	% imports per year.

	Figure~\ref{fig:import} shows a simplified fragment $\kb = \an{\ont, \prog}$
	of such a knowledge base. In this fragment, a shipment has several
	attributes: the country of its origination, the commodity it contains, its
	importer, and its producer. The ontology contains a geographic
	classification, along with information about producers who are located in
	various countries. It also contains a classification of commodities based on
	their harmonised tariff information (HTS chapters, headings and codes,
	cf.~\url{http://www.usitc.gov/tata/hts}). Tariff information is also
	present, based on the classification of commodities. Finally, the ontology
	contains (partial) information about three shipments: $s_1$, $s_2$ and
	$s_3$. There is also a set of rules indicating information about importers,
	and about whether to inspect a shipment either to check for compliance of
	tariff information or for food safety issues.
\end{example}

\newcommand{\pr}[1]{\ensuremath{\mathsf{#1}}}
\newcommand{\ct}[1]{\ensuremath{\mathit{#1}}}

\begin{figure}
	\longline

	\begin{center}
	*\ \ *\ \ *\ \ $\ont$ *\ \ *\ \ *
	\end{center}
	\vspace{-.4cm}
	\begin{tabbing}
	foooooooooooooooooooooooooooooooooooooooooo\=\kill
	$\pr{Commodity} \equiv (\exists \pr{HTSCode}.\top)$
		\> $\pr{EdibleVegetable} \equiv (\exists \pr{HTSChapter}.\set{\text{`07'}})$ \\
	$\pr{CherryTomato} \equiv (\exists \pr{HTSCode}.\set{\text{`07020020'}})$
		\> $\pr{Tomato} \equiv (\exists \pr{HTSHeading}.\set{\text{`0702'}})$ \\
	$\pr{GrapeTomato} \equiv (\exists \pr{HTSCode}.\set{\text{`07020010'}})$
		\> $\pr{Tomato} \sqsubseteq \pr{EdibleVegetable}$ \\
	$\pr{CherryTomato} \sqsubseteq \pr{Tomato}$
		\> $\pr{GrapeTomato} \sqsubseteq \pr{Tomato}$ \\
	$\pr{CherryTomato} \sqcap \pr{Bulk}
			\equiv (\exists \pr{TariffCharge}.\set{\text{\$0}})$
		\> $\pr{CherryTomato} \sqcap \pr{GrapeTomato} \sqsubseteq \bot$ \\
	$\pr{GrapeTomato} \sqcap \pr{Bulk}
			\equiv (\exists \pr{TariffCharge}.\set{\text{\$40}})$
		\> $\pr{Bulk} \sqcap \pr{Prepackaged} \sqsubseteq \bot$ \\
	$\pr{CherryTomato} \sqcap \pr{Prepackaged}
			\equiv (\exists \pr{TariffCharge}.\set{\text{\$50}})$ \\
	$\pr{GrapeTomato} \sqcap \pr{Prepackaged}
			\equiv (\exists \pr{TariffCharge}.\set{\text{\$100}})$ \\
	$\pr{EURegisteredProducer}
			\equiv (\exists \pr{RegisteredProducer}.\pr{EUCountry})$ \\
	$\pr{LowRiskEUCommodity}
			\equiv (\exists \pr{ExpeditableImporter}.\top)
			\sqcap (\exists \pr{CommodCountry}.\pr{EUCountry})$ \\
	\\
	$\an{\ct{p_1}, \ct{portugal}} : \pr{RegisteredProducer}$
		\> $\an{\ct{p_2}, \ct{slovakia}} : \pr{RegisteredProducer}$ \\
	$\ct{portugal} : \pr{EUCountry}$
		\> $\ct{slovakia} : \pr{EUCountry}$ \\
	$\an{\ct{s_1}, \ct{c_1}} : \pr{ShpmtCommod}$
		\> $\an{\ct{s_1}, \text{`07020020'}} : \pr{ShpmtDeclHTSCode}$ \\
	$\an{\ct{s_1}, \ct{i_1}} : \pr{ShpmtImporter}$
		\> $\ct{c_1} : \pr{CherryTomato} \quad \ct{c_1} : \pr{Bulk}$ \\
	$\an{\ct{s_2}, \ct{c_2}} : \pr{ShpmtCommod}$
		\> $\an{\ct{s_2}, \text{`07020020'}} : \pr{ShpmtDeclHTSCode}$ \\
	$\an{\ct{s_2}, \ct{i_2}} : \pr{ShpmtImporter}$
		\> $\ct{c_2} : \pr{CherryTomato} \quad \ct{c_2} : \pr{Prepackaged}$ \\
	$\an{\ct{s_2}, portugal} : \pr{ShpmtCountry}$
		\> \\
	$\an{\ct{s_3}, \ct{c_3}} : \pr{ShpmtCommod}$    
		\> $\an{\ct{s_3}, \text{`07020010'}} : \pr{ShpmtDeclHTSCode}$ \\
	$\an{\ct{s_3}, \ct{i_3}} : \pr{ShpmtImporter}$
		\> $\ct{c_3} : \pr{GrapeTomato} \quad\, \ct{c_3} : \pr{Bulk}$ \\
	$\an{\ct{s_3}, portugal} : \pr{ShpmtCountry}$
		\> $\an{\ct{s_3}, \ct{p_1}} : \pr{ShpmtProducer}$
	\end{tabbing}

	\begin{center}
	*\ \ *\ \ *\ \ $\prog$ *\ \ *\ \ *
	\end{center}
	\vspace{-.4cm}
	\begin{tabbing}
	foooooooooooooooooooooooooooooooooooooooooo\=\kill
	$\pr{CommodCountry}(C, \ct{Country})
		\lpif \pr{ShpmtCommod}(S, C), \pr{ShpmtCountry}(S, \ct{Country}).$ \\
	$\pr{AdmissibleImporter}(I) \lpif{} \!\lpnot \pr{SuspectedBadGuy}(I).$ \\
	$\pr{ExpeditableImporter}(C, I)
		\lpif \pr{AdmissibleImporter}(I), \pr{ApprovedImporterOf}(I, C).$ \\
	$\pr{SuspectedBadGuy}(\ct{i_1})$. \\
	$\pr{ApprovedImporterOf}(\ct{i_2}, C) \lpif \pr{EdibleVegetable}(C).$ \\
	$\pr{ApprovedImporterOf}(\ct{i_3}, C) \lpif \pr{GrapeTomato}(C).$ \\
	$\pr{CompliantShpmt}(S)
		\lpif \pr{ShpmtCommod}(S, C), \pr{HTSCode}(C, D), \pr{ShpmtDeclHTSCode}(S, D).$ \\
	$\pr{RandomInspection}(S) \lpif \pr{ShpmtCommod}(S, C), \pr{Random}(C).$ \\
	$\pr{PartialInspection}(S) \lpif \pr{RandomInspection}(S).$ \\
	$\pr{PartialInspection}(S)
		\lpif \pr{ShpmtCommod}(S, C), \lpnot \pr{LowRiskEUCommodity}(C).$ \\
	$\pr{FullInspection}(S) \lpif{} \!\lpnot \pr{CompliantShpmt}(S).$ \\
	$\pr{FullInspection}(S)
		\lpif \pr{ShpmtCommod}(S, C), \pr{Tomato}(C), \pr{ShpmtCountry}(S, \ct{slovakia}).$
	\end{tabbing}

	\longline
	\caption{A Hybrid Knowledge Base for Cargo Imports}
	\label{fig:import}
\end{figure}

In this paper, we define an update semantics for hybrid knowledge bases that
can be used to deal with scenarios such as the one described above. As a
theoretical basis for this operator, we first establish a splitting theorem
for Hybrid MKNF Knowledge Bases, analogical to the splitting theorem for logic
programs \cite{Lifschitz1994a}. The underlying notions then serve us as
theoretical ground for identifying a constrained class of hybrid knowledge
bases for which a plausible update semantics can be defined by modularly
combining a classical and a rule update operator. We then examine basic
properties of this semantics, showing that it
\begin{itemize}
	\item generalises Hybrid MKNF Knowledge Bases \cite{Motik2007}.
	\item generalises the classical minimal change update operator
		\cite{Winslett1990}.
	\item generalises the refined dynamic stable model semantics
		\cite{Alferes2005}.
	\item adheres to the principle of primacy of new information
		\cite{Dalal1988}.
%	\item is syntax-independent w.r.t.~the ontology component and its updates.
\end{itemize}
Finally, we demonstrate that it properly deals with nontrivial updates in
scenarios such as the one described in Example \ref{ex:import:description}.

The rest of this document is structured as follows: We introduce the necessary
theoretical background in Sect.~\ref{sect:preliminaries}. Then, in
Sect.~\ref{sect:core}, we establish the splitting theorem for Hybrid MKNF
Knowledge Bases, identify a constrained class of such knowledge bases and
define a plausible update operator for it. We also take a closer look at its
properties and show how it can be applied to deal with updates of the
knowledge base introduced in Example~\ref{ex:import:description}. We then
discuss our results in Sect.~\ref{sect:discussion} and point towards desirable
future developments.\footnote{At
\url{http://centria.di.fct.unl.pt/~jleite/iclp11full.pdf} the reader can find
an extended version of this paper with proofs.}

\section{Preliminaries} \label{sect:preliminaries}

In this section we present the formal basis for our investigation. We
introduce the unifying semantic framework of Hybrid MKNF Knowledge Bases
\cite{Motik2007} that gives a semantics to a knowledge base composed of both
DL axioms and rules. Since our hybrid update operator is based on a modular
combination a classical and a rule update operator, we introduce a pair of
such operators known from the literature and briefly discuss the choices we
make.

% \noindent \textbf{Sequences.}
% A \emph{(transfinite) sequence} is a family whose index set is an initial
% segment of ordinals, $\set{\alpha | \alpha < \mu}$. The ordinal $\mu$ is the
% \emph{length} of the sequence. A sequence $\an{U_\alpha}_{\alpha < \mu}$ of
% sets is \emph{monotone} if $U_\alpha \subseteq U_\beta$ whenever $\alpha \leq
% \beta$, and \emph{continuous} if, for each limit ordinal $\alpha < \mu$,
% $U_\alpha = \bigcup_{\eta < \alpha} U_\eta$. A sequence $\an{U_i}_{i < n}$
% is \emph{finite} if $n < \omega$.

\noindent \textbf{MKNF.}
The logic of Minimal Knowledge and Negation as Failure (MKNF)
\cite{Lifschitz1991} forms the logical basis of Hybrid MKNF Knowledge Bases.
It is an extension of first-order logic with two modal operators: $\mk[]$ and
$\mnot[]$. We use the variant of this logic introduced in \cite{Motik2007}. We
assume a function-free first-order syntax extended by the mentioned modal
operators in a natural way.
\begin{extended}%
An \emph{atom} is a formula $P(t_1, t_2, \dotsc,
t_n)$ where $P$ is a predicate symbol of arity $n$ and $t_i$ are terms. An
MKNF formula $\phi$ is a \emph{sentence} if it has no free variables; $\phi$
is \emph{ground} if it does not contain variables;
$\phi$ is \emph{subjective} if all atoms in $\phi$ occur within
the scope of a modal operator;
$\phi$ is \emph{first-order} if it does not contain modal operators. By
$\phi[t/x]$ we denote the formula obtained by simultaneously replacing in
$\phi$ all free occurrences of variable $x$ by term $t$. A set of first-order
sentences is a \emph{first-order theory}.
\end{extended}
\begin{normal}%
We denote the set of all predicate symbols by $\lpre$ and the set of all
predicate symbols occurring in a formula $\phi$ by $\preds{\phi}$.
\end{normal}
\begin{extended}%
We denote the set of all predicate symbols by $\lpre$. Given a formula $\phi$,
we inductively define the \emph{set of predicate symbols relevant to $\phi$},
denoted by $\preds{\phi}$, as follows:
\begin{enumerate}
	\renewcommand{\labelenumi}{\arabic{enumi}$^\circ$}
	\item If $\phi$ is an atom $P(t_1, t_2, \dotsc, t_n)$, then $\preds{\phi} =
		\set{P}$;
	\item If $\phi$ is of the form $\lnot \psi$, then $\preds{\phi} =
		\preds{\psi}$;
	\item If $\phi$ is of the form $\phi_1 \land \phi_2$, then $\preds{\phi} =
		\preds{\phi_1} \cup \preds{\phi_2}$;
	\item If $\phi$ is of the form $\exists x : \psi$, then $\preds{\phi} =
		\preds{\psi}$;
	\item If $\phi$ is of the form $\mk \psi$, then $\preds{\phi} =
		\preds{\psi}$;
	\item If $\phi$ is of the form $\mnot \psi$, then $\preds{\phi} =
		\preds{\psi}$.
\end{enumerate}
Given a set of formulae $\thr$, we define the \emph{set of predicate symbols
relevant to $T$}, denoted by $\preds{\thr}$, as follows:
\[
	\preds{T} = \bigcup_{\phi \in T} \preds{\phi} \enspace.
\]
\end{extended}

As in \cite{Motik2007}, we only consider Herbrand interpretations in our
semantics. We adopt the \emph{standard names assumption}, and apart from the
constants used in formulae, we assume our signature to contain a countably
infinite supply of constants. The Herbrand Universe of such a signature is
denoted by $\Delta$. The set of all (Herbrand) interpretations is denoted by
$\foint$. An \emph{MKNF structure} is a triple $\mstr$ where $I$ is an
interpretation and $M, N$ are sets of Herbrand interpretations. The
satisfiability of a ground atom $p$ and of an MKNF sentence $\phi$ in $\mstr$
is defined as follows
\vspace{-1.1em}
\begin{list}{}
	{
		\setlength{\parsep}{0em}
		\setlength{\partopsep}{0em}
		\setlength{\itemsep}{0em}
		\setlength{\topsep}{0em}
	}
	\item \small
	{\allowdisplaybreaks\begin{alignat*}{2} 
		\mstr \ent p
			& \text{ iff } I \ent p \\
		\mstr \ent \lnot \phi
			& \text{ iff } \mstr \nent \phi \\
		\mstr \ent \phi_1 \land \phi_2
			& \text{ iff } \mstr \ent \phi_1 \text{ and } \mstr \ent \phi_2 \\
		\mstr \ent \exists x : \phi
			& \text{ iff } \mstr \ent \phi[c/x] \text{ for some } c \in \Delta \\
		\mstr \ent \mk \phi
			& \text{ iff } \mstr[J, M, N] \ent \phi \text{ for all } J \in M \\
		\mstr \ent \mnot \phi
			& \text{ iff } \mstr[J, M, N] \nent \phi \text{ for some } J \in N
	\end{alignat*}}%
\end{list}
The symbols $\lor$, $\forall$ and $\lif$ are interpreted as usual.
\begin{normal}%
An \emph{MKNF interpretation} $M$ is a nonempty set of interpretations. Given
a set $\thr$ of MKNF sentences, we say $M$ is an \emph{S5 model} of $\thr$,
written $M \ent \thr$, if $\mstr[I, M, M] \ent \phi$ for every $\phi \in \thr$
and all $I \in M$. If there exists the greatest S5 model $M$ of $\thr$, then
it is denoted by $\smod(\thr)$. If $\thr$ has no S5 model, then $\smod(\thr)$
denotes the empty set. For all other sets of formulae, $\smod(\thr)$ stays
undefined. $M$ is an \emph{MKNF model} of $\thr$ if $M$ is an S5 model of
$\thr$ and for every MKNF interpretation $M' \supsetneq M$ there is some $I'
\in M'$ such that $\mstr[I', M', M] \nent \phi$ for some $\phi \in \thr$.
% For a sentence $\phi$, the S5 models of $\phi$, MKNF models of $\phi$ and
% $\smod(\phi)$ are defined as S5 models of $\set{\phi}$, MKNF models of
% $\set{\phi}$ and $\smod(\set{\phi})$.
\end{normal}
\begin{extended}%
The semantics of MKNF sentences is summarised in the following definition:

\begin{definition}[MKNF Interpretation and MKNF Model] \label{def:mknf:models}
An \emph{MKNF interpretation} is a nonempty set of Herbrand interpretations.
We also define $\mint = 2^\foint$ to be the set of all MKNF interpretations
together with the empty set.

Let $\thr$ be a set of MKNF sentences and $M \in \mint$. We write $M \ent
\thr$ if $\mstr[I, M, M] \ent \phi$ for every $\phi \in \thr$ and all $I \in
M$.\footnote{Notice that if $M$ is empty, this condition is vacuously
satisfied for any formula $\phi$, so any formula is true in $\emptyset$. For
this reason, $\emptyset$ is not considered an interpretation and only nonempty
subsets of $\mint$ can be models of formulae.} Otherwise we write $M \nent
\thr$.

If there exists the greatest $M \in \mint$ such that $M \ent \thr$, then we
denote it by $\smod(\thr)$. For all other sets of formulae $\smod(\thr)$ stays
undefined.

If $M$ is an MKNF interpretation (i.e. $M$ is nonempty and $M \in \mint$), we
say $M$ is
\begin{itemize}
	\item an \emph{S5 model of $\thr$} if $M \ent \thr$;
	\item an \emph{MKNF model of $\thr$} if $M$ is an S5 model of $\thr$ and for
		every MKNF interpretation $M' \supsetneq M$ there is some $I' \in M'$ and
		some $\phi \in \thr$ such that $\mstr[I', M', M] \nent \phi$.
\end{itemize}
We say a set of MKNF formulae $\thr$ is \emph{MKNF satisfiable} if an MKNF
model of $\thr$ exists; otherwise it is \emph{MKNF unsatisfiable}. The S5
(un)satisfiability is defined analogously by considering S5 models instead of
MKNF models. For a sentence $\phi$, we write $M \ent \phi$ if and only if $M
\ent \set{\phi}$; otherwise we write $M \nent \phi$. Also, $\smod(\phi)$, S5
models of $\phi$, MKNF models of $\phi$, MKNF (un)satisfiability and S5
(un)satisfiability of $\phi$ are defined as $\smod(\set{\phi})$, S5 models of
$\set{\phi}$, MKNF models of $\set{\phi}$, MKNF (un)satisfiability of
$\set{\phi}$ and S5 (un)satisfiability of $\set{\phi}$, respectively.

Let $U \subseteq \lpre$ be a set of predicate symbols, $I, J \in \foint$ and
$M, N \in \mint$. We define the \emph{restriction of $I$ to $U$} as $\restr{I}
= \Set{p \in I | \preds{p} \subseteq U}$ and the \emph{restriction of $M$ to
$U$} as $\restr{M} = \Set{\restr{I} | I \in M}$. We say that \emph{$I$
coincides with $J$ on $U$} if $\restr{I} = \restr{J}$ and that \emph{$M$
coincides with $N$ on $U$} if $\restr{M} = \restr{N}$.
\end{definition}
\end{extended}

\noindent \textbf{Description Logics.}
Description Logics \cite{Baader2003} are (usually) decidable fragments of
first-order logic that are frequently used for knowledge representation and
reasoning in applications. Throughout the paper we assume that some
Description Logic is used to describe an ontology, i.e. it is used to specify
a shared conceptualisation of a domain of interest. Basic building blocks of
such a specification are \emph{constants}, representing objects (or
individuals), \emph{concepts}, representing groups of objects, and
\emph{roles}, representing binary relations between objects and properties of
objects. Typically, an ontology is composed of two distinguishable parts: a
TBox specifying the required terminology, i.e. concept and role definitions,
and an ABox with assertions about constants.

Most Description Logics can be equivalently translated into function-free
first-order logic, with constants represented by constant symbols, atomic
concepts represented by unary predicates and atomic roles represented by
binary predicates. We assume that for any DL axiom $\phi$, $\dlfo{\phi}$
denotes such a translation of $\phi$. We also define $\preds{\phi}$ as
$\preds{\dlfo{\phi}}$.

\noindent \textbf{Generalised Logic Programs.}
We consider ground logic programs for specifying nonmonotonic domain
knowledge. The basic syntactic blocks of such programs are ground atoms. A
\emph{default literal} is a ground atom preceded by $\lpnot[]$. A
\emph{literal} is either a ground atom or a default literal. As a convention,
due to the semantics of rule updates that we adopt in what follows, double
default negation is absorbed, so that $\lpnot[]\lpnot p$ denotes the atom $p$.
A \emph{rule} $r$ is an expression of the form
$
	L_0 \lpif L_1, L_2, \dotsc, L_k
$
where $k$ is a natural number and $L_0, L_1, \dotsc, L_k$ are literals. We say
$H(r) = L_0$ is the \emph{head of $r$} and $B(r) = \set{L_1, L_2, \dotsc,
L_n}$ is the \emph{body of $r$}. A rule $r$ is a \emph{fact} if its body is
empty; $r$ is \emph{positive} if its head is an atom%
%; $r$ is \emph{negative} if its head is a default literal
. A \emph{generalised logic program} (GLP) $\prog$ is a set of rules. The set
of predicate symbols occurring in a literal $L$, set of literals $B$ and a
rule $r$ is denoted by $\preds{L}$, $\preds{B}$ and $\preds{r}$, respectively.
An interpretation $I$ is a \emph{stable model} of a GLP $\prog$ if
$
	I' = \least{\prog \cup \set{ \lpnot p |
		p \text{ is an atom and } p \notin I } },
$
where $I' = I \cup \set{ \mathit{not}\!\_p | p \text{ is an atom and } p
\notin I }$ and $\least{\cdot}$ denotes the least model of the program
obtained from the argument program by replacing every default literal $\lpnot
p$ by a fresh atom $\mathit{not}\!\_p$.

\noindent \textbf{Hybrid MKNF Knowledge Bases.}
A hybrid knowledge base is a pair $\an{\ont, \prog}$ where $\ont$ is an
ontology and $\prog$ is a generalised logic program. The semantics is assigned
to a hybrid knowledge base using a translation function $\pi$ that translates
both ontology axioms and rules into MKNF sentences. For any ontology $\ont$,
ground atom $p$, set of literals $B$, rule $r$, program $\prog$ and hybrid
knowledge base $\kb = \an{\ont, \prog}$, we define: $\pi(\ont) = \Set{ \mk
\dlfo{\phi} | \phi \in \ont }$, $\pi(p) = \mk p$, $\pi(\lpnot p) = \mnot p$,
$\pi(B) = \set{ \pi(L) | L \in B }$, $\pi(r) = (\pi(H(r)) \lif \bigland
\pi(B(r)))$, $\pi(\prog) = \Set{ \pi(r) | r \in \prog }$ and $\pi(\kb) =
\pi(\ont) \cup \pi(\prog)$. An MKNF interpretation $M$ is an \emph{S5 (MKNF)
model of $\kb$} if $M$ is an S5 (MKNF) model of $\pi(\kb)$.

As was shown in \cite{Lifschitz1991}, the MKNF semantics generalises the
stable model semantics for logic programs \cite{Gelfond1988} -- for every
logic program $\prog$, the stable models of $\prog$ directly correspond to
MKNF models of $\pi(\prog)$.

\noindent \textbf{Classical Updates.}
As a basis for our update operator, we adopt an update semantics called the
\emph{minimal change update semantics} \cite{Winslett1990} for updating
first-order theories. This update semantics offers a simple realisation of
atom inertia, satisfies all Katsuno and Mendelzon's postulates for belief
update \cite{Katsuno1991}, and it has successfully been used to deal with ABox
updates \cite{Liu2006,Giacomo2006}.

A notion of closeness between interpretations w.r.t.~a fixed interpretation
$I$ is used to determine the result of an update. This closeness is based on
the set of ground atoms that are interpreted differently than in $I$. For a
predicate symbol $P$ and an interpretation $I$, we denote the set $\set{p \in
I | \preds{p} = \set{P}}$ by $\restr[P]{I}$. Given interpretations $I$, $J$,
$J'$, the \emph{difference in the interpretation of $P$ between $I$ and $J$},
written $\mathit{diff}(P, I, J)$, is the set $(\restr[P]{I} \setminus
\restr[P]{J}) \cup (\restr[P]{J} \setminus \restr[P]{I})$. We say that $J$ is
at least as close to $I$ as $J'$, denoted by $J \leq_I J'$, if for every
predicate symbol $P$ it holds that $\mathit{diff}(P, I, J)$ is a subset of
$\mathit{diff}(P, I, J')$. We also say that \emph{$J$ is closer to $I$ than
$J'$}, denoted by $J <_I J'$, if $J \leq_I J'$ and not $J' \leq_I J$.

The minimal change update semantics then keeps those models of the updating
theory that are the closest w.r.t.~the relation $\leq_I$ to some model $I$ of
the original theory. Given an interpretation $I$, sets of interpretations $M$,
$N$, and first-order theories $\thr, \upd$, we define: $I \oplus N = \Set{ J
\in N | \lnot (\exists J' \in N) (J' <_I J) }$, $M \oplus N = \bigcup_{I \in
M} (I \oplus N)$, and $\smod(\thr \oplus \upd) = \smod(\thr) \oplus
\smod(\upd)$. If $\smod(\thr \oplus \upd)$ is nonempty, we say it is the
\emph{minimal change update model of $\thr \oplus \upd$}. This notion can be
naturally generalised to allow for sequences of updates. %
% Starting from the models of the original theory, for each update in the
% sequence we can transform the set of models according to the minimal change
% update semantics defined above. The resulting set of models then determines
% the updated theory.
Formally, given a finite sequence of first-order theories $\upd =
\an{\upd_i}_{i < n}$, we define
$
	\smod(\upd) = (\dotsb((\smod(\upd_0) \oplus \smod(\upd_1)) \oplus \smod(\upd_2))
	\oplus \dotsb ) \oplus \smod(\upd_{n-1}).
$
If $\smod(\upd)$ is nonempty, we say it is the \emph{minimal change update
model of $\upd$}.

\noindent \textbf{Rule Updates.}
There exists a variety of different approaches to rule change
\cite{Leite1997,Alferes2000,Eiter2002,Sakama2003,Alferes2005,Zhang2006,Osorio2007,Delgrande2007,Delgrande2008,Slota2010b}.
The more recent, purely semantic approaches \cite{Delgrande2008,Slota2010b}
are closely related to classical update operators such as Winslett's operator
presented above. However, as indicated in \cite{Slota2010b}, their main
disadvantage is that they violate the property of support that lies at the
very heart of semantics of logic programs. Out of the approaches that do
respect support, only the rule update semantics presented in
\cite{Alferes2005,Zhang2006} possess another important property: immunity to
tautological and cyclic updates. We henceforth adopt the approach taken in
\cite{Alferes2005} because, unlike in \cite{Zhang2006}, it can be applied to
any initial program, can easily be used to perform iterative updates and has a
lower computational complexity.

A \emph{dynamic logic program} (DLP) is a finite sequence of GLPs. In order to
define the semantics for DLPs, based on \emph{causal rejection of rules}, we
define the notion of a conflict between rules as follows: two rules $r$ and
$r'$ are \emph{conflicting}, written $r \Join r'$, if $H(r) =\,\lpnot\!H(r')$.
Given a DLP $\prog = \an{\prog_i}_{i < n}$ and an interpretation $I$, we use
$\rho(\prog)$ to denote the multiset of all rules appearing in members of
$\prog$ and introduce the following notation:
\begin{align*}
	\dynrej{\prog, I} &= \Set{ r | (\exists i, j, r')
		( r \in \prog_i \land r' \in \prog_j \land i \leq j \land
		r \Join r' \land I \ent B(r') ) } \enspace, \\
	\dyndef{\prog, I} &=
		\Set{\lpnot p | (\lnot \exists r \in \rho(\prog)) (H(r) = p \land I \ent B(r))}
		\enspace.
\end{align*}
An interpretation $I$ is a \emph{dynamic stable model} of a DLP
$\prog$ if
$
	I' = \least{ [ \rho(\prog) \setminus \dynrej{\prog,I} ] \cup
		\dyndef{\prog, I} }
$,
% \begin{align*}
% 	I' &= \least{ [ \rho(\prog) \setminus \dynrej{\prog,I} ] \cup
% 		\dyndef{\prog, I} } \enspace,
% \end{align*}
where $I'$ and $\least{\cdot}$ are as in the definition of a stable model.

\section{Splitting and Updating Hybrid Knowledge Bases} \label{sect:core}

Our general objective is to define an update semantics for finite sequences
of hybrid knowledge bases, where each component represents knowledge about
a new state of the world.
\begin{definition}[Dynamic Hybrid Knowledge Base]
	A \emph{dynamic hybrid knowledge base} is a finite sequence of hybrid
	knowledge bases.
\end{definition}
In this paper we develop an update operator for a particular class of
syntactically constrained hybrid knowledge bases. The purpose of the
constraints we impose is to ensure that the ontology and rules can be updated
separately from one another, and the results can then be combined to obtain a
plausible update semantics for the whole hybrid knowledge base. Formally, the
update semantics we introduce generalises and modularly combines a classical
and a rule update operator.

In order to identify these constraints, we introduce the splitting theorem for
Hybrid MKNF Knowledge Bases in Subsect.~\ref{subsect:splitting theorem}. Based
on it, we identify a constrained class of dynamic hybrid knowledge bases and
define an update operator for that class in Subsect.~\ref{subsect:update
operator}. Finally, in Subsect \ref{subsect:scenario} we examine basic
properties of the operator and illustrate how it can deal with updates to the
hybrid knowledge base from Example \ref{ex:import:description}.

\subsection{Splitting Theorem}

\label{subsect:splitting theorem}

The splitting theorem for Logic Programs \cite{Lifschitz1994a} is a
generalisation of the notion of program stratification. Given a logic program
$\prog$, a splitting set for $\prog$ is a set of atoms $U$ such that the
program can be divided in two subprograms, the bottom and the top of $\prog$,
such that rules in the bottom only contain atoms from $U$, and no atom from
$U$ occurs in the head of any rule from the top. As a consequence, rules in
the top of $\prog$ cannot influence the stable models of its bottom.  The
splitting theorem captures this intuition, guaranteeing that each stable model
of $\prog$ is a union of a stable model $X$ of the bottom of $\prog$ and of a
stable model $Y$ of a reduced version of the top of $\prog$ where atoms
belonging to $U$ are interpreted under $X$. This can be further generalised to
sequences of splitting sets that divide a program $\prog$ into a sequence of
layers. The splitting sequence theorem then warrants that stable models of
$\prog$ consist of a union of stable models of each of its layers after
appropriate reductions.

In the following we generalise these notions to Hybrid MKNF Knowledge Bases
\cite{Motik2007}. The first definition establishes the notion of a splitting
set in this context.

\begin{definition}[Splitting Set]
	A splitting set for a hybrid knowledge base $\kb = \an{\ont, \prog}$ is any
	set of predicate symbols $U \subseteq \lpre$ such that
	\begin{enumerate}
		\item For every ontology axiom $\phi \in \ont$, if $\preds{\phi} \cap U
			\neq \emptyset$, then $\preds{\phi} \subseteq U$.
		\item For every rule $r \in \prog$, if $\preds{H(r)} \cap U \neq
			\emptyset$, then $\preds{r} \subseteq U$;
	\end{enumerate}

	The set of ontology axioms $\phi \in \ont$ such that $\preds{\phi} \subseteq
	U$ is called the \emph{bottom of $\ont$ relative to $U$} and denoted by
	$b_U(\ont)$. The set of rules $r \in \prog$ such that $\preds{r} \subseteq U$
	is called the \emph{bottom of $\prog$ relative to $U$} and denoted by
	$b_U(\prog)$. The hybrid knowledge base $b_U(\kb) = \an{b_U(\ont),
	b_U(\prog)}$ is called \emph{bottom of $\kb$ relative to $U$}.

	The set $t_U(\ont) = \ont \setminus b_U(\ont)$ is the \emph{top of $\ont$
	relative to $U$}. The set $t_U(\prog) = \prog \setminus b_U(\prog)$ is the
	\emph{top of $\prog$ relative to $U$}. The hybrid knowledge base $t_U(\kb) =
	\an{t_U(\ont), t_U(\prog)}$ is the \emph{top of $\kb$ relative to $U$}.
\end{definition}

Note that instead of defining a splitting set as a set of atoms, as was done
in the case of propositional logic programs, we define it as a set of
predicate symbols. By doing this, the set of ground atoms with the same
predicate symbol is considered either completely included in a splitting set,
or completely excluded from it. While this makes our approach slightly less
general than it could be if we considered each ground atom individually, we
believe the conceptual simplicity is worth this sacrifice. Also, since all
TBox axioms are universally quantified, in many cases we would end up
adding or excluding the whole set of ground atoms with the same predicate
symbol anyway.

Next, we need to define the reduction that makes it possible to properly
transfer information from an MKNF model of the bottom of $\kb$, and use it to
simplify the top of $\kb$.

\begin{definition}[Splitting Set Reduct]
	Let $U$ be a splitting set for a hybrid knowledge base $\kb = \an{\ont,
	\prog}$ and $X \in \mint$. The \emph{splitting set reduct of $\kb$ relative
	to $U$ and $X$} is a hybrid knowledge base $e_U(\kb, X) = \an{t_U(\ont),
	e_U(\prog, X)}$, where $e_U(\prog, X)$ consists of all rules $r'$ such that
	there exists a rule $r \in t_U(\prog)$ with the following properties: $X
	\ent \pi(\Set{ L \in B(r) | \preds{L} \subseteq U })$, $H(r') = H(r)$, and
	$B(r') = \Set{ L \in B(r) | \preds{L} \subseteq \lpre \setminus U }$.
% 	\begin{align}
% 		H(r') &= H(r) \enspace, \label{eq:def:splitting reduct:1} \\
% 		B(r') &= \Set{ L \in B(r) | \preds{L} \subseteq \lpre \setminus U }
% 			\enspace, \label{eq:def:splitting reduct:2} \\
% 		X &\ent \pi(\Set{ L \in B(r) | \preds{L} \subseteq U })
% 			\enspace. \label{eq:def:splitting reduct:3}	
% 	\end{align}
\end{definition}

This leads us to the notion of a solution to $\kb$ w.r.t.~a splitting set $U$.

\begin{definition}[Solution w.r.t.~a Splitting Set]
	Let $U$ be a splitting set for a hybrid knowledge base $\kb$. A
	\emph{solution to $\kb$ w.r.t.~$U$} is a pair of MKNF interpretations
	$\an{X, Y}$ such that $X$ is an MKNF model of $b_U(\kb)$ and $Y$ is an MKNF
	model of $e_U(\kb, X)$.
\end{definition}

The splitting theorem now ensures that solutions to $\kb$ w.r.t.~any splitting
set $U$ are in one to one correspondence with the MKNF models of $\kb$.

\begin{theorem}[Splitting Theorem for Hybrid MKNF Knowledge Bases]
	\label{thm:splitting}
	Let $U$ be a splitting set for a hybrid knowledge base $\kb$. Then $M$ is an
	MKNF model of $\kb$ if and only if $M = X \cap Y$ for some solution $\an{X,
	Y}$ to $\kb$ w.r.t.~$U$.
\end{theorem}
\begin{extended}
	\begin{proof}
		See \ref{app:splitting theorem}, page \pageref{proof:thm:splitting}.
	\end{proof}
\end{extended}

This result makes it possible to characterise an MKNF model of a hybrid
knowledge base in terms of a pair of MKNF models of two layers inside it, such
that, as far as the MKNF semantics is concerned, the first layer is
independent of the second. If instead of a single splitting set we consider a
sequence of such sets, we can divide a hybrid knowledge in a sequence of
layers, keeping similar properties as in the case of a single splitting set.

\begin{definition}[Splitting Sequence]
	A \emph{splitting sequence} for a hybrid knowledge base $\kb$ is a monotone,
	continuous sequence $\an{U_\alpha}_{\alpha < \mu}$ of splitting sets for
	$\kb$ such that $\bigcup_{\alpha < \mu} U_\alpha = \lpre$.
\end{definition}

The first layer of $\kb$ relative to such a splitting sequence is the part of
$\kb$ that only contains predicates from $U_0$. Formally, this is exactly the
hybrid knowledge base $b_{U_0}(\kb)$. Furthermore, for every ordinal $\alpha +
1 < \mu$, the corresponding layer of $\kb$ is the part of $\kb$ that contains
predicates from $U_{\alpha + 1} \setminus U_\alpha$, and, in addition,
predicate symbols from $U_\alpha$ are allowed to appear in rule bodies. Given
our notation this can be written as $t_{U_\alpha}(b_{U_{\alpha + 1}}(\kb))$.
The following definition uses these observations and combines them with
suitable reductions to introduce a solution w.r.t.~a splitting sequence.

\begin{definition}[Solution w.r.t.~a Splitting Sequence]
	Let $U = \an{U_\alpha}_{\alpha < \mu}$ be a splitting sequence for a hybrid
	knowledge base $\kb$. A \emph{solution to $\kb$ w.r.t.~$U$} is a
	sequence $\an{X_\alpha}_{\alpha < \mu}$ of MKNF interpretations such that
	\begin{enumerate}
		\item $X_0$ is an MKNF model of $b_{U_0}(\kb)$;
		\item For any ordinal $\alpha$ such that $\alpha + 1 < \mu$, $X_{\alpha +
			1}$ is an MKNF model of
			\[
				e_{U_\alpha} \left(
					b_{U_{\alpha+1}}(\kb),
					\textstyle \bigcap_{\eta \leq \alpha} X_\eta
				\right) \enspace;
			\]
		\item For any limit ordinal $\alpha < \mu$, $X_\alpha = \foint$.
	\end{enumerate}
\end{definition}

The splitting sequence theorem now guarantees a one to one correspondence
between MKNF models and solutions w.r.t.~a splitting sequence.

\begin{theorem}[Splitting Sequence Theorem for Hybrid MKNF Knowledge Bases]
	\label{thm:seq}
	Let $\an{U_\alpha}_{\alpha < \mu}$ be a splitting sequence for a hybrid
	knowledge base $\kb$. Then $M$ is an MKNF model of $\kb$ if and only if $M =
	\bigcap_{\alpha < \mu} X_\alpha$ for some solution $\an{X_\alpha}_{\alpha <
	\mu}$ to $\kb$ w.r.t.~$\an{U_\alpha}_{\alpha < \mu}$.
\end{theorem}
\begin{extended}
	\begin{proof}
		See \ref{app:splitting sequence theorem}, page \pageref{proof:thm:seq}.
	\end{proof}
\end{extended}

A hybrid knowledge base can be split in a number of different ways. For
example, $\emptyset$ and $\lpre$ are splitting sets for any hybrid knowledge
base and sequences such as $\an{\lpre}$, $\an{\emptyset, \lpre}$ are splitting
sequences for any hybrid knowledge base. The following example shows a more
elaborate splitting sequence for the Cargo Import knowledge base.

\begin{example}[Splitting the Cargo Import Knowledge Base]
	\label{ex:import:splitting}

	Consider the hybrid knowledge base $\kb = \an{\ont, \prog}$ presented in
	Fig.~\ref{fig:import}. One of the nontrivial splitting sequences for $\kb$
	is $U = \an{U_0, U_1, U_2, U_3, \lpre}$, where
	\vspace{-.7em}
	\begin{list}{}
		{
			\setlength{\parsep}{0em}
			\setlength{\partopsep}{0em}
			\setlength{\itemsep}{0em}
			\setlength{\topsep}{0em}
		}
		\item \footnotesize
			{\allowdisplaybreaks
		\begin{align*}
			U_0 = \phantom{U_0 \cup{}} \{\,
				& \pr{Commodity}/1, \pr{EdibleVegetable}/1, \pr{Tomato}/1,
					\pr{CherryTomato}/1, \pr{GrapeTomato}/1, \\
				& \pr{HTSCode}/2, \pr{HTSChapter}/2, \pr{HTSHeading}/2,
					\pr{Bulk}/1, \pr{Prepackaged}/1, \pr{TariffCharge}/2, \\
				& \pr{ShpmtCommod}/2, \pr{ShpmtImporter}/2,
					\pr{ShpmtDeclHTSCode}/2, \pr{ShpmtProducer}/2, \\
				& \pr{ShpmtCountry}/2
			\,\} \\
			U_1 = U_0 \cup \{\,
				& \pr{AdmissibleImporter}/1, \pr{SuspectedBadGuy}/1,
					\pr{ApprovedImporterOf}/2
			\,\} \\
			U_2 = U_1 \cup \{\,
				& \pr{RegisteredProducer}/2, \pr{EUCountry}/1,
					\pr{EURegisteredProducer}/1, \pr{CommodCountry}/2, \\
				& \pr{ExpeditableImporter}/2, \pr{LowRiskEUCommodity}/1
			\,\} \\
			U_3 = U_2 \cup \{\,
				& \pr{CompliantShpmt}/1, \pr{Random}/1, \pr{RandomInspection}/1,
					\pr{PartialInspection}/1, \\
				& \pr{FullInspection}/1
			\,\} \enspace.
		\end{align*}
		}
	\end{list}
	This splitting sequence splits $\kb$ in four layers. The first layer
	contains all ontological knowledge regarding commodity types as well as
	information about shipments. The second layer contains rules that use
	information from the first layer together with internal records to classify
	importers. The third layer contains axioms with geographic classification,
	information about registered producers and, based on information about
	commodities and importers from the first two layers, it defines low risk
	commodities coming from the European Union. The final layer contains rules
	for deciding which shipments should be inspected based on information from
	previous layers.
% 	By applying the definition of splitting sequence to Fig.~\ref{fig:import}
% 	we note that there are at least two rule and two ontology strata. Note that
% 	the ontological concept of a {\em LowRiskEUCommodity} depends on the rule
% 	definition of {\em ExpeditableImporter}, leading to a split of the ontology.
% 	Furthermore, the second clause for {\em FullInspection} depends on {\em
% 	LowRiskEUCommodity}, leading to a split in the program. There are thus two
% 	ontology layers interleaved with two program layers.
\end{example}

\subsection{Update Operator}

\label{subsect:update operator}

With the concepts and results related to splitting hybrid knowledge bases from
the previous subsection, we are now ready to examine the constraints under
which a plausible modular update semantics for a hybrid knowledge base can be
defined. Obviously, this is the case with hybrid knowledge bases that contain
either only ontology axioms, or only rules. We call such knowledge bases
\emph{basic}, and define the dynamic MKNF model for basic dynamic knowledge
bases by referring to the classical and rule update semantics defined in
Sect.~\ref{sect:preliminaries}.

\begin{definition}[Dynamic MKNF Model of a Basic Dynamic Hybrid Knowledge Base]
	\label{def:basic dynamic hybrid knowledge base}
	We say a hybrid knowledge base $\kb = \an{\ont, \prog}$ is
	\emph{$\ont$-based} if $\prog$ contains only positive facts;
	\emph{$\prog$-based} if $\ont$ is empty; \emph{basic} if it is either
	$\ont$-based or $\prog$-based. A dynamic hybrid knowledge base $\kb =
	\an{\kb_i}_{i < n}$ is \emph{$\ont$-based} if for all $i < n$, $\kb_i$ is
	$\ont$-based; \emph{$\prog$-based} if for all $i < n$, $\kb_i$ is
	$\prog$-based; \emph{basic} if it is either $\ont$-based or $\prog$-based.

	An MKNF interpretation $M$ is a \emph{dynamic MKNF model} of a basic dynamic
	hybrid knowledge base $\kb = \an{\kb_i}_{i < n}$, where $\kb_i = \an{\ont_i,
	\prog_i}$, if either $\kb$ is $\ont$-based and $M$ is the minimal change
	update model of $\an{\dlfo{\ont_i} \cup \prog_i}_{i < n}$, or $\kb$ is
	$\prog$-based and $M = \set{J \in \foint | I \subseteq J}$ for some dynamic
	stable model $I$ of $\an{\prog_i}_{i < n}$.
\end{definition}

As can be seen, our definition is slightly more general than described above,
as in the case of $\ont$-based knowledge bases it allows the program part to
contain positive facts. This amounts to the reasonable assumption that
positive facts in a logic program carry the same meaning as the corresponding
ground first-order atom. As will be seen in the following, this allows us to
extend the class of basic hybrid knowledge bases and define dynamic MKNF
models for it. To this end, we utilise the splitting-related concepts from the
previous subsection. Their natural generalisation for dynamic hybrid knowledge
bases follows.

\begin{definition}[Splitting Set and Splitting Sequence]
	A set of predicate symbols $U$ is a splitting set for a dynamic hybrid
	knowledge base $\kb = \an{\kb_i}_{i < n}$ if for all $i < n$, $U$ is a
	splitting set for $\kb_i$.

	The dynamic hybrid knowledge base $\an{b_U(\kb_i)}_{i < n}$ is called the
	\emph{bottom of $\kb$ relative to $U$} and denoted by $b_U(\kb)$. The
	dynamic hybrid knowledge base $\an{t_U(\kb_i)}_{i < n}$ is called the
	\emph{top of $\kb$ relative to $U$} and denoted by $t_U(\kb)$. Given some $X
	\in \mint$, the dynamic hybrid knowledge base $\an{e_U(\kb_i, X)}_{i < n}$
	is called the \emph{splitting set reduct of $\kb$ relative to $U$ and $X$}
	and denoted by $e_U(\kb, X)$.

	A sequence of sets of predicate symbols $U$ is a splitting sequence for
	$\kb$ if for all $i < n$, $U$ is a splitting sequence for $\kb_i$.
\end{definition}

In the static case, given a splitting set $U$, the splitting set theorem
guarantees that an MKNF model $M$ of a hybrid knowledge base $\kb$ is an
intersection of an MKNF model $X$ of $b_U(\kb)$ and of an MKNF model $Y$ of
$e_U(\kb, X)$. In the dynamic case, we can use this correspondence to
\emph{define} a dynamic MKNF model. More specifically, we can say that $M$ is
a dynamic MKNF model of a dynamic hybrid knowledge base $\kb$ if $M$ is an
intersection of a dynamic MKNF model $X$ of $b_U(\kb)$ and of a dynamic MKNF
model $Y$ of $e_U(\kb, X)$.  For the definition to be sound, we need to
guarantee that $X$ and $Y$ are defined. In other words, $\kb$ has to be such
that both $b_U(\kb)$ and $e_U(\kb, X)$ are basic. When we move to the more
general case of a splitting sequence $U = \an{U_\alpha}_{\alpha < \mu}$, what
we need to ensure is that $b_{U_0}(\kb)$ is basic and for any ordinal $\alpha$
such that $\alpha + 1 < \mu$, $e_{U_\alpha}(b_{U_{\alpha + 1}}(\kb),
\bigcap_{\eta < \alpha} X_\eta)$ is also basic. A class of dynamic hybrid
knowledge bases that satisfies these conditions can be defined as follows:

\begin{definition}[Updatable Dynamic Hybrid Knowledge Base]
	Let $U$ be a set of predicate symbols. We say a hybrid knowledge base $\kb$
	is \emph{$\ont$-reducible relative to $U$} if all rules $r$ from $\prog$ are
	positive and $\preds{B(r)} \subseteq U$; \emph{$\prog$-reducible relative to
	$U$} if $\ont$ is empty; \emph{reducible relative to $U$} if it is either
	$\ont$-reducible or $\prog$-reducible relative to $U$. A dynamic hybrid
	knowledge base $\kb = \an{\kb_i}_{i < n}$ is \emph{$\ont$-reducible relative
	to $U$} if for all $i < n$, $\kb_i$ is $\ont$-reducible relative to $U$;
	\emph{$\prog$-reducible relative to $U$} if for all $i < n$, $\kb_i$ is
	$\prog$-reducible relative to $U$; \emph{reducible relative to $U$} if it is
	either $\ont$-reducible or $\prog$-reducible relative to $U$.

	Let $\kb$ be a (dynamic) hybrid knowledge base and $U =
	\an{U_\alpha}_{\alpha < \mu}$ be a splitting sequence for $\kb$. We say $U$
	is \emph{update-enabling for $\kb$} if $b_{U_0}(\kb)$ is reducible relative
	to $\emptyset$ and for any $\alpha$ such that $\alpha + 1 < \mu$, the hybrid
	knowledge base $t_{U_\alpha}(b_{U_{\alpha + 1}}(\kb))$ is reducible relative
	to $U_\alpha$. We say $\kb$ is \emph{updatable} if some update-enabling
	splitting sequence for $\kb$ exists.
\end{definition}

The following proposition now guarantees the property of updatable dynamic
hybrid knowledge bases that we discussed above.

\begin{proposition}
	[Layers of an Updatable Dynamic Hybrid Knowledge Base are Basic]
	\label{prop:update-enabling implies basic}
	Let $U$ be an update-enabling splitting sequence for a dynamic hybrid
	knowledge base $\kb$ and $X \in \mint$. Then $b_{U_0}(\kb)$ is a basic
	dynamic hybrid knowledge base and for any ordinal $\alpha$ such that $\alpha
	+ 1 < \mu$, $e_{U_\alpha}(b_{U_{\alpha+1}}(\kb), X)$ is also a basic dynamic
	hybrid knowledge base.
\end{proposition}
\begin{extended}%
	\begin{proof}
		See \ref{app:hybrid update operator}, page
		\pageref{proof:prop:update-enabling implies basic}.
	\end{proof}
\end{extended}

This result paves the way to the following definition of a solution to an
updatable dynamic hybrid knowledge base together with the notion of a dynamic
MKNF model w.r.t.~an updatable splitting sequence.

\begin{definition}[Solution to an Updatable Dynamic Hybrid Knowledge Base]
	Let $U = \an{U_\alpha}_{\alpha < \mu}$ be an update-enabling splitting
	sequence for a dynamic hybrid knowledge base $\kb$. A \emph{solution to
	$\kb$ w.r.t.~$U$} is a sequence of MKNF interpretations
	$\an{X_\alpha}_{\alpha < \mu}$ such that 
	\begin{enumerate}
		\item $X_0$ is a dynamic MKNF model of $b_{U_0}(\kb)$;
		\item For any ordinal $\alpha$ such that $\alpha + 1 < \mu$, $X_{\alpha +
			1}$ is a dynamic MKNF model of
			\[
				e_{U_\alpha} \left(
					b_{U_{\alpha+1}}(\kb),
					\textstyle \bigcap_{\eta \leq \alpha} X_\eta
				\right) \enspace;
			\]
		\item For any limit ordinal $\alpha < \mu$, $X_\alpha = \foint$.
	\end{enumerate}
	We say that an MKNF interpretation $M$ is a \emph{dynamic MKNF model of
	$\kb$ w.r.t.~$U$} if $M = \bigcap_{\alpha < \mu} X_\alpha$ for some solution
	$\an{X_\alpha}_{\alpha < \mu}$ to $\kb$ w.r.t.~$U$.
\end{definition}

The last step required to define a dynamic MKNF model of an updatable dynamic
hybrid knowledge base, without the need to refer to a context of a particular
splitting sequence, is to ensure that Def.~\ref{def:basic dynamic hybrid
knowledge base} of a dynamic MKNF model for basic dynamic hybrid knowledge
bases is properly generalised. The following proposition guarantees that the
set of dynamic MKNF models is independent of a particular update-enabling
splitting sequence.

\begin{proposition}
	[Solution Independence]
	\label{prop:solution independence}
	Let $U, V$ be update-enabling splitting sequences for a dynamic hybrid
	knowledge base $\kb$. Then $M$ is a dynamic MKNF model of $\kb$ w.r.t.~$U$
	if and only if $M$ is a dynamic MKNF model of $\kb$ w.r.t.~$V$.
\end{proposition}
\begin{extended}%
	\begin{proof}
		See \ref{app:hybrid update operator}, page
		\pageref{proof:prop:solution independence}.
	\end{proof}
\end{extended}

If $\kb$ is a basic dynamic hybrid knowledge base, then it can be verified
easily that dynamic MKNF models of $\kb$, as originally defined in
Def.~\ref{def:basic dynamic hybrid knowledge base}, coincide with dynamic MKNF
models of $\kb$ w.r.t.~the splitting sequence $\an{\lpre}$. We obtain the
following corollary:

\begin{corollary}
	[Compatibility with Def.~\ref{def:basic dynamic hybrid knowledge base}]
	\label{cor:compatibility with basic}
	Let $\kb$ be a basic dynamic hybrid knowledge base and $U$ be a splitting
	sequence for $\kb$. Then $M$ is a dynamic	MKNF model of $\kb$ if and only if
	$M$ is a dynamic MKNF model of $\kb$ w.r.t.~$U$.
\end{corollary}
\begin{extended}%
	\begin{proof}
		See \ref{app:hybrid update operator}, page
		\pageref{proof:cor:compatibility with basic}.
	\end{proof}
\end{extended}

We can now safely introduce the dynamic MKNF model for any updatable dynamic
hybrid knowledge base as follows:

\begin{definition}[Dynamic MKNF Model of Updatable Dynamic Hybrid Knowledge Base]
	An MKNF interpretation $M$ is a \emph{dynamic MKNF model} of an updatable
	dynamic hybrid knowledge base $\kb$ if $M$ is a dynamic MKNF model of $\kb$
	w.r.t.~some update-enabling splitting sequence for $\kb$.
\end{definition}

\subsection{Properties and Use} \label{subsect:scenario}

The purpose of this section is to twofold. First, we establish the most basic
properties of the defined update semantics, relating it to the static MKNF
semantics and the adopted classical and rule update semantics and showing that
it respects one of the most widely accepted principles behind update semantics
in general, the principle of primacy of new information. Second, we
illustrate its usefulness by considering updates of the hybrid knowledge base
presented in Example \ref{ex:import:description}.

The first result shows that our update semantics generalises the static MKNF
semantics.

\begin{theorem}
	[Generalisation of MKNF Models]
	\label{thm:generalisation of mknf}
	Let $\kb$ be an updatable hybrid knowledge base and $M$ be an MKNF
	interpretation. Then $M$ is a dynamic MKNF model of $\an{\kb}$ if and only
	if $M$ is an MKNF model of $\kb$.
\end{theorem}
\begin{extended}%
	\begin{proof}
		See \ref{app:hybrid update operator}, page
		\pageref{proof:thm:generalisation of mknf}.
	\end{proof}
\end{extended}

\noindent It also generalises the classical and rule update semantics it is based on.

\begin{theorem}
	[Generalisation of Minimal Change Update Semantics]
	\label{thm:generalisation of pma}
	Let $\an{\kb_i}_{i < n}$, where $\kb_i = \an{\ont_i, \prog_i}$, be a dynamic
	hybrid knowledge base such that $\prog_i$ is empty for all $i < n$. Then $M$
	is a dynamic MKNF model of $\an{\kb_i}_{i < n}$ if and only if $M$ is the
	minimal change update model of $\an{\dlfo{\ont_i}}_{i < n}$.
\end{theorem}
\begin{extended}%
	\begin{proof}
		See \ref{app:hybrid update operator}, page
		\pageref{proof:thm:generalisation of pma}.
	\end{proof}
\end{extended}

\begin{theorem}
	[Generalisation of Dynamic Stable Model Semantics]
	\label{thm:generalisation of dlp}
	Let $\kb = \an{\kb_i}_{i < n}$, where $\kb_i = \an{\ont_i, \prog_i}$, be a
	dynamic hybrid knowledge base such that $\ont_i$ is empty for all $i < n$.
	Then $M$ is a dynamic MKNF model of $\kb$ if and only if $M = \set{J \in
	\foint | I \subseteq J}$ for some dynamic stable model $I$ of
	$\an{\prog_i}_{i < n}$.
\end{theorem}
\begin{extended}%
	\begin{proof}
		See \ref{app:hybrid update operator}, page
		\pageref{proof:thm:generalisation of dlp}.
	\end{proof}
\end{extended}

\noindent Besides, the semantics respects the principle of primacy of new information \cite{Dalal1988}.

\begin{theorem}
	[Principle of Primacy of New Information]
	\label{thm:primacy of new information}
	Let $\kb = \an{\kb_i}_{i < n}$ be an updatable dynamic hybrid knowledge base
	with $n > 0$ and $M$ be a dynamic MKNF model of $\kb$. Then $M \ent
	\pi(\kb_{n-1})$.
\end{theorem}
\begin{extended}%
	\begin{proof}
		See \ref{app:hybrid update operator}, page
		\pageref{proof:thm:primacy of new information}.
	\end{proof}
\end{extended}

% \begin{theorem}
% 	Let $U$ be an update-enabling splitting sequence for both a dynamic hybrid
% 	knowledge base $\kb = \an{\kb_i}_{i < n}$ and for a hybrid knowledge base $\kb_n
% 	= \an{\ont_n, \prog_n}$. Suppose that for some $\beta < \mu$, all ontology
% 	axioms and heads of rules in $\prog_n$ contain only predicate symbols from
% 	$\lpre \setminus U_\beta$, $\phi$ is a first-order sentence such
% 	that $\preds{\phi} \subseteq U_\beta$, $M$ is a dynamic MKNF model of $\kb$
% 	and $M'$ is a dynamic MKNF model of $\an{\kb_i}_{i \leq n}$. Then $M \ent
% 	\phi$ if and only if $M' \ent \phi$.
% 
% 	Let $U$ be a splitting sequence for both $\kb$ and $\kb_n$ and suppose
% 	that
% 	$\ont_n$ and heads of rules in $\prog_n$ do not contain predicate symbols
% 	from $U_\beta$ for some $\beta < \mu$. 
% 
% 	Suppose also that $M$ is a dynamic MKNF model of $\kb$ and $M'$ is a dynamic
% 	model of $\an{\kb_i}_{i < n+1}$ and that a first-order formula $\phi$ is
% 	such that $\preds{\phi} \subseteq U_\beta$. Then $M \ent \phi$ if and only if
% 	$M' \ent \phi$.
% \end{theorem}

The following example illustrates how the semantics can be used in the Cargo
Imports domain to incorporate new, conflicting information into a hybrid
knowledge base.

\begin{extended}
	\begin{figure}
		\longline

		\begin{center}
		*\ \ *\ \ *\ \ $b_{U_0}(\kb)$ *\ \ *\ \ *
		\end{center}
		\vspace{-.4cm}
		\begin{tabbing}
		foooooooooooooooooooooooooooooooooooooooooo\=\kill
		$\pr{Commodity} \equiv (\exists \pr{HTSCode}.\top)$
			\> $\pr{EdibleVegetable} \equiv (\exists \pr{HTSChapter}.\set{\text{`07'}})$ \\
		$\pr{CherryTomato} \equiv (\exists \pr{HTSCode}.\set{\text{`07020020'}})$
			\> $\pr{Tomato} \equiv (\exists \pr{HTSHeading}.\set{\text{`0702'}})$ \\
		$\pr{GrapeTomato} \equiv (\exists \pr{HTSCode}.\set{\text{`07020010'}})$
			\> $\pr{Tomato} \sqsubseteq \pr{EdibleVegetable}$ \\
		$\pr{CherryTomato} \sqsubseteq \pr{Tomato}$
			\> $\pr{GrapeTomato} \sqsubseteq \pr{Tomato}$ \\
		$\pr{CherryTomato} \sqcap \pr{Bulk}
				\equiv (\exists \pr{TariffCharge}.\set{\text{\$0}})$
			\> $\pr{CherryTomato} \sqcap \pr{GrapeTomato} \sqsubseteq \bot$ \\
		$\pr{GrapeTomato} \sqcap \pr{Bulk}
				\equiv (\exists \pr{TariffCharge}.\set{\text{\$40}})$
			\> $\pr{Bulk} \sqcap \pr{Prepackaged} \sqsubseteq \bot$ \\
		$\pr{CherryTomato} \sqcap \pr{Prepackaged}
				\equiv (\exists \pr{TariffCharge}.\set{\text{\$50}})$ \\
		$\pr{GrapeTomato} \sqcap \pr{Prepackaged}
				\equiv (\exists \pr{TariffCharge}.\set{\text{\$100}})$ \\
		\\
		$\an{\ct{s_1}, \ct{c_1}} : \pr{ShpmtCommod}$
			\> $\an{\ct{s_1}, \text{`07020020'}} : \pr{ShpmtDeclHTSCode}$ \\
		$\an{\ct{s_1}, \ct{i_1}} : \pr{ShpmtImporter}$
			\> $\ct{c_1} : \pr{CherryTomato} \quad \ct{c_1} : \pr{Bulk}$ \\
		$\an{\ct{s_2}, \ct{c_2}} : \pr{ShpmtCommod}$
			\> $\an{\ct{s_2}, \text{`07020020'}} : \pr{ShpmtDeclHTSCode}$ \\
		$\an{\ct{s_2}, \ct{i_2}} : \pr{ShpmtImporter}$
			\> $\ct{c_2} : \pr{CherryTomato} \quad \ct{c_2} : \pr{Prepackaged}$ \\
		$\an{\ct{s_2}, portugal} : \pr{ShpmtCountry}$
			\> \\
		$\an{\ct{s_3}, \ct{c_3}} : \pr{ShpmtCommod}$    
			\> $\an{\ct{s_3}, \text{`07020010'}} : \pr{ShpmtDeclHTSCode}$ \\
		$\an{\ct{s_3}, \ct{i_3}} : \pr{ShpmtImporter}$
			\> $\ct{c_3} : \pr{GrapeTomato} \quad\, \ct{c_3} : \pr{Bulk}$ \\
		$\an{\ct{s_3}, portugal} : \pr{ShpmtCountry}$
			\> $\an{\ct{s_3}, \ct{p_1}} : \pr{ShpmtProducer}$
		\end{tabbing}

		\begin{center}
		*\ \ *\ \ *\ \ $t_{U_0}(b_{U_1}(\kb))$ *\ \ *\ \ *
		\end{center}
		\vspace{-.4cm}
		\begin{tabbing}
		fooooooooooooooooooooooooooooooooooooooooo\=fooo\=\kill
		$\pr{AdmissibleImporter}(I) \lpif{} \!\lpnot \pr{SuspectedBadGuy}(I).$ \\
		$\pr{SuspectedBadGuy}(\ct{i_1})$. \\
		$\pr{ApprovedImporterOf}(\ct{i_2}, C) \lpif \pr{EdibleVegetable}(C).$ \\
		$\pr{ApprovedImporterOf}(\ct{i_3}, C) \lpif \pr{GrapeTomato}(C).$
		\end{tabbing}

		\begin{center}
		*\ \ *\ \ *\ \ $t_{U_1}(b_{U_2}(\kb))$ *\ \ *\ \ *
		\end{center}
		\vspace{-.4cm}
		\begin{tabbing}
		fooooooooooooooooooooooooooooooooooooooooo\=fooo\=\kill
		$\pr{CommodCountry}(C, \ct{Country})
			\lpif \pr{ShpmtCommod}(S, C), \pr{ShpmtCountry}(S, \ct{Country}).$ \\
		$\pr{ExpeditableImporter}(C, I)
			\lpif \pr{AdmissibleImporter}(I), \pr{ApprovedImporterOf}(I, C).$ \\
		$\pr{EURegisteredProducer}
				\equiv (\exists \pr{RegisteredProducer}.\pr{EUCountry})$ \\
		$\pr{LowRiskEUCommodity}
				\equiv (\exists \pr{ExpeditableImporter}.\top)
				\sqcap (\exists \pr{CommodCountry}.\pr{EUCountry})$ \\
		\\
		$\an{\ct{p_1}, \ct{portugal}} : \pr{RegisteredProducer}$
			\> $\an{\ct{p_2}, \ct{slovakia}} : \pr{RegisteredProducer}$ \\
		$\ct{portugal} : \pr{EUCountry}$
			\> $\ct{slovakia} : \pr{EUCountry}$
		\end{tabbing}

		\begin{center}
		*\ \ *\ \ *\ \ $t_{U_2}(b_{U_3}(\kb))$ *\ \ *\ \ *
		\end{center}
		\vspace{-.4cm}
		\begin{tabbing}
		foooooooooooooooooooooooooooooooooooooooooooo\=\kill
		$\pr{CompliantShpmt}(S)
			\lpif \pr{ShpmtCommod}(S, C), \pr{HTSCode}(C, D), \pr{ShpmtDeclHTSCode}(S, D).$ \\
		$\pr{RandomInspection}(S) \lpif \pr{ShpmtCommod}(S, C), \pr{Random}(C).$ \\
		$\pr{PartialInspection}(S) \lpif \pr{RandomInspection}(S).$ \\
		$\pr{PartialInspection}(S)
			\lpif \pr{ShpmtCommod}(S, C), \lpnot \pr{LowRiskEUCommodity}(C).$ \\
		$\pr{FullInspection}(S) \lpif{} \!\lpnot \pr{CompliantShpmt}(S).$ \\
		$\pr{FullInspection}(S)
			\lpif \pr{ShpmtCommod}(S, C), \pr{Tomato}(C), \pr{ShpmtCountry}(S, \ct{slovakia}).$
		\end{tabbing}

		\longline
		\caption{Layers of the Hybrid Knowledge Base for Cargo Imports}
		\label{fig:layering}
	\end{figure}

	\begin{figure}
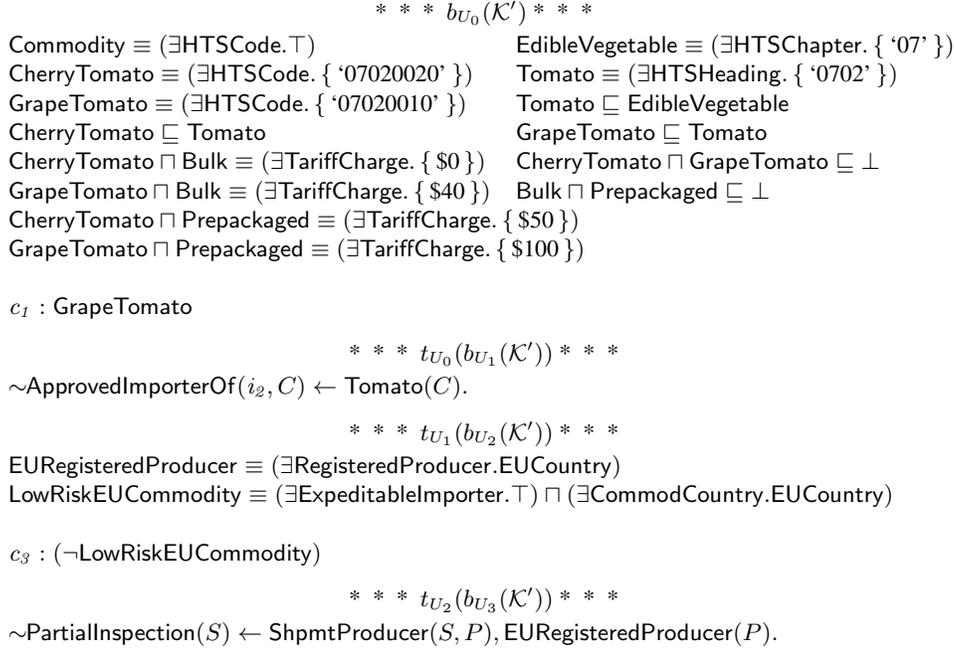

		\longline

		\begin{center}
		*\ \ *\ \ *\ \ $b_{U_0}(\kb')$ *\ \ *\ \ *
		\end{center}
		\vspace{-.4cm}
		\begin{tabbing}
		foooooooooooooooooooooooooooooooooooooooooo\=\kill
		$\pr{Commodity} \equiv (\exists \pr{HTSCode}.\top)$
			\> $\pr{EdibleVegetable} \equiv (\exists \pr{HTSChapter}.\set{\text{`07'}})$ \\
		$\pr{CherryTomato} \equiv (\exists \pr{HTSCode}.\set{\text{`07020020'}})$
			\> $\pr{Tomato} \equiv (\exists \pr{HTSHeading}.\set{\text{`0702'}})$ \\
		$\pr{GrapeTomato} \equiv (\exists \pr{HTSCode}.\set{\text{`07020010'}})$
			\> $\pr{Tomato} \sqsubseteq \pr{EdibleVegetable}$ \\
		$\pr{CherryTomato} \sqsubseteq \pr{Tomato}$
			\> $\pr{GrapeTomato} \sqsubseteq \pr{Tomato}$ \\
		$\pr{CherryTomato} \sqcap \pr{Bulk}
				\equiv (\exists \pr{TariffCharge}.\set{\text{\$0}})$
			\> $\pr{CherryTomato} \sqcap \pr{GrapeTomato} \sqsubseteq \bot$ \\
		$\pr{GrapeTomato} \sqcap \pr{Bulk}
				\equiv (\exists \pr{TariffCharge}.\set{\text{\$40}})$
			\> $\pr{Bulk} \sqcap \pr{Prepackaged} \sqsubseteq \bot$ \\
		$\pr{CherryTomato} \sqcap \pr{Prepackaged}
				\equiv (\exists \pr{TariffCharge}.\set{\text{\$50}})$ \\
		$\pr{GrapeTomato} \sqcap \pr{Prepackaged}
				\equiv (\exists \pr{TariffCharge}.\set{\text{\$100}})$ \\
		\\
		$\ct{c_1} : \pr{GrapeTomato}$
		\end{tabbing}

		\begin{center}
		*\ \ *\ \ *\ \ $t_{U_0}(b_{U_1}(\kb'))$ *\ \ *\ \ *
		\end{center}
		\vspace{-.4cm}
		\begin{tabbing}
		foooooooooooooooooooooooooooooooooooooooooo\=\kill
		$\lpnot \pr{ApprovedImporterOf}(\ct{i_2}, C) \lpif \pr{Tomato}(C).$
		\end{tabbing}

		\begin{center}
		*\ \ *\ \ *\ \ $t_{U_1}(b_{U_2}(\kb'))$ *\ \ *\ \ *
		\end{center}
		\vspace{-.4cm}
		\begin{tabbing}
		foooooooooooooooooooooooooooooooooooooooooo\=\kill
		$\pr{EURegisteredProducer}
				\equiv (\exists \pr{RegisteredProducer}.\pr{EUCountry})$ \\
		$\pr{LowRiskEUCommodity}
				\equiv (\exists \pr{ExpeditableImporter}.\top)
				\sqcap (\exists \pr{CommodCountry}.\pr{EUCountry})$ \\
		\\
		$\ct{c_3} : (\lnot \pr{LowRiskEUCommodity})$
		\end{tabbing}

		\begin{center}
		*\ \ *\ \ *\ \ $t_{U_2}(b_{U_3}(\kb'))$ *\ \ *\ \ *
		\end{center}
		\vspace{-.4cm}
		\begin{tabbing}
		foooooooooooooooooooooooooooooooooooooooooo\=\kill
		$\lpnot \pr{PartialInspection}(S)
			\lpif \pr{ShpmtProducer}(S, P), \pr{EURegisteredProducer}(P).$
		\end{tabbing}

		\longline
		\caption{Layers of the Update to the Hybrid Knowledge Base for Cargo Imports}
		\label{fig:update layering}
	\end{figure}
\end{extended}

\begin{example}[Updating the Cargo Import Knowledge Base]
	The hybrid knowledge base $\kb$ in Fig.~\ref{fig:import} has a single MKNF
	model $M$. We shortly summarise what is entailed by this model. First, since
	the shipments \ct{s_1}, \ct{s_2}, \ct{s_3} differ in the kind of tomatoes
	and their packaging, each of them is associated a different tariff charge.
	The HTS codes of commodities inside all three shipments match the declared
	HTS codes, so \mbox{\pr{CompliantShipment}(\ct{s_i})} is entailed for all
	$i$. The rules for importers imply that while both
	\mbox{$\pr{AdmissibleImporter}(\ct{i_2})$} and
	\mbox{$\pr{AdmissibleImporter}(\ct{i_3})$} are true,
	\mbox{$\pr{AdmissibleImporter}(\ct{i_1})$} is not true because \ct{i_1} is a
	suspected bad guy. It also follows that
	\mbox{$\pr{ApprovedImporterOf}(\ct{i_2}, \ct{c_2})$} and
	\mbox{$\pr{ApprovedImporterOf}(\ct{i_3}, \ct{c_3})$} hold and because of
	that \mbox{$\pr{ExpeditableImporter}(\ct{c_2}, \ct{i_2})$} and
	\mbox{$\pr{ExpeditableImporter}(\ct{c_3}, \ct{i_3})$} are also true. Both of
	these shipments come from a European country, so \ct{c_2} and \ct{c_3}
	belong to \pr{LowRiskEUCommodity}. But this is not true for \ct{c_1} since
	there is no expeditable importer for it. Consequently,
	\mbox{$\pr{PartialInspection}(\ct{s_1})$} holds.

	We now consider an update caused by several independent events in order to
	illustrate different aspects of our hybrid update semantics.

	Suppose that during the partial inspection of \ct{s_1}, grape tomatoes
	are found instead of cherry tomatoes. Second, we suppose that \ct{i_2}
	is no longer an approved importer for any kind of tomatoes due to a history
	of mis-filing. Third, due to rat infestation on the boat with shipment
	\ct{s_3}, \ct{c_3} is no longer considered a low risk commodity.
	Finally, due to workload constraints, partial inspections for shipments with
	commodities from a producer registered in a country of the European Union
	will be waived. These events lead to the following update $\kb' = \an{\ont',
	\prog'}$: where $\ont'$ contains \mbox{$\ct{c_1} : \pr{GrapeTomato}$} and
	\mbox{$\ct{c_3} : (\lnot \pr{LowRiskEUCommodity})$} as well as all TBox
	axioms from $\ont$,\footnote{We reinclude all TBox axioms in $\ont'$ in
	order to keep them static throughout the example.} and $\prog'$ contains the
	following rules:\footnote{We assume that all rule variables are DL-safe and
	rules are grounded prior to applying our theory.}
	\begin{tabbing}
		fooooooooooooooooooooooooooooooooooooooooooooo\=\kill
		$\lpnot \pr{ApprovedImporterOf}(\ct{i_2}, C) \lpif \pr{Tomato}(C).$ \\
		$\lpnot \pr{PartialInspection}(S)
			\lpif \pr{ShpmtProducer}(S, P), \pr{EURegisteredProducer}(P).$
	\end{tabbing}
	\begin{normal}
		Note that the splitting sequence $U$ defined in Example
		\ref{ex:import:splitting} is update-enabling for the dynamic hybrid
		knowledge base $\an{\kb, \kb'}$. This dynamic hybrid knowledge base has a
		single dynamic MKNF model $M'$ that entails the following:\footnote{A more
		complete explanation including technical details can be found at \\
		\url{http://centria.di.fct.unl.pt/~jleite/iclp11full.pdf}.} (1) Commodity
		$\ct{c_1}$ is no longer a member of \pr{CherryTomato} and its HTS code is
		now $\text{`07020010'}$. As a consequence,
		\mbox{$\pr{CompliantShpmt}(\ct{s_1})$} does not hold, so
		\mbox{$\pr{FullInspection}(\ct{s_1})$} holds. (2) Due to the rule update
		$\ct{i_2}$ is no longer an approved importer for $\ct{c_2}$. Hence,
		$\ct{c_2}$ is no longer a member of \pr{LowRiskEUCommodity} and
		\mbox{$\pr{PartialInspection}(\ct{s_2})$} holds. (3) A similar situation
		occurs with $\ct{c_3}$ because the update directly asserts that $\ct{c_3}$
		belongs to the complement of \pr{LowRiskEUCommodity}. But due to the
		update of \pr{PartialInspection}, the fact that $\ct{c_3}$ was produced by
		$\ct{p_1}$, and that \mbox{$\pr{EURegisteredProducer}(\ct{p_1})$} holds,
		\mbox{$\pr{PartialInspection}(\ct{s_3})$} is not true even though both
		\mbox{$\pr{ShpmtCommod}(\ct{s_3}, \ct{c_3})$} and \mbox{$\lpnot
		\pr{LowRiskEUCommodity}(\ct{c_3})$} are true.
	\end{normal}
	\begin{extended}%
		Note that the splitting sequence $U$ defined in Example
		\ref{ex:import:splitting} is update-enabling for the dynamic hybrid
		knowledge base $\an{\kb, \kb'}$. The four nonempty layers of $\kb$ are
		listed in Fig.~\ref{fig:layering}. The first layer, $b_{U_0}(\kb)$
		contains only ontology axioms, and so is $\ont$-reducible relative to
		$\emptyset$. The second and fourth layers ($t_{U_0}(b_{U_1}(\kb))$ and
		$t_{U_2}(b_{U_3}(\kb))$) contain only rules and so are $\prog$-reducible
		relative to $U_0$ and $U_2$, respectively. Finally, the third layer
		$t_{U_1}(b_{U_2}(\kb))$ contains a mixture of rules and ontology axioms,
		but all the rules are positive and all predicate symbols of all rule body
		literals belong to $U_1$, so the layer is $\ont$-reducible relative to
		$U_1$.

		The layers of the updating hybrid knowledge base $\kb'$ are shown in
		Fig.~$\ref{fig:update layering}$. It can be easily verified that they
		satisfy the same reducibility criteria, so $U$ is indeed an
		update-enabling sequence for $\an{\kb, \kb'}$. In order to arrive at a
		dynamic MKNF model of $\an{\kb, \kb'}$ with respect to $U$, a dynamic MKNF
		model of each layer is computed separately and models of previous layers
		serve to ``import'' information to the current layer.

		In our case, we first need to find the minimal change update model $X_0$
		of the first layer of $\kb$ updated by the first layer of $\kb'$. Due to
		the TBox axioms, this results in $\ct{c_1}$ no longer being a member of
		\pr{CherryTomato}. The HTS code of $\ct{c_1}$ also changes to
		$\text{`07020010'}$. Note that the conflict between old and new knowledge
		is properly resolved by the minimal change update semantics.

		Subsequently, the dynamic stable model semantics is used to find the
		dynamic MKNF model $X_1$ of the second layer of $\kb$ updated by the
		second layer of $\kb'$. The rule update results in $\ct{i_2}$ no longer
		being an approved importer for $\ct{c_2}$. As before, the conflict that
		arose is resolved by the rule update semantics.

		Given the dynamic MKNF models of the first two layers, the model of the
		third layer of $\kb$ is now different because $\ct{i_2}$ is no longer an
		expeditable importer of $\ct{c_2}$. As a consequence, $\ct{c_2}$ is no
		longer a member of the concept \pr{LowRiskEUCommodity}. Also, due to the
		update of the third layer, $\ct{c_3}$ is also not a member of
		\pr{LowRiskEUCommodity}. The conflicting situation was again resolved by
		the minimal change update operator and results in the dynamic MKNF model
		$X_2$ of the third layer.

		Finally, due to the changes in all three previous layers, the rules in the
		fourth layer now imply that $\pr{CompliantShpmt}(\ct{s_1})$ does not hold
		and, as a consequence, $\pr{FullInspection}(\ct{s_1})$ holds. Also,
		$\pr{PartialInspection}(\ct{s_2})$ holds because $\ct{c_2}$ is not a low
		risk commodity.  But even though $\ct{c_3}$ is also not a low risk
		commodity, $\pr{PartialInspection}(\ct{s_3})$ does not hold. This is due
		to the rule update of the fourth layer according to which the inspection
		of $\ct{s_3}$ must be waived because $\ct{s_3}$ comes from an EU
		registered producer.
	\end{extended}
\end{example}

\section{Discussion} \label{sect:discussion}

The class of updatable hybrid knowledge bases for which we defined an update
semantics in the previous section is closely related to multi-context systems
\cite{Brewka2007}. Each layer of a hybrid knowledge base relative to a
particular update-enabling splitting sequence can be viewed as a context
together with all its bridge rules. At the same time, the constraints we
impose guarantee that each such context either contains only rules, so the
context logic can be the stable model semantics, or it contains only DL axioms
so that first-order logic can be used as its logic. On the other hand,
different splitting sequences induce different multi-context systems, though
their overall semantics stays the same. We believe that a further study of
this close relationship may bring about new insights.

Another direction in which the proposed framework can be generalised is by
letting the ontology and rule update operators be given as parameters instead
of using a fixed pair.\footnote{We would like to thank the anonymous reviewer
for pointing this out.} This seems to have even more appeal given the fact
that no general consensus has been reached in the community regarding the
``right way'' to perform rule updates, and the situation with ontology update
operators also seems to be similar. Although Winslett's operator has been used
to deal with ABox updates \cite{Liu2006,Giacomo2006}, its use for dealing with
TBox updates has recently been criticised \cite{Calvanese2010,Slota2010a} and
a number of considerably different methods for dealing with TBox evolution
have been proposed \cite{Qi2006b,Qi2009,Yang2009,Calvanese2010,Wang2010}, many
tailored to a specific Description Logic.

To sum up, the contribution of this paper is twofold. First, we generalised
the splitting theorems for Logic Programs \cite{Lifschitz1994a} to the case of
Hybrid MKNF Knowledge Bases \cite{Motik2007}. This makes it possible to divide
a hybrid knowledge base into layers and guarantees that its overall semantics
can be reconstructed from the semantics of layers inside it. Second, we used
the theorem and related notions to identify a class of hybrid knowledge bases
for which we successfully defined an update semantics, based on a modular
combination of a classical and a rule update semantics. We showed that our
semantics properly generalises the semantics it is based on, particularly the
static semantics of Hybrid MKNF Knowledge Bases \cite{Motik2007}, the
classical minimal change update semantics \cite{Winslett1990}, and the refined
dynamic stable model semantics for rule updates \cite{Alferes2005}. We then
illustrated on an example motivated by a real world application how the
defined semantics deals with nontrivial updates, automatically resolving
conflicts and propagating new information across the hybrid knowledge base.

\bibliographystyle{acmtrans}
\bibliography{bibliography}

\newpage
\appendix

\begin{extended}

\section{Proofs of Auxiliary Propositions}

\subsection{Restricted MKNF Interpretations}

\begin{proposition} \label{prop:ent_relevant}
	Let $\phi$ be an MKNF sentence, $U \subseteq \lpre$ be a set of predicate
	symbols such that $U \supseteq \preds{\phi}$ and $\mstr$ be an MKNF structure.
	Then:
	\[
		\mstr \ent \phi \mlequiv \mstr[\restr{I}, \restr{M}, \restr{N}] \ent \phi
		\enspace.
	\]
\end{proposition}
\begin{proof*}
	We will prove by structural induction on $\phi$:
	\begin{enumerate}%
		\renewcommand{\labelenumi}{\arabic{enumi}$^\circ$}
		\item If $\phi$ is a ground atom $P(t_1, t_2, \dotsc, t_n)$, then $P \in
			\preds{\phi}$, so $P \in U$. The following chain of equivalences now
			proves the claim:
			\begin{align*}
				\mstr \ent \phi &\mlequiv P(t_1, t_2, \dotsc, t_n) \in I
					\mlequiv P(t_1, t_2, \dotsc, t_n) \in \restr{I} \\
				&\mlequiv \mstr[\restr{I}, \restr{M}, \restr{N}] \ent \phi
					\enspace;
			\end{align*}
		
		\item If $\phi$ is of the form $\lnot \psi$, then $\preds{\phi} =
			\preds{\psi}$, so $U \supseteq \preds{\psi}$. Hence we can use the
			inductive hypothesis for $\psi$ as follows:
			\begin{align*}
				\mstr \ent \phi &\mlequiv \mstr \nent \psi
					\mlequiv \mstr[\restr{I}, \restr{M}, \restr{N}] \nent \psi \\
				& \mlequiv \mstr[\restr{I}, \restr{M}, \restr{N}] \ent \phi \enspace;
			\end{align*}
		
		\item If $\phi$ is of the form $\phi_1 \land \phi_2$, then $\preds{\phi} =
			\preds{\phi_1} \cup \preds{\phi_2}$, so we easily obtain both $U
			\supseteq \preds{\phi_1}$ and $U \supseteq \preds{\phi_2}$. Applying the
			inductive hypothesis to $\phi_1$ and $\phi_2$ now yields the claim:
			\begin{align*}
				\mstr \ent \phi &\mlequiv \mstr \ent \phi_1 \land \mstr \ent \phi_2 \\
				&\mlequiv \mstr[\restr{I}, \restr{M}, \restr{N}] \ent \phi_1
					\land \mstr[\restr{I}, \restr{M}, \restr{N}] \ent \phi_2 \\
				&\mlequiv \mstr[\restr{I}, \restr{M}, \restr{N}] \ent \phi \enspace;
			\end{align*}

		\item If $\phi$ is of the form $\exists x : \psi$, then for any $c \in
			\Delta$, $\preds{\phi} = \preds{\psi} = \preds{\psi[c/x]}$, so $U
			\supseteq \preds{\psi[c/x]}$. Hence we can use the inductive hypothesis
			for the formulae $\psi[c/x]$ as follows:
			\begin{align*}
				\mstr \ent \phi
				&\mlequiv (\exists c \in \Delta)(\mstr \ent \psi[c/x]) \\
				&\mlequiv (\exists c \in \Delta)
					\left( \mstr[\restr{I}, \restr{M}, \restr{N}] \ent \psi[c/x] \right) \\
				&\mlequiv \mstr[\restr{I}, \restr{M}, \restr{N}] \ent \phi \enspace;
			\end{align*}

		\item If $\phi$ is of the form $\mk \psi$, then $\preds{\phi} =
			\preds{\psi}$, so $U \supseteq \preds{\psi}$. The claim now follows from
			the inductive hypothesis for $\psi$:
			\begin{align*}
				\mstr \ent \phi
				&\mlequiv \br{\forall J \in M}\br{\mstr[J, M, N] \ent \psi} \\
				&\mlequiv \br{\forall J \in M}\br{\mstr[\restr{J}, \restr{M}, \restr{N}] \ent \psi} \\
				&\mlequiv \br{\forall J \in \restr{M}}\br{\mstr[J, \restr{M}, \restr{N}] \ent \psi} \\
				&\mlequiv \mstr[\restr{I}, \restr{M}, \restr{N}] \ent \phi \enspace;
			\end{align*}
		
		\item If $\phi$ is of the form $\mnot \psi$, then $\preds{\phi} =
			\preds{\psi}$, so $U \supseteq \preds{\psi}$. The claim follows
			similarly as in the previous case:
			\begin{align*}
				\mstr \ent \phi
				&\mlequiv \br{\exists J \in N}\br{\mstr[J, M, N] \nent \psi} \\
				&\mlequiv \br{\exists J \in N}\br{\mstr[\restr{J}, \restr{M}, \restr{N}] \nent \psi} \\
				&\mlequiv \br{\exists J \in \restr{N}}\br{\mstr[J, \restr{M}, \restr{N}] \nent \psi} \\
				&\mlequiv \mstr[\restr{I}, \restr{M}, \restr{N}] \ent \phi \enspace. \qedhere
			\end{align*}
	\end{enumerate}
\end{proof*}

\begin{corollary} \label{cor:ent_relevant}
	Let $\thr$ be a set of formulae, $U$ be a set of predicate symbols such that
	$U \supseteq \preds{\thr}$ and $M, N \in \mint$ be such that they coincide on
	$U$. Then
	\[
		M \ent \thr \mlequiv N \ent \thr \enspace.
	\]
\end{corollary}
\begin{proof}
	We will prove the equivalence only in one direction, the proof of the second
	direction can be written analogically.

	Suppose that $M \ent \thr$. Then for every $\phi \in \thr$ and all $I \in M$
	we have $\mstr[I, M, M] \ent \phi$. We want to show that $N \ent \thr$.
	Let's pick some $\phi \in \thr$ and some $J \in N$. Since $\restr{M} =
	\restr{N}$, there must be some $I \in M$ such that $\restr{I} = \restr{J}$.
	By assumption, $U$ contains $\preds{\phi}$ and $\mstr[I, M, M] \ent \phi$,
	so Proposition \ref{prop:ent_relevant} yields
	\[
		\mstr[\restr{I}, \restr{M}, \restr{M}] \ent \phi \enspace.
	\]
	As mentioned above, $\restr{I} = \restr{J}$ and $\restr{M} = \restr{N}$, so
	\[
		\mstr[\restr{J}, \restr{N}, \restr{N}] \ent \phi \enspace.
	\]
	Another application of Proposition \ref{prop:ent_relevant} now yields
	$\mstr[J, N, N] \ent \phi$ and since $J$ and $\phi$ were chosen arbitrarily,
	we can conclude that $N \ent \thr$.
\end{proof}

\begin{proposition} \label{prop:restr:basic}
	Let $U$ be a set of predicate symbols and $M, N \in \mint$. Then the
	following implications hold:
	\begin{enumerate}
		\renewcommand{\labelenumi}{(\arabic{enumi})}
		\item If $M \subseteq N$, then $\restr{M} \subseteq \restr{N}$.
			\label{eq:prop:restr:basic:1}
		\item If $M = N$, then $\restr{M} = \restr{N}$.
			\label{eq:prop:restr:basic:2}
		\item If $M \subseteq N$ and $\restr{M} \subsetneq \restr{N}$, then $M
			\subsetneq N$. \label{eq:prop:restr:basic:3}
	\end{enumerate}
\end{proposition}
\begin{proof*}
	\vspace{-1.5em}
	\begin{enumerate}
		\renewcommand{\labelenumi}{(\arabic{enumi})}
		\item Follows by definition of $\restr{M}$ and $\restr{N}$.
		\item This is a direct consequence of \eqref{eq:prop:restr:basic:1}.
		\item Suppose $M \subseteq N$ and $\restr{M} \subsetneq \restr{N}$. Then
			there must be some $J \in N$ such that $\restr{J} \notin \restr{M}$.
			Consequently, $J \notin M$ and so $M$ must be a proper subset of $N$.
			\qedhere
	\end{enumerate}
\end{proof*}

\subsection{Saturated MKNF Interpretations}

An interesting class of MKNF interpretations are \emph{saturated} MKNF
interpretations. As we will see, all MKNF models of some formula or set of
formulae are saturated in a certain sense. Furthermore, a strengthened version
of Proposition \ref{prop:restr:basic} can be shown for saturated MKNF
interpretations, with implications replaced by equivalences. In this
subsection we formally define the class of saturated MKNF interpretations and
then we prove some of their properties.
\begin{definition}[Saturated MKNF Interpretation]
	Let $U$ be a set of predicate symbols and $M \in \mint$. We say an $M$ is
	\emph{saturated relative to $U$} if for every interpretation $I \in \foint$ the
	following holds:
	\[
		\text{If } \restr{I} \in \restr{M} \text{, then } I \in M.
	\]
\end{definition}
Interestingly, all MKNF models of a theory $\thr$ are saturated relative to
the set of predicate symbols relevant to $\thr$:

\begin{proposition} \label{prop:saturated:mknf}
	Let $U$ be a set of predicate symbols, $\thr$ be an MKNF theory such that $U
	\supseteq \preds{\thr}$ and $M$ be an MKNF model of $\thr$. Then $M$ is
	saturated relative to $U$.
\end{proposition}
\begin{proof}
	Suppose $M$ is not saturated relative to $U$. Then there is some $I \in \foint$
	such that $\restr{I} \in \restr{M}$ and $I \notin M$. Let $M' = M \cup
	\set{I}$. $M$ is an MKNF model of $\thr$, so by definition $\mstr[I', M', M]
	\nent \thr$ for some $I' \in M'$. But $\restr{I'} \in \restr{M'} = \restr{M}$,
	so there must be some $I'' \in M$ such that $\restr{I''} = \restr{I'}$. By two
	applications of Proposition \ref{prop:ent_relevant} we now obtain
	\[
		\mstr[I', M', M] \nent \thr
			\mlthen \mstr[\restr{I'}, \restr{M'}, \restr{M}] \nent \thr
			\mlthen \mstr[I'', M, M] \nent \thr \enspace.
	\]
	This is in conflict with the assumption that $M$ is an MKNF model of $\thr$.
\end{proof}

\begin{proposition} \label{prop:restr:sat:1}
	Let $U$ be a set of predicate symbols and $M, N \in \mint$ be such that $M$
	saturated relative to $U$. Then the following equivalences hold:
	\begin{enumerate}
		\renewcommand{\labelenumi}{(\arabic{enumi})}
		\item $M = N$ if and only if $M \subseteq N$ and $\restr{M} \supseteq
			\restr{N}$. \label{eq:prop:restr:sat:1:1}
		\item $M \subsetneq N$ if and only if $M \subseteq N$ and $\restr{M}
			\subsetneq \restr{N}$. \label{eq:prop:restr:sat:1:2}
	\end{enumerate}
\end{proposition}
\begin{proof*}
	\vspace{-1.5em}
	\begin{enumerate}
		\renewcommand{\labelenumi}{(\arabic{enumi})}
		\item The direct implication follows from Proposition
			\ref{prop:restr:basic}. We will prove the converse implication. Suppose
			$\restr{M} \supseteq \restr{N}$ and $I \in N$. We immediately obtain
			$\restr{I} \in \restr{M}$ and since $M$ is saturated relative to $U$, we
			can conclude that $I \in M$.
		\item For the direct implication suppose that $M \subsetneq N$. Then there
			is some $I \in N$ such that $I \notin M$. Since $M$ is saturated
			relative to $U$, we obtain that $\restr{I} \notin \restr{M}$.
			Consequently, $\restr{M}$ is a proper subset of $\restr{N}$. The
			converse implication is a cosequence of Proposition
			\ref{prop:restr:basic}\eqref{eq:prop:restr:basic:3}. \qedhere
	\end{enumerate}
\end{proof*}

\begin{proposition} \label{prop:restr:sat:2}
Let $U$ be a set of predicate symbols and $M, N \in \mint$ be such that $N$ is
saturated relative to $U$. Then:
\begin{enumerate}
	\renewcommand{\labelenumi}{(\arabic{enumi})}
	\item $M \subseteq N$ if and only if $\restr{M} \subseteq \restr{N}$.
		\label{eq:prop:restr:sat:2:1}
	\item If $\restr{M} \subsetneq \restr{N}$, then $M \subsetneq N$.
		\label{eq:prop:restr:sat:2:2}
\end{enumerate}
\end{proposition}
\begin{proof*}
	\vspace{-1.5em}
	\begin{enumerate}
		\renewcommand{\labelenumi}{(\arabic{enumi})}
		\item The direct implication follows from Proposition
			\ref{prop:restr:basic}\eqref{eq:prop:restr:basic:1}. We will prove the
			converse implication. Suppose $\restr{M} \subseteq \restr{N}$ and $I \in
			M$. We immediately obtain $\restr{I} \in \restr{M}$, hence also
			$\restr{I} \in \restr{N}$. Since $N$ is saturated relative to $U$, we
			can conclude that $I \in N$. Consequently, $M \subseteq N$.
		\item This is a consequence of \eqref{eq:prop:restr:sat:2:1} and
			Proposition \ref{prop:restr:basic}\eqref{eq:prop:restr:basic:3}.
			\qedhere
	\end{enumerate}
\end{proof*}

\begin{corollary} \label{cor:restr:sat}
	Let $U$ be a set of predicate symbols and $M, N$ be MKNF interpretations
	that are both saturated relative to $U$. Then the following equivalences
	hold:
	\begin{enumerate}
		\renewcommand{\labelenumi}{(\arabic{enumi})}
		\item $M \subseteq N$ if and only if $\restr{M} \subseteq \restr{N}$.
			\label{eq:cor:restr:sat:1}
		\item $M = N$ if and only if $\restr{M} = \restr{N}$.
			\label{eq:cor:restr:sat:2}
		\item $M \subsetneq N$ if and only if $\restr{M} \subsetneq \restr{N}$.
			\label{eq:cor:restr:sat:3}
	\end{enumerate}
\end{corollary}
\begin{proof*}
	\vspace{-1.5em}
	\begin{enumerate}
		\renewcommand{\labelenumi}{(\arabic{enumi})}
		\item Follows from Proposition
			\ref{prop:restr:sat:2}\eqref{eq:prop:restr:sat:2:1}.
		\item This is a consequence of \eqref{eq:cor:restr:sat:1}.
		\item This is a consequence of \eqref{eq:cor:restr:sat:1} and
			\eqref{eq:cor:restr:sat:2}. \qedhere
	\end{enumerate}
\end{proof*}

\begin{definition} \label{def:sigma}
	Let $U$ be a set of predicate symbols and $M \in \mint$. Then we introduce
	the following notation:
	\[
		\sat{M} = \Set{ I \in \foint | \restr{I} \in \restr{M} }
	\]
\end{definition}

\begin{proposition} \label{prop:sigma}
	Let $U$ be a set of predicate symbols and $M, N \in \mint$. Then the
	following conditions are equivalent:
	\begin{enumerate}
		\item $N = \sat{M}$;
		\item $N$ coincides with $M$ on $U$ and is saturated relative to $U$.
		\item $N$ is the greatest among all $N' \in \mint$ coinciding with $M$ on
			$U$;
	\end{enumerate}
	Furthermore, if $N$ satisfies one of the conditions above, then $M \subseteq
	N$.
\end{proposition}
\begin{proof}
	We will prove that 1. implies 2., 2. implies 3. and finally that 3. implies 1.

	Suppose $N = \sat{M}$. Then
	\[
		\restr{N} = \set{ \restr{I} | I \in \foint \land \restr{I} \in \restr{M} }
			= \restr{M} \enspace,
	\]
	so $N$ coincides with $M$ on $U$. Furthermore, any $I \in \foint$ with
	$\restr{I} \in \restr{N}$ must also satisfy $\restr{I} \in \restr{M}$. Thus,
	$I \in N$, so $N$ is saturated relative to $U$. This shows that 1. implies
	2.

	To show that 2. implies 3., suppose $N$ coincides with $M$ on $U$ and is
	saturated relative to $U$. Suppose $N' \in \mint$ coincides with $M$ on $U$
	and $I \in N'$. Then
	\[
		\restr{I} \in \restr{N'} = \restr{M} = \restr{N} \enspace,
	\]
	so, since $N$ is saturated relative to $U$, we can conclude that $I$ belongs
	to $N$. Consequently, $N'$ is contained in $N$, so $N$ is the greatest among
	all $N' \in \mint$ coinciding with $M$ on $U$.

	Finally, suppose $N$ is the greatest among all $N' \in \mint$ coinciding
	with $M$ on $U$. It can be easily seen that $\sat{M}$ coincides with
	$M$ on $U$, so $\sat{M}$ must be a subset of $N$. It remains to show
	that $N$ is a subset of $\sat{M}$. But that is an easy consequence of
	the fact that for any $I \in N$, $\restr{I}$ must belong to $\restr{M}$.

	It still remains to show that $M$ is a subset of $N$ if $N$ satisfies one of
	the above conditions. We already know that the conditions are equivalent, so
	we only need to consider one of them. So suppose $N = \sat{M}$
	(condition 1.). It can be easily seen from the definition of $\sat{M}$
	that every $I \in M$ belongs also to $N$. Hence, $M$ is a subset of $N$.
\end{proof}

\begin{proposition} \label{prop:sigma:repeated}
	Let $U_1, U_2$ be sets of predicate symbols and $M \in \mint$. Then
	\[
		\sat[U_2]{\sat[U_1]{M}} = \sat[U_1 \cap U_2]{M}
		\enspace.
	\]
\end{proposition}
\begin{proof*}
	Consider the following sequence of equivalences:
	\begin{align*}
		I \in \sat[U_2]{\sat[U_1]{M}} &\mlequiv \restr[U_2]{I} \in \restr[U_2]{\sat[U_1]{M}} \\
			&\mlequiv (\exists J \in \sat[U_1]{M})(\restr[U_2]{J} = \restr[U_2]{I}) \\
			&\mlequiv (\exists J \in \foint)((\exists K \in M)
				(\restr[U_1]{K} = \restr[U_1]{J} \land \restr[U_2]{J} = \restr[U_2]{I}) \\
			&\mlequiv (\exists K \in M)(\exists J \in \foint)
				(\restr[U_1]{J} = \restr[U_1]{K} \land \restr[U_2]{J} = \restr[U_2]{I}) \enspace.
	\end{align*}
	Moreover, we also obtain the following:
	\[
		I \in \sat[U_1 \cap U_2]{M}
			\mlequiv (\exists K \in M)(\restr[U_1 \cap U_2]{K} = \restr[U_1 \cap U_2]{I})
			\enspace.
	\]
	So it remains to show that
	\[
		(\exists J \in \foint)
			(\restr[U_1]{J} = \restr[U_1]{K} \land \restr[U_2]{J} = \restr[U_2]{I})
	\]
	holds if and only if
	\[
		\restr[U_1 \cap U_2]{K} = \restr[U_1 \cap U_2]{I} \enspace.
	\]
	Indeed, if such a $J$ exists, then for every ground atoms $p$ the following
	holds:
	\begin{align*}
		p \in \restr[U_1 \cap U_2]{K}
			&\mlequiv p \in K \land \preds{p} \subseteq U_1 \cap U_2 \\
			&\mlequiv (p \in K \land \preds{p} \subseteq U_1) \land \preds{p} \subseteq U_2
				\mlequiv p \in \restr[U_1]{K} \land \preds{p} \subseteq U_2 \\
			&\mlequiv p \in \restr[U_1]{J} \land \preds{p} \subseteq U_2
				\mlequiv (p \in J \land \preds{p} \subseteq U_2) \land \preds{p} \subseteq U_1 \\
			&\mlequiv p \in \restr[U_2]{J} \land \preds{p} \subseteq U_1
				\mlequiv p \in \restr[U_2]{I} \land \preds{p} \subseteq U_1 \\
			&\mlequiv p \in I \land \preds{p} \subseteq U_1 \cap U_2
				\mlequiv p \in \restr[U_1 \cap U_2]{I}
	\end{align*}
	On the other hand, if the other condition holds, then for
	\[
		J = \Set{p \in K | \preds{p} \subseteq U_1}
			\cup \Set{p \in I | \preds{p} \subseteq U_2}
	\]
	and any ground atom $p$ we obtain
	\begin{align*}
		p \in \restr[U_1]{J} 
			&\mlequiv p \in J \land \preds{p} \subseteq U_1 \\
			&\mlequiv (p \in K \land \preds{p} \subseteq U_1)
				\lor (p \in I \land \preds{p} \subseteq U_1 \cap U_2) \\
			&\mlequiv (p \in K \land \preds{p} \subseteq U_1)
				\lor p \in \restr[U_1 \cap U_2]{I} \\
			&\mlequiv (p \in K \land \preds{p} \subseteq U_1)
				\lor p \in \restr[U_1 \cap U_2]{K} \\
			&\mlequiv (p \in K \land \preds{p} \subseteq U_1)
				\lor (p \in K \land \preds{p} \subseteq U_1 \cap U_2) \\
			&\mlequiv p \in K \land \preds{p} \subseteq U_1 \\
			&\mlequiv p \in \restr[U_1]{K}
	\end{align*}
	and also
	\begin{align*}
		p \in \restr[U_2]{J} 
			&\mlequiv p \in J \land \preds{p} \subseteq U_2 \\
			&\mlequiv (p \in K \land \preds{p} \subseteq U_1 \cap U_2)
				\lor (p \in I \land \preds{p} \subseteq U_2) \\
			&\mlequiv p \in \restr[U_1 \cap U_2]{K}
				\lor (p \in I \land \preds{p} \subseteq U_2) \\
			&\mlequiv p \in \restr[U_1 \cap U_2]{I}
				\lor (p \in I \land \preds{p} \subseteq U_2) \\
			&\mlequiv (p \in I \land \preds{p} \subseteq U_1 \cap U_2)
				\lor (p \in I \land \preds{p} \subseteq U_2) \\
			&\mlequiv p \in I \land \preds{p} \subseteq U_2 \\
			&\mlequiv p \in \restr[U_2]{I} \enspace. \qedhere
	\end{align*}
\end{proof*}

\begin{proposition} \label{prop:sigma:restr}
	Let $U_1, U_2 \subseteq \lpre$ be sets of atoms such that $U_1 \subseteq
	U_2$ and $M \in \mint$. Then
	\[
		\restr[U_1]{\sat[U_2]{M}} = \restr[U_1]{M} \enspace.
	\]
\end{proposition}
\begin{proof}
	First suppose that $I$ belongs to $\restr[U_1]{\sat[U_2]{M}}$. Then for
	some $J \in \sat[U_2]{M}$ we have $I = \restr[U_1]{J}$, so
	\begin{equation} \label{eq:proof:prop:sigma:restr}
		I = \Set{p \in J | \preds{p} \subseteq U_1} \enspace.
	\end{equation}
	Also, since $J$ belongs to $\sat[U_2]{M}$, there must be some $K \in M$
	such that $\restr[U_2]{K} = \restr[U_2]{J}$, which means that for any ground
	atom $p$ with $\preds{p} \subseteq U_2$ we have $p \in J \mlequiv p \in K$.
	This, together with \eqref{eq:proof:prop:sigma:restr} and the assumption
	that $U_1$ is a subset of $U_2$, implies that
	\[
		I = \Set{p \in K | \preds{p} \subseteq U_1} \enspace.
	\]
	Thus, $I$ belongs to $\restr[U_1]{M}$.
	
	The converse inclusion follows from the fact that $M$ is a subset of
	$\sat[U_2]{M}$ (see Proposition \ref{prop:sigma}).
\end{proof}

\begin{lemma} \label{lemma:saturated:superset}
	Let $U_1, U_2$ be sets of predicate symbols such that $U_1$ is a subset of
	$U_2$ and let $X \in \mint$. If $X$ is saturated relative to $U_1$, then it
	is saturated relative to $U_2$.
\end{lemma}
\begin{proof}
	Suppose $X$ is saturated relative to $U_1$ and $I \in \foint$ is such that
	$\restr[U_2]{I} \in \restr[U_2]{X}$. We need to prove that $I$ belongs to
	$X$. We know $X$ contains some $J$ such that $\restr[U_2]{I} =
	\restr[U_2]{J}$. In other words, for every ground atom $p$ with $\preds{p}
	\subseteq U_2$ the following equivalence holds:
	\[
		p \in I \mlequiv p \in J \enspace.
	\]
	Since $U_1$ is a subset of $U_2$, every ground atom $p$ with $\preds{p}
	\subseteq U_1$ also satisfies the above equivalence. Thus, $\restr[U_1]{I} =
	\restr[U_1]{J}$ and we conclude that $\restr[U_1]{I}$ belongs to
	$\restr[U_1]{X}$. Since $X$ is saturated relative to $U_1$, $I$ must belong
	to $X$.
\end{proof}

\begin{lemma} \label{lemma:pmasplitting:restr sat disjoint}
	Let $U, V$ be disjoint sets of predicate symbols and $M \in \mint$ be
	nonempty. Then
	\[
		\restr[V]{\sat[U]{M}} = \restr[V]{\foint} \enspace.
	\]
\end{lemma}
\begin{proof}
	Since $\sat[U]{M}$ is a subset of $\foint$, left to right inclusion holds.
	Suppose $I$ some interpretation from $\restr[V]{\foint}$, i.e. $I$ contains
	only atoms with predicate symbols from $V$. Furthermore, take some $I' \in
	M$ and put $I'' = I \cup \restr{I'}$. Since $U$ is disjoint from $V$,
	$\restr{I''} = \restr{I'}$ and $\restr[V]{I''} = I$. This implies that $I''$
	belongs to $\sat{M}$, so $I$ belongs to $\restr[V]{\sat{M}}$.
\end{proof}

\subsection{Semi-saturated MKNF Interpretations}

There is also another class of MKNF interpretations for which a slightly
modified version of Proposition \ref{prop:restr:basic} holds. We introduce it
here and then show another result using the newly introduced notion.

\begin{definition}[Semi-saturated MKNF Interpretation]
	Let $U$ be a set of predicate symbols and $M \in \mint$. We say $M$ is
	\emph{semi-saturated relative to $U$} if for every interpretation $I \in
	\foint$ the following holds:
	\[
		\text{If } \restr{I} \in \restr{M}
			\text{ and } \restrc{I} \in \restrc{M}
			\text{, then } I \in M.
	\]
\end{definition}

\begin{proposition} \label{prop:sat_semisat}
	Let $U$ be a set of predicate symbols and $M, N$ be MKNF interpretations
	such that $M$ is saturated relative to $U$ and $N$ is saturated relative to
	$\lpre \setminus U$. Then $M \cap N$ is an MKNF interpretation that is
	semi-saturated relative to $U$, it coincides with $M$ on $U$ and with $N$ on
	$\lpre \setminus U$.
\end{proposition}
\begin{proof}
	We will first prove the following claim: For evey $I \in M$ and every $J \in
	N$ there exists some $K \in M \cap N$ such that $\restr{K} = \restr{I}$ and
	$\restrc{K} = \restrc{J}$. The reason
	this holds is that the sets $U$ and $(\lpre \setminus U)$ are disjoint.
	Let's take some $I \in M$ and some $J \in N$ and let
	\[
		K = \restr{I} \cup \restrc{J}
			= \Set{p \in I | \preds{p} \subseteq U}
				\cup \Set{p \in J | \preds{p} \subseteq \lpre \setminus U} \enspace.
	\]
	The following can now be derived:
	\begin{align*}
	\restr{K}
		&= \Set{p \in K | \preds{p} \subseteq U} \\
		&= \Set{p \in I | \preds{p} \subseteq U}
			\cup \Set{p \in J | \preds{p} \subseteq U \cap (\lpre \setminus U)} \\
		&= \Set{p \in I | \preds{p} \subseteq U}
		= \restr{I} \\
	\restrc{K}
		&= \Set{p \in K | \preds{p} \subseteq \lpre \setminus U} \\
		&= \Set{p \in I | \preds{p} \subseteq U \cap (\lpre \setminus U)}
			\cup \Set{p \in J | \preds{p} \subseteq \lpre \setminus U} \\
		&= \Set{p \in J | \preds{p} \subseteq \lpre \setminus U}
		= \restrc{J} \enspace.
	\end{align*}
	Consequently, $K$ must belong to $M$ and also to $N$ (because they are
	saturated relative to $U$ and $\lpre \setminus U$, respectively), so $K$
	also belongs to $M \cap N$.

	From the above it easily follows that $M \cap N$ is nonempty and that both
	$\restr{(M \cap N)} = \restr{M}$ and $\restrc{(M \cap N)}
	= \restrc{N}$ hold as well.

	It remains to show that $M \cap N$ is semi-saturated relative to $U$. Let $I
	\in \foint$ be such that $\restr{I}$ belongs to $\restr{(M \cap N)}$ and
	$\restrc{I}$ belongs to $\restrc{(M \cap
	N)}$. We need to prove that $I$ belongs to $M \cap N$. We know that
	$\restr{(M \cap N)}$ is a subset of $\restr{M}$ and since $M$ is saturated
	relative to $U$, we conclude that $I$ belongs to $M$. Similarly,
	$\restrc{(M \cap N)}$ is a subset of $\restr[\lpre
	\setminus U]{N}$ and since $N$ is saturated relative to $\lpre \setminus U$,
	we conclude that $I$ belongs to $N$. Hence, $I$ belongs to $M \cap N$.
\end{proof}

\begin{proposition} \label{prop:restr:semisat:1}
	Let $U$ be a set of predicate symbols and $M, N \in \mint$ be such that $M$
	is semi-saturated relative to $U$. Then the following equivalences hold:
	\begin{enumerate}
		\renewcommand{\labelenumi}{(\arabic{enumi})}
		\item $M = N$ if and only if $M \subseteq N$ and $\restr{M} \supseteq
			\restr{N}$ and $\restrc{M} \supseteq \restr[\lpre
			\setminus U]{N}$. \label{eq:prop:restr:semisat:1:1}
		\item $M \subsetneq N$ if and only if $M \subseteq N$ and either
			$\restr{M} \subsetneq \restr{N}$ or $\restrc{M}
			\subsetneq \restrc{N}$.
			\label{eq:prop:restr:semisat:1:2}
	\end{enumerate}
\end{proposition}
\begin{proof*}
	\vspace{-1.5em}
	\begin{enumerate}
		\renewcommand{\labelenumi}{(\arabic{enumi})}
		\item The direct implication follows from Proposition
			\ref{prop:restr:basic}\eqref{eq:prop:restr:basic:1}. We will prove the
			converse implication. Suppose $\restr{M} \supseteq \restr{N}$,
			$\restrc{M} \supseteq \restrc{N}$
			and $I \in N$. We immediately obtain $\restr{I} \in \restr{M}$ and
			$\restrc{I} \in \restrc{M}$ and
			since $M$ is semi-saturated relative to $U$, we can conclude that $I \in
			M$.
		\item For the direct implication suppose that $M \subsetneq N$. Then there
			is some $I \in N$ such that $I \notin M$. Since $M$ is semi-saturated
			relative to $U$, we obtain that $\restr{I} \notin \restr{M}$ or
			$\restrc{I} \notin \restrc{M}$.
			Consequently, either $\restr{M}$ is a proper subset of $\restr{N}$ or
			$\restrc{M}$ is a proper subset of $\restr[\lpre
			\setminus U]{N}$. The converse implication is a cosequence of
			Proposition \ref{prop:restr:basic}\eqref{eq:prop:restr:basic:3}.
			\qedhere
	\end{enumerate}
\end{proof*}

\begin{proposition} \label{prop:restr:semisat:2}
	Let $U$ be a set of predicate symbols and $M, N \in \mint$ be such that $N$
	is semi-saturated relative to $U$. Then:
	\begin{enumerate}
		\renewcommand{\labelenumi}{(\arabic{enumi})}
		\item $M \subseteq N$ if and only if $\restr{M} \subseteq \restr{N}$ and
			$\restrc{M} \subseteq \restrc{N}$.
			\label{eq:prop:restr:semisat:2:1}
		\item If $\restr{M} \subseteq \restr{N}$ and $\restrc{M}
			\subseteq \restrc{N}$ and at least one of the
			inclusions is proper, then $M \subsetneq N$.
			\label{eq:prop:restr:semisat:2:2}
	\end{enumerate}
\end{proposition}
\begin{proof*}
	\vspace{-1.5em}
	\begin{enumerate}
		\renewcommand{\labelenumi}{(\arabic{enumi})}
		\item The direct implication follows from Proposition
			\ref{prop:restr:basic}\eqref{eq:prop:restr:basic:1}. We will prove the
			converse implication. Suppose $\restr{M} \subseteq \restr{N}$,
			$\restrc{M} \subseteq \restrc{N}$
			and $I \in M$. We immediately obtain $\restr{I} \in \restr{M}$, hence
			also $\restr{I} \in \restr{N}$. Similarly, $\restrc{I}
			\in \restrc{M}$ and, consequently, $\restr[\lpre
			\setminus U]{I} \in \restrc{N}$. Since $N$ is
			semi-saturated relative to $U$, we can conclude that $I \in N$.
			Consequently, $M \subseteq N$.
		\item This is a consequence of \eqref{eq:prop:restr:semisat:2:1} and
			Proposition \ref{prop:restr:basic}\eqref{eq:prop:restr:basic:3}.
			\qedhere
	\end{enumerate}
\end{proof*}

\begin{corollary} \label{cor:restr:semisat}
	Let $U$ be a set of predicate symbols, $M, N$ be MKNF interpretations
	semi-saturated relative to $U$. Then the following equivalences hold:
	\begin{enumerate}
		\renewcommand{\labelenumi}{(\arabic{enumi})}
		\item $M \subseteq N$ if and only if $\restr{M} \subseteq \restr{N}$ and
			$\restrc{M} \subseteq \restrc{N}$.
			\label{eq:cor:restr:semisat:1}
		\item $M = N$ if and only if $\restr{M} = \restr{N}$ and $\restr[\lpre
			\setminus U]{M} = \restrc{N}$.
			\label{eq:cor:restr:semisat:2}
		\item $M \subsetneq N$ if and only if $\restr{M} \subseteq \restr{N}$ and
			$\restrc{M} \subseteq \restrc{N}$
			and at least one of the inclusions is proper.
			\label{eq:prop:restr:semisat:2:3}
	\end{enumerate}
\end{corollary}
\begin{proof*}
	\vspace{-1.5em}
	\begin{enumerate}
		\renewcommand{\labelenumi}{(\arabic{enumi})}
		\item The direct implication follows directly from definitions of
			$\restr{M}$, $\restr{N}$, $\restrc{M}$ and
			$\restrc{N}$. For the converse implication let $I \in
			M$. Then $\restr{I} \in \restr{M} \subseteq \restr{N}$ and $\restr[\lpre
			\setminus U]{I} \in \restrc{M} \subseteq \restr[\lpre
			\setminus U]{N}$. Since $N$ is semi-saturated, we conclude $I \in N$.
		\item This is a consequence of \eqref{eq:prop:restr:semisat:2:1}.
		\item This is a consequence of \eqref{eq:prop:restr:semisat:2:1} and
			\eqref{eq:prop:restr:semisat:2:2}. \qedhere
	\end{enumerate}
\end{proof*}

\begin{proposition} \label{prop:semisat_exists}
	Let $U$ be a set of predicate symbols and $M_1, M_2$ be MKNF
	interpretations. Then there exists the greatest MKNF interpretation $N$ that
	coincides with $M_1$ on $U$ and with $M_2$ on $\lpre \setminus U$.
	Furthermore, $N$ is semi-saturated relative to $U$ and $M_1 \cap M_2
	\subseteq N$.
\end{proposition}
\begin{proof}
	Let $N_1 = \sat{M_1}$, $N_2 = \satc{M_2}$ and $N =
	N_1 \cap N_2$. The claim now follows by Propositions \ref{prop:sigma} and
	\ref{prop:sat_semisat}.
\end{proof}

\subsection{Sequence-saturated MKNF Interpretations}

\begin{definition}[Saturation Sequence, Difference Sequence]
	A \emph{saturation sequence} is a sequence $\an{S_\alpha}_{\alpha < \mu}$ of
	pairwise disjoint sets of predicate symbols such that $\bigcup_{\alpha <
	\mu} S_\alpha = \lpre$.
\end{definition}

\begin{definition}[Sequence Saturated MKNF Interpretation]
	Let $\an{S_\alpha}_{\alpha < \mu}$ be  a saturation sequence and $M \in
	\mint$. We say $M$ is \emph{sequence-saturated relative to $S$} if for every
	interpretation $I \in \foint$ the following holds:
	\[
		\text{If } \restr[S_\alpha]{I} \in \restr[S_\alpha]{M}
			\text{ for every ordinal } \alpha < \mu \text{, then } I \in M.
	\]
\end{definition}

\begin{proposition} \label{prop:seqsat:equivalence}
	Let $S = \an{S_\alpha}_{\alpha < \mu}$ be a saturation sequence and $M \in
	\mint$. Then the following conditions are equivalent:
	\begin{enumerate}
		\item $M$ is sequence-saturated relative to $S$.
		\item $M = \bigcap_{\alpha < \mu} \sat[S_\alpha]{M}$.
		\item $M = \bigcap_{\alpha < \mu} X_\alpha$ and for any $\alpha < \mu$,
			$X_\alpha$ is saturated relative to $S_\alpha$.
	\end{enumerate}
\end{proposition}
\begin{proof}
	We first prove that 1. implies 2. Suppose $M$ is sequence-saturated relative
	to $S$. It follows from Proposition \ref{prop:sigma} that $M$ is a subset
	of $\sat[S_\alpha]{M}$ for any $\alpha < \mu$. Thus, $M$ is a subset of
	$\bigcap_{\alpha < \mu} \sat[S_\alpha]{M}$. To show that the converse
	inclusion holds as well, take some Herbrand interpretation $I$ from
	$\bigcap_{\alpha < \mu} \sat[S_\alpha]{M}$. For any $\alpha < \mu$,
	Proposition \ref{prop:sigma} implies that $\sat[S_\alpha]{M}$ coincides with
	$M$ on $S_\alpha$. Thus, since $I$ belongs to $\sat[S_\alpha]{M}$,
	$\restr[S_\alpha]{I}$ belongs to $\restr[S_\alpha]{M}$. Since $M$ is
	sequence-saturated relative to $S$, this implies that $I$ belongs to $M$.

	The implication from 2. to 3. is immediate by putting $X_\alpha =
	\sat[S_\alpha]{M}$ and observing that, by Proposition \ref{prop:sigma},
	$X_\alpha$ is saturated relative to $S_\alpha$.

	Finally, suppose that 3. holds and $I$ is a Herbrand interpretation such
	that for any $\alpha < \mu$, $\restr[S_\alpha]{I}$ belongs to
	\[
		\restr[S_\alpha]{M}
			= \restr[S_\alpha]{\left( \bigcap_{\beta < \mu} X_\beta \right)}
			\subseteq \restr[S_\alpha]{X_\alpha}
	\]
	Since $X_\alpha$ is saturated relative to $S_\alpha$, we conclude that $I$
	belongs to $X_\alpha$. The choice of $\alpha < \mu$ was arbitrary so we have
	proven that $I$ belongs to $M$.
\end{proof}

\begin{proposition} \label{prop:seqsat:int sigma}
	Let $S = \an{S_\alpha}_{\alpha < \mu}$ be a saturation sequence,
	$\an{X_\alpha}_{\alpha < \mu}$ be a sequence of MKNF interpretations such
	that for all $\alpha < \mu$, $X_\alpha$ is saturated relative to $S_\alpha$,
	and $M = \bigcap_{\alpha < \mu} X_\alpha$. Then $M$ is nonempty and for all
	$\alpha < \mu$, $X_\alpha = \sat[S_\alpha]{M}$.
\end{proposition}
\begin{proof}
	Pick some $\alpha < \mu$ and some $I \in X_\alpha$. We will prove that $I$
	belongs to $\sat[S_\alpha]{M}$. Let $I_\alpha = I$ and for all $\beta <	\mu$
	such that $\beta \neq \alpha$, let $I_\beta$ be any member of $X_\beta$
	(note that $X_\beta$ is nonempty because it is an MKNF interpretation). Now
	put
	\[
		J = \bigcup_{\beta < \mu} \restr[S_\beta]{I_\beta} \enspace.
	\]
	To see that $J$ belongs to $M$, take some $\beta < \mu$ and observe that
	$\restr[S_\beta]{J} = \restr[S_\beta]{I_\beta} \in
	\restr[S_\beta]{X_\beta}$. Since $X_\beta$ is saturated relative to
	$S_\beta$, this implies that $J$ belongs to $X_\beta$. Hence, $J$ belongs to
	$M$. Moreover, $\restr[S_\alpha]{I} = \restr[S_\alpha]{J}$, so
	$\restr[S_\alpha]{I}$ belongs to $\restr[S_\alpha]{M}$. Thus, $I$ belongs to
	$\sat[S_\alpha]{M}$.

	For the converse inclusion, suppose that $I$ belongs to $\sat[S_\alpha]{M}$
	for some $\alpha < \mu$. Then $\restr[S_\alpha]{I}$ belongs to
	$\restr[S_\alpha]{M}$. Also,
	\[
		\restr[S_\alpha]{M}
			= \restr[S_\alpha]{ \left( \bigcap_{\beta < \mu} X_\beta \right) }
			\subseteq \restr[S_\alpha]{X_\alpha} \enspace.
	\]
	Thus, $\restr[S_\alpha]{I}$ belongs to $\restr[S_\alpha]{X_\alpha}$ and
	since $X_\alpha$ is saturated relative to $S_\alpha$, $I$ belongs to
	$X_\alpha$.

	Finally, it follows from $X_\alpha = \sat[S_\alpha]{M}$ and the fact that
	$X_\alpha$ is nonempty that $M$ must also be nonempty.
\end{proof}

\subsection{Properties of Subjective Formulae}

\begin{proposition}[Models of Subjective Formulae] \label{prop:sub_s5}
	Let $\phi$ be a subjective MKNF formula, $I_1, I_2$ be Herbrand
	interpretations and $M, N \in \mint$. Then
	\[
	\mstr[I_1, M, N] \ent \phi \mlequiv \mstr[I_2, M, N] \ent \phi
	\]
\end{proposition}
\begin{proof}
	Follows directly from Def.~\ref{def:mknf:models} and the fact that the
	valuation of a subjective formula in a structure $\mstr$ is independent of
	$I$.
\end{proof}

This property of subjective formulae gives rise to the following shortcut
notation which simplifies many of the following formalizations.

\begin{definition}[Satisfiability for Subjective Formulae] \label{def:mknf:sub_sat}
	Let $\phi$ be a subjective formula, $\thr$ be a set of subjective formulae
	and $M, N \in \mint$. We introduce the following notation:
	\begin{align*}
		\mstr[M, N] \ent \phi &\overset{\textrm{def}}{\mlequiv}
			(\exists I \in \foint)(\mstr \ent \phi) \\
		\mstr[M, N] \ent \thr &\overset{\textrm{def}}{\mlequiv}
			(\forall \phi \in \thr)(\mstr[M, N] \ent \phi)
	\end{align*}
\end{definition}

The following result relates the introduced shortcut notation to the notions
of S5 and MKNF models.

\begin{proposition} \label{prop:sub_shortcut}
	Let $\thr$ be a set of subjective formulae and $M \in \mint$. Then $M \ent
	\thr$ if and only if $\mstr[M, M] \ent \thr$. Furthermore, $M$ is an MKNF
	model of $\thr$ if and only if $M$ is an S5 model of $\thr$ and for every
	MKNF interpretation $M' \supsetneq M$ it holds that $\mstr[M', M] \nent
	\thr$.
\end{proposition}
\begin{proof}
	Straightforward by Definitions~\ref{def:mknf:models} and
	\ref{def:mknf:sub_sat} and Proposition~\ref{prop:sub_s5}.
\end{proof}

\begin{proposition} \label{prop:sub:ent_relevant}
	Let $\phi$ be a subjective formula, $U \subseteq \lpre$ be a set of
	predicate symbols such that $U \supseteq \preds{\phi}$ and $M, N \in \mint$.
	Then:
	\[
		\mstr[M, N] \ent \phi \mlequiv \mstr[\restr{M}, \restr{N}] \ent \phi
		\enspace.
	\]
\end{proposition}
\begin{proof*}
	By Definition~\ref{def:mknf:sub_sat} we have
	\[
		\mstr[M, N] \ent \phi \mlequiv (\exists I \in \foint)(\mstr \ent \phi)
		\enspace.
	\]
	By Proposition \ref{prop:ent_relevant} we can equivalently rewrite the right
	hand side into
	\[
		(\exists I \in \foint)
			\left( \mstr[\restr{I}, \restr{M}, \restr{N}] \ent \phi \right) \enspace.
	\]
	Furthermore, since $\phi$ is subjective, we can use Proposition
	\ref{prop:sub_s5} to further rewrite the previous formula into
	\[
		(\exists I \in \foint)
			\left( \mstr[I, \restr{M}, \restr{N}] \ent \phi \right)
	\]
	which is by Definition~\ref{def:mknf:sub_sat} equivalent to
	\[
		\mstr[\restr{M}, \restr{N}] \ent \phi \enspace. \qedhere
	\]
\end{proof*}

\begin{corollary} \label{cor:sub:ent_relevant}
	Let $\phi$ be a subjective formula, $U \subseteq \lpre$ be a set of
	predicate symbols such that $U \supseteq \preds{\phi}$ and $M, M', N, N' \in
	\mint$ be such that $M$ coincides with $M'$ on $U$ and $N$ coincides with
	$N'$ on $U$. Then
	\[
		\mstr[M, N] \ent \phi \mlequiv \mstr[M', N'] \ent \phi \enspace.
	\]
\end{corollary}
\begin{proof*}
	By assumptions we know that $\restr{M} = \restr{M'}$ and $\restr{N} =
	\restr{N'}$. Proposition \ref{prop:sub:ent_relevant} now yields:
	\begin{align*}
		\mstr[M, N] \ent \phi
			&\mlequiv \mstr[\restr{M}, \restr{N}]
				\ent \phi \mlequiv \mstr[\restr{M'}, \restr{N'}] \ent \phi \\
			&\mlequiv \mstr[M', N'] \ent \phi \enspace. \qedhere
	\end{align*}
\end{proof*}

\section{Proof of Splitting Set Theorem for Hybrid MKNF Knowledge Bases} \label{app:splitting theorem}

\begin{remark}
	Note that whenever $U$ is a splitting set for a hybrid knowledge base $\kb =
	\an{\ont, \prog}$, the following can be easily shown to hold:
	\begin{align*}
		\preds{b_U(\ont)} &\subseteq U \enspace, \\
		\preds{b_U(\prog)} &\subseteq U \enspace, \\
		\preds{\pi(b_U(\kb))} &\subseteq U \enspace, \\
		\preds{t_U(\ont)} &\subseteq \lpre \setminus U \enspace.
	\end{align*}
	Also note that the heads of rules in $t_U(\prog)$ contain only predicate
	symbols from $\lpre \setminus U$ while their bodies may also contain
	predicate symbols from $U$. However, for any $X \in \mint$, the above
	defined reducts $e_U(\prog, X)$ and $e_U(\kb, X)$ can be shown to mention
	only atoms not belonging to $U$:
	\[
		\preds{\pi(e_U(\kb, X))} \subseteq \lpre \setminus U \enspace.
	\]
	Please stay warned that these rather basic observations will be used in the
	following text (especially in the proofs) without further notice or
	reference.
\end{remark}

\begin{definition}[Generalised Splitting Set Reduct]
	Let $U$ be a splitting set for a hybrid knowledge base $\kb = \an{\ont,
	\prog}$ and $X, X' \in \mint$. The \emph{generalised splitting set reduct of
	$\kb$ relative to $U$ and $\mstr[X', X]$} is a hybrid knowledge base
	$e_U(\kb, X', X) = \an{t_U(\ont), e_U(\prog, X', X)}$ where $e_U(\prog, X',
	X)$ consists of all rules $r'$ such that there exists a rule $r \in
	t_U(\prog)$ satisfying the following conditions:
	\begin{align}
		H(r') &= H(r) \enspace, \label{eq:def:generalised splitting reduct:1} \\
		B(r') &= \Set{ L \in B(r) | \preds{L} \subseteq \lpre \setminus U }
			\enspace, \label{eq:def:generalised splitting reduct:2} \\
		\mstr[X', X] &\ent \pi(B(r) \setminus B(r'))
			= \pi(\Set{ L \in B(r) | \preds{L} \subseteq U }) \enspace.
			\label{eq:def:generalised splitting reduct:3}	
	\end{align}
\end{definition}

\begin{remark}
	Note that for every hybrid knowledge base $\kb$ and every $X \in \mint$ the
	following holds: $e_U(\kb, X) = e_U(\kb, X, X)$. This will be heavily used
	in the following proofs.
\end{remark}

\begin{lemma} \label{lemma:splitting:aux:1}
	Let $U$ be a splitting set for a hybrid knowledge base $\kb$ and $D, D', E,
	E', F, F', G, G' \in \mint$ be such that the following conditions are
	satisfied:
	\begin{enumerate}
		\itemsep=0pt
		\item $\restr[U]{E} = \restr[U]{D}$ and $\restr[U]{E'} = \restr[U]{D'}$;
		\item $\restrc{F} = \restrc{D}$ and
			$\restrc{F'} = \restrc{D'}$;
		\item $\restr[U]{G} = \restr[U]{D}$ and $\restr[U]{G'} = \restr[U]{D'}$.
	\end{enumerate}
	Then:
	\[
		\mstr[D', D] \ent \pi(\kb) \mlthen \mstr[E', E] \ent \pi(b_U(\kb))
			\land \mstr[F', F] \ent \pi(e_U(\kb, G', G)) \enspace.
	\]
\end{lemma}
\begin{proof}
	Suppose that $\mstr[D', D] \ent \pi(\kb)$. Since $\pi(b_U(\kb)) \subseteq
	\pi(\kb)$, we immediately obtain $\mstr[D', D] \ent \pi(b_U(\kb))$.
	Furthermore, from Corollary \ref{cor:sub:ent_relevant} we now obtain
	$\mstr[E', E] \ent \pi(b_U(\kb))$.

	It remains to show that $\mstr[F', F] \ent \pi(e_U(\kb, G', G))$. Let $\kb =
	\an{\ont, \prog}$. We know that $\pi(e_U(\kb, G', G))$ consists of two sets:
	$\pi(t_U(\ont))$ and $\pi(e_U(\prog, G', G))$. Since $\pi(t_U(\ont))
	\subseteq \pi(\kb)$, we conclude that $\mstr[D', D] \ent \pi(t_U(\ont))$.
	Moreover, Corollary \ref{cor:sub:ent_relevant} now implies that $\mstr[F',
	F] \ent \pi(t_U(\ont))$. Now take some rule $r' \in e_U(\kb, G', G)$. If
	$\mstr[F', F] \nent \pi(B(r'))$, then $\mstr[F', F] \ent \pi(r')$. On the
	other hand, if $\mstr[F', F] \ent \pi(B(r'))$, then Corollary
	\ref{cor:sub:ent_relevant} implies that $\mstr[D', D] \ent \pi(B(r'))$.
	Moreover, by the definition of $e_U(\kb, G', G)$, there must be some rule $r
	\in \prog$ such that $H(r') = H(r)$ and $\mstr[G', G] \ent \pi(B(r)
	\setminus B(r'))$. From the last property and Corollary
	\ref{cor:sub:ent_relevant} we obtain $\mstr[D', D] \ent \pi(B(r) \setminus
	B(r'))$.  So $\mstr[D', D] \ent \pi(B(r))$ and since $\mstr[D', D] \ent
	\pi(\kb)$ and $\pi(\kb)$ contains $\pi(r)$, we conclude that $\mstr[D', D]
	\ent \pi(H(r))$.  Consequently, since $H(r') = H(r)$, by Corollary
	\ref{cor:sub:ent_relevant} we obtain $\mstr[F', F] \ent \pi(H(r'))$ and so
	$\mstr[F', F] \ent \pi(r')$. The choice of $r'$ was arbitrary, so we have
	also proven that $\mstr[F', F] \ent \pi(e_U(\prog, G', G))$.
\end{proof}

\begin{lemma} \label{lemma:splitting:aux:2}
	Let $U$ be a splitting set for a hybrid knowledge base $\kb$ and $D, D', E,
	E', F, F', G, G' \in \mint$ be such that the following conditions are
	satisfied:
	\begin{enumerate}
		\itemsep=0pt
		\item $\restr[U]{E} = \restr[U]{D}$ and $\restr[U]{E'} = \restr[U]{D'}$;
		\item $\restrc{F} = \restrc{D}$ and
			$\restrc{F'} = \restrc{D'}$;
		\item $\restr[U]{G} = \restr[U]{D}$ and $\restr[U]{G'} = \restr[U]{D'}$.
	\end{enumerate}
	Then:
	\[
		\mstr[E', E] \ent \pi(b_U(\kb)) \land \mstr[F', F] \ent \pi(e_U(\kb, G', G))
			\mlthen \mstr[D', D] \ent \pi(\kb) \enspace.
	\]
\end{lemma}
\begin{proof*}
	Take some $\phi \in \pi(\kb)$. We consider three cases depending on which
	part of $\kb$ this formula originates from:
	\begin{enumerate}
		\renewcommand{\labelenumi}{\alph{enumi})}
		\item If $\phi$ belongs to $\pi(b_U(\kb))$, then $\preds{\phi} \subseteq
			U$, so we can use Corollary \ref{cor:sub:ent_relevant} to infer
			$\mstr[D', D] \ent \phi$ from $\mstr[E', E] \ent \phi$.

		\item If $\phi$ belongs to $\pi(t_U(\ont))$, then $\phi$ also belongs to
			$\pi(e_U(\kb, G', G))$ and $\preds{\phi} \subseteq \lpre \setminus U$,
			so by Corollary \ref{cor:sub:ent_relevant} we can infer $\mstr[D', D]
			\ent \phi$ from $\mstr[F', F] \ent \phi$.

		\item If $\phi$ belongs to $\pi(t_U(\prog))$, then $\phi = \pi(r)$ for
			some rule $r$. If $\mstr[D', D] \nent \pi(B(r))$, then $\mstr[M', M]
			\ent \pi(r)$ and we are finished. On the other hand, if $\mstr[D', D]
			\ent \pi(B(r))$, then let $B = \set{ L \in B(r) | \preds{L} \subseteq U
			}$. We have $\mstr[D', D] \ent \pi(B)$ and by Corollary
			\ref{cor:sub:ent_relevant} we can conclude that $\mstr[G', G] \ent
			\pi(B)$. Consequently, $e_U(\prog, G', G)$ contains a rule $r'$ such
			that $H(r') = H(r)$ and $B(r') = B(r) \setminus B$. We know that
			$\mstr[F', F] \ent \pi(r')$ and since $\preds{r'} \subseteq \lpre
			\setminus U$, we can use Corollary \ref{cor:sub:ent_relevant} to infer
			$\mstr[D', D] \ent \pi(r')$. Furthermore, $\mstr[D', D] \ent
			\pi(B(r'))$, so $\mstr[D', D] \ent \pi(H(r'))$. Consequently, $\mstr[D',
			D] \ent \pi(r)$, which is the desired result. \qedhere
	\end{enumerate}
\end{proof*}

\begin{proposition} \label{prop:splitting:aux}
	Let $U$ be a splitting set for a hybrid knowledge base $\kb$ and $D, D', E,
	E', F, F', G, G' \in \mint$ be such that the following conditions are
	satisfied:
	\begin{enumerate}
		\itemsep=0pt
		\item $\restr[U]{E} = \restr[U]{D}$ and $\restr[U]{E'} = \restr[U]{D'}$;
		\item $\restrc{F} = \restrc{D}$ and
			$\restrc{F'} = \restrc{D'}$;
		\item $\restr[U]{G} = \restr[U]{D}$ and $\restr[U]{G'} = \restr[U]{D'}$.
	\end{enumerate}
	Then:
	\[
		\mstr[D', D] \ent \pi(\kb)
			\mlequiv \mstr[E', E] \ent \pi(b_U(\kb)) \land \mstr[F', F] \ent \pi(e_U(\kb, G', G))
			\enspace.
	\]
\end{proposition}
\begin{proof}
	Follows by Lemmas \ref{lemma:splitting:aux:1} and
	\ref{lemma:splitting:aux:2}.
\end{proof}

\begin{corollary} \label{cor:splitting:aux:1}
	Let $U$ be a splitting set for a hybrid knowledge base $\kb$ and $M, X \in
	\mint$ be such that $\restr[U]{X} = \restr[U]{M}$. Then:
	\[
		M \ent \pi(\kb)
			\mlequiv M \ent \pi(b_U(\kb)) \land M \ent \pi(e_U(\kb, X)) \enspace.
	\]
\end{corollary}
\begin{proof}
	Proposition \ref{prop:splitting:aux} for $D = D' = E = E' = F = F' = M$ and
	$G = G' = X$ implies that
	\[
		\mstr[M, M] \ent \pi(\kb)
			\mlequiv \mstr[M, M] \ent \pi(b_U(\kb)) \land \mstr[M, M] \ent \pi(e_U(\kb, X, X))
			\enspace.
	\]
	The claim of this corollary now follows from Proposition
	\ref{prop:sub_shortcut}.
\end{proof}

\begin{corollary} \label{cor:splitting:aux:2}
	Let $U$ be a splitting set for a hybrid knowledge base $\kb$ and $M, X, Y$
	be MKNF interpretations such that $\restr[U]{X} = \restr[U]{M}$ and
	$\restrc{Y} = \restrc{M}$. Then:
	\[
		M \ent \pi(\kb)
			\mlequiv X \ent \pi(b_U(\kb)) \land Y \ent \pi(e_U(\kb, X)) \enspace.
	\]
\end{corollary}
\begin{proof}
	Proposition \ref{prop:splitting:aux} for $D = D' = M$, $E = E' = X$, $F = F'
	= Y$ and $G = G' = X$ implies that
	\[
		\mstr[M, M] \ent \pi(\kb)
			\mlequiv \mstr[X, X] \ent \pi(b_U(\kb)) \land \mstr[Y, Y] \ent \pi(e_U(\kb, X, X))
			\enspace.
	\]
	The claim of this corollary now follows from Proposition
	\ref{prop:sub_shortcut}.
\end{proof}

\begin{corollary} \label{cor:splitting:aux:3}
	Let $U$ be a splitting set for a hybrid knowledge base $\kb$ and $M, M', X,
	X'$ be MKNF interpretations such that $M \ent \pi(\kb)$, $\restr[\lpre
	\setminus U]{M'} = \restrc{M}$, $\restr[U]{X} =
	\restr[U]{M}$ and $\restr[U]{X'} = \restr[U]{M'}$. Then:
	\[
		\mstr[M', M] \nent \pi(\kb) \mlthen \mstr[X', X] \nent \pi(b_U(\kb))
		\enspace.
	\]
\end{corollary}
\begin{proof}
	Proposition \ref{prop:splitting:aux} for $D = M$, $D' = M'$, $E = X$, $E' =
	X'$, $F = F' = M$, $G = G' = X$ implies that
	\[
		\mstr[M', M] \nent \pi(\kb)
			\mlequiv \mstr[X', X] \nent \pi(b_U(\kb)) \lor \mstr[M, M] \nent \pi(e_U(\kb, X, X))
			\enspace.
	\]
	Furthermore, from Corollary \ref{cor:splitting:aux:1} we know that $M \ent
	\pi(e_U(\kb, X))$ is always satisfied because $M \ent \pi(\kb)$. Hence, by
	Corollary \ref{cor:sub:ent_relevant}, the second disjunct on the right hand
	side of the above equivalence can be safely omitted and we obtain the claim
	of this corollary.
\end{proof}

\begin{corollary} \label{cor:splitting:aux:4}
	Let $U$ be a splitting set for a hybrid knowledge base $\kb$ and $M, M', X,
	Y, Y'$ be MKNF interpretations such that $M \ent \pi(\kb)$, $\restr[U]{M'} =
	\restr[U]{X} = \restr[U]{M}$, $\restrc{Y} = \restr[\lpre
	\setminus U]{M}$ and $\restrc{Y'} = \restr[\lpre \setminus
	U]{M'}$. Then:
	\[
		\mstr[M', M] \nent \pi(\kb) \mlthen \mstr[Y', Y] \nent \pi(e_U(\kb, X))
		\enspace.
	\]
\end{corollary}
\begin{proof}
	Proposition \ref{prop:splitting:aux} for $D = M$, $D' = M'$, $E = E' = M$,
	$F = Y$, $F' = Y'$, $G = G' = X$ implies that
	\[
		\mstr[M', M] \nent \pi(\kb)
			\mlequiv \mstr[M, M] \nent \pi(b_U(\kb)) \lor \mstr[Y', Y] \nent \pi(e_U(\kb, X, X))
			\enspace.
	\]
	Furthermore, from Corollary \ref{cor:splitting:aux:1} we know that $M \ent
	\pi(b_U(\kb))$ is always satisfied because $M \ent \pi(\kb)$. Hence, by
	Corollary \ref{cor:sub:ent_relevant}, the first disjunct in the right hand
	side of the above equivalence can be safely omitted and we obtain the claim
	of this corollary.
\end{proof}

\begin{corollary} \label{cor:splitting:aux:5}
	Let $U$ be a splitting set for a hybrid knowledge base $\kb$ and $M, M', X,
	X'$ be MKNF interpretations such that $\restr[U]{X} = \restr[U]{M}$ and
	$\restr[U]{X'} = \restr[U]{M'}$. Then:
	\[
		\mstr[X', X] \nent \pi(b_U(\kb)) \mlthen \mstr[M', M] \nent \pi(\kb)
		\enspace.
	\]
\end{corollary}
\begin{proof}
	Proposition \ref{prop:splitting:aux} for $D = M$, $D' = M'$, $E = X$, $E' =
	X'$, $F = G = M$, $F' = G' = M'$ implies that
	\[
		\mstr[M', M] \nent \pi(\kb)
			\mlequiv \mstr[X', X] \nent \pi(b_U(\kb)) \lor \mstr[M', M] \nent \pi(e_U(\kb, M', M))
			\enspace.
	\]
	The claim of this corollary follows directly from this equivalence.
\end{proof}

\begin{corollary} \label{cor:splitting:aux:6}
	Let $U$ be a splitting set for a hybrid knowledge base $\kb$ and $M, M', X,
	Y, Y'$ be MKNF interpretations such that $\restr[U]{M'} =
	\restr[U]{X} = \restr[U]{M}$, $\restrc{Y} = \restr[\lpre
	\setminus U]{M}$ and $\restrc{Y'} = \restr[\lpre \setminus
	U]{M'}$. Then:
	\[
		\mstr[Y', Y] \nent \pi(e_U(\kb, X)) \mlthen \mstr[M', M] \nent \pi(\kb)
		\enspace.
	\]
\end{corollary}
\begin{proof}
	Proposition \ref{prop:splitting:aux} for $D = M$, $D' = M'$, $E = E' = M$,
	$F = Y$, $F' = Y'$, $G = G' = X$ implies that
	\[
		\mstr[M', M] \nent \pi(\kb)
			\mlequiv \mstr[M, M] \nent \pi(b_U(\kb)) \lor \mstr[Y', Y] \nent \pi(e_U(\kb, X, X))
			\enspace.
	\]
	The claim of this corollary follows directly from this equivalence.
\end{proof}

\begin{proposition} \label{prop:splitting:X}
	Let $U$ be a splitting set for a hybrid knowledge base $\kb$, $M$ be an MKNF
	model of $\kb$ and $X = \sat{M}$. Then $X$ is an MKNF model of
	$b_U(\kb)$.
\end{proposition}
\begin{proof}
	By Proposition \ref{prop:sigma} we know that $M \subseteq X$ and that $X$ is
	saturated relative to $U$. We need to show that $X$ is an MKNF model of
	$b_U(\kb)$. By Proposition \ref{prop:sub_shortcut}, this holds if and only
	if $X$ is an S5 model of $\pi(b_U(\kb))$ and for every $X' \supsetneq X$ it
	holds that $\mstr[X', X] \nent \pi(b_U(\kb))$. The former follows directly
	from Corollary \ref{cor:splitting:aux:2}, so we will focus on the latter
	condition.

	Let's pick some $X' \supsetneq X$. By Proposition \ref{prop:semisat_exists}
	there exists the greatest MKNF interpretation $M'$ that coincides with $X'$
	on $U$ (i.e. $\restr[U]{M'} = \restr[U]{X'}$) and with $M$ on $\lpre
	\setminus U$ (i.e. $\restrc{M'} = \restr[\lpre \setminus
	U]{M}$) and which contains $X' \cap M$. Hence,
	\begin{equation}
		M \subseteq X \cap M \subseteq X' \cap M \subseteq M' \enspace.
		\label{eq:prop:splitting:X:1}
	\end{equation}
	Furthermore, we know that $X$ is saturated relative to $U$, so we can use
	Proposition \ref{prop:restr:sat:1}\eqref{eq:prop:restr:sat:1:2} to conclude
	that
	\begin{equation}
		\restr[U]{M} = \restr[U]{X} \subsetneq \restr[U]{X'} = \restr[U]{M'}
		\enspace. \label{eq:prop:splitting:X:2}
	\end{equation}
	Consequently, by \eqref{eq:prop:splitting:X:1},
	\eqref{eq:prop:splitting:X:2} and Proposition
	\ref{prop:restr:basic}\eqref{eq:prop:restr:basic:3}, we obtain $M \subsetneq
	M'$. This, together with the assumption that $M$ is an MKNF model of $\kb$,
	implies that $\mstr[M', M] \nent \pi(\kb)$. We can now apply Corollary
	\ref{cor:splitting:aux:3} to conclude that $\mstr[X', X] \nent
	\pi(b_U(\kb))$, which is also the desired conclusion.
\end{proof}

\begin{proposition} \label{prop:splitting:Y}
	Let $U$ be a splitting set for a hybrid knowledge base $\kb$, $M$ be an MKNF
	model of $\kb$, $X$ be the greatest MKNF interpretation that coincides with
	$M$ on $U$ and $Y$ be the greatest MKNF interpretation that coincides with
	$M$ on $\lpre \setminus U$. Then $Y$ is an MKNF model of $e_U(\kb, X)$.
\end{proposition}
\begin{proof}
	By Proposition \ref{prop:sigma} we know that $M \subseteq Y$ and that $Y$ is
	saturated relative to $\lpre \setminus U$. We need to show that $Y$ is an
	MKNF model of $e_U(\kb, X)$. By Proposition \ref{prop:sub_shortcut}, this
	holds if and only if $Y$ is and S5 model of $\pi(e_U(\kb, X))$ and for every
	$Y' \supsetneq Y$ it holds that $\mstr[Y', Y] \nent \pi(e_U(\kb, X))$. The
	former follows directly from Corollary \ref{cor:splitting:aux:2}, so we will
	focus on the latter condition.

	Let's pick some $Y' \supsetneq Y$. By Proposition \ref{prop:semisat_exists}
	there exists the greatest MKNF interpretation $M'$ that coincides with $M$
	on $U$ (i.e. $\restr[U]{M'} = \restr[U]{M}$) and with $Y'$ on $\lpre
	\setminus U$ (i.e. $\restrc{M'} = \restr[\lpre \setminus
	U]{Y'}$) and which contains $M \cap Y'$. Hence,
	\begin{equation}
		M \subseteq M \cap Y \subseteq M \cap Y' \subseteq M' \enspace.
		\label{eq:prop:splitting:Y:1}
	\end{equation}
	Furthermore, we know that $Y$ is saturated relative to $\lpre \setminus U$,
	so we can use Proposition
	\ref{prop:restr:sat:1}\eqref{eq:prop:restr:sat:1:2} to conclude that
	\begin{equation}
		\restrc{M} = \restrc{Y'} \subsetneq
		\restrc{Y'} = \restrc{M'} \enspace.
		\label{eq:prop:splitting:Y:2}
	\end{equation}
	Consequently, by \eqref{eq:prop:splitting:Y:1},
	\eqref{eq:prop:splitting:Y:2} and Proposition
	\ref{prop:restr:basic}\eqref{eq:prop:restr:basic:3}, we obtain $M \subsetneq
	M'$. This, together with the assumption that $M$ is an MKNF model of $\kb$,
	implies that $\mstr[M', M] \nent \pi(\kb)$. We can now apply Corollary
	\ref{cor:splitting:aux:4} to conclude that $\mstr[Y', Y] \nent \pi(e_U(\kb,
	X))$, which is also the desired conclusion.
\end{proof}

\begin{proposition} \label{prop:splitting:XcapY}
	Let $U$ be a splitting set for a hybrid knowledge base $\kb$ and $\an{X, Y}$
	be a solution to $\kb$ with respect to $U$. Then $X \cap Y$ is an MKNF model
	of $\kb$.
\end{proposition}
\begin{proof*}
	Let $\kb = \an{\ont, \prog}$ and $M = X \cap Y$. In order to show that $M$
	is an MKNF model of $\kb$, we need to prove that $M \ent \pi(\kb)$ and that
	for every $M' \supsetneq M$ it holds that $\mstr[M', M] \nent \pi(\kb)$. We
	will verify the two conditions separately.

	Since $\an{X, Y}$ is a solution to $\kb$ with respect to $U$, $X$ must an
	MKNF model of $b_U(\kb)$ and $Y$ an MKNF model of $e_U(\kb, X)$. So $X \ent
	\pi(b_U(\kb))$ and $Y \ent \pi(e_U(\kb, X))$. Consequently, by Corollary
	\ref{cor:splitting:aux:2}, $M \ent \pi(\kb)$.

	We know that $X$ is an MKNF model of $b_U(\kb)$, so, by Proposition
	\ref{prop:saturated:mknf}, $X$ is saturated relative to $U$. Similarly,
	since $Y$ is an MKNF model of $e_U(\kb, X)$, it must be saturated relative
	to $\lpre \setminus U$. Hence, by Proposition \ref{prop:sat_semisat}, $M$ is
	semi-saturated relative to $U$, $\restr[U]{M} = \restr[U]{X}$ and
	$\restrc{M} = \restrc{Y}$.

	Now take some MKNF interpretation $M' \supsetneq M$ and let $X' = X \cup M'$
	and $Y' = Y \cup M'$. We already inferred that $M$ is semi-saturated
	relative to $U$, which means that by Proposition
	\ref{prop:restr:semisat:1}\eqref{eq:prop:restr:semisat:1:2} one of the
	following cases must occur:
	\begin{enumerate}
		\renewcommand{\labelenumi}{\alph{enumi})}
		\item If $\restr[U]{M'} \supsetneq \restr[U]{M}$, then
			\[
				\restr[U]{X'} = \restr[U]{X} \cup \restr[U]{M'} = \restr[U]{M'}
				\supsetneq \restr[U]{M} = \restr[U]{X} \enspace,
			\]
			and so Proposition \ref{prop:restr:basic}\eqref{eq:prop:restr:basic:3}
			implies that $X' \supsetneq X$. Hence, since $X$ is an MKNF model of
			$b_U(\kb)$, we infer that $\mstr[X', X] \nent \pi(b_U(\kb))$ and by
			Corollary \ref{cor:splitting:aux:5} we obtain $\mstr[M', M] \nent
			\pi(\kb)$, which is what we wanted to prove.

		\item If $\restr[U]{M'} = \restr[U]{M}$ and $\restrc{M'}
			\supsetneq \restrc{M}$, then $\restr[U]{X'} =
			\restr[U]{M'} = \restr[U]{M}$ and 
			\[
				\restrc{Y'} = \restrc{Y} \cup
				\restrc{M'} = \restrc{M'}
				\supsetneq \restrc{M} = \restrc{Y}
				\enspace,
			\]
			and so Proposition \ref{prop:restr:basic}\eqref{eq:prop:restr:basic:3}
			implies that $Y' \supsetneq Y$. Hence, since $Y$ is an MKNF model of
			$e_U(\kb, X)$, we infer that $\mstr[Y', Y] \nent \pi(e_U(\kb, X))$ and
			by Corollary \ref{cor:splitting:aux:6} we obtain $\mstr[M', M] \nent
			\pi(\kb)$, which is what we wanted to prove. \qedhere
	\end{enumerate}
\end{proof*}

\begin{theorem*}
	[Splitting Theorem for Hybrid MKNF Knowledge Bases]
	{thm:splitting}%
	Let $U$ be a splitting set for a hybrid knowledge base $\kb$. An MKNF
	interpretation $M$ is an MKNF model of $\kb$ if and only if $M = X \cap Y$
	for some solution $\an{X, Y}$ to $\kb$ with respect to $U$.
\end{theorem*}
\begin{proof}[Proof of Theorem \ref{thm:splitting}]
	\label{proof:thm:splitting}%
	First suppose that $M$ is an MKNF model of $\kb$. By Proposition
	\ref{prop:splitting:X} we know that $X = \sat{M}$ is an MKNF model of
	$b_U(\kb)$ and by Proposition \ref{prop:splitting:Y} that the $Y = \satc{M}$
	is an MKNF model of $e_U(\kb, X)$. Furthermore, by Proposition
	\ref{prop:sigma}, $\restr{X} = \restr{M}$ and $X$ is saturated relative to
	$U$, $\restrc{Y} = \restr[\lpre \setminus U]{M}$ and $Y$ is saturated
	relative to $\lpre \setminus U$, and $M \subseteq X \cap Y$. Finally, by
	Proposition \ref{prop:sat_semisat} we obtain the following:
	\begin{align*}
		\restr{(X \cap Y)} &= \restr{X} = \restr{M}  \enspace, \\
		\restrc{(X \cap Y)} &= \restrc{Y} =
			\restrc{M} \enspace.
	\end{align*}
	It remains to show that $M \supseteq X \cap Y$. Suppose this is not the
	case, so $M \subsetneq X \cap Y$. Then, since $M$ is an MKNF model of $\kb$,
	$\mstr[X \cap Y, M] \nent \pi(\kb)$ and by Proposition
	\ref{prop:splitting:aux} for $D = M$, $D' = X \cap Y$, $E = E' = F = F' =
	M$, $G = G' = X$, we obtain
	\[
		\mstr[M, M] \nent \pi(b_U(\kb)) \lor \mstr[M, M] \nent \pi(e_U(\kb, X, X))
		\enspace.
	\]
	However, Corollary \ref{cor:splitting:aux:1} now entails $M \nent \pi(\kb)$,
	a conflict with the assumption that $M$ is an MKNF model of $\kb$.
	Consequently, $M = X \cap Y$.

	The converse implication follows directly from Proposition
	\ref{prop:splitting:XcapY}.
\end{proof}

\begin{corollary} \label{cor:splitting:1}
	Let $U$ be a splitting set for a hybrid knowledge base $\kb$ and $M \in
	\mint$. If $M$ is an MKNF model of $\kb$, then the pair
	\[
		\an{\sat{M}, \satc{M}}
	\]
	is a solution to $\kb$ with respect to $U$, $M = \sat{M} \cap \satc{M}$ and
	$M$ is semi-saturated relative to $U$.
\end{corollary}
\begin{proof}
	This is a consequence of the proof of Theorem \ref{thm:splitting} and of
	Proposition \ref{prop:sat_semisat}.
\end{proof}

\begin{corollary} \label{cor:splitting:2}
	Let $U$ be a splitting set for a hybrid knowledge base $\kb$ such that there
	exists at least one solution to $\kb$ relative to $U$. Then $\kb$ is MKNF
	satisfiable and an MKNF interpretation $M$ is an MKNF model of $\kb$ if and
	only if $M = X \cap Y$ for some solution $\an{X, Y}$ to $\kb$ with respect
	to $U$.
\end{corollary}
\begin{proof}
	Follows from Theorem \ref{thm:splitting}.
\end{proof}

\begin{corollary} \label{cor:splitting:3}
	Let $U$ be a splitting set for a hybrid knowledge base $\kb$. An MKNF
	formula $\phi$ is MKNF entailed by $\kb$ if and only if, for every solution
	$\an{X, Y}$ of $\kb$ with respect to $U$, $X \cap Y \ent \phi$.
\end{corollary}
\begin{proof}
	Follows from Theorem \ref{thm:splitting}.
\end{proof}

\section{Proof of Splitting Sequence Theorem for Hybrid MKNF Knowledge Bases}

\label{app:splitting sequence theorem}

Many parts of the proofs in this section are adapted from \cite{Turner1996}.

\begin{remark}
	It is easy to see that solutions to $\kb$ with respect to a splitting
	sequence $\an{U, \lpre}$ are the same as the solutions to $\kb$ with respect
	to the splitting set $U$.

	Let $U = \an{U_\alpha}_{\alpha < \mu}$ be a splitting sequence for a hybrid
	knowledge base $\kb$, and let $\an{X_\alpha}_{\alpha < \mu}$ be a sequence
	of MKNF interpretations. Then:
	\begin{align*}
		\preds{b_{U_0}(\kb)} &\subseteq U_0 \\
		\predsP{e_{U_\alpha} \left( b_{U_{\alpha+1}}(\kb),
			\textstyle \bigcap_{\eta \leq \alpha} X_\eta \right)}
			&\subseteq U_{\alpha + 1} \setminus U_\alpha
			\text{ whenever } \alpha + 1 < \mu
	\end{align*}

	Furthermore, when $X$ is a solution to $\kb$ with respect to $U$, then $X_0$
	is saturated relative to $U_0$ and for every $\alpha$ such that $\alpha + 1
	< \mu$, $X_{\alpha + 1}$ is saturated relative to $U_{\alpha + 1} \setminus
	U_\alpha$. Also note that for any limit ordinal $\alpha$, $X_\alpha =
	\mint$, so $X_\alpha$ is saturated relative to any set of predicate symbols.
\end{remark}

\begin{lemma} \label{lemma:seq:sat}
	Let $\an{U_\alpha}_{\alpha < \mu}$ be a sequence of sets of atoms and
	$\an{X_\alpha}_{\alpha < \mu}$ be a sequence of members of $\mint$ such that
	for all $\alpha < \mu$, $X_\alpha$ is saturated relative to $U_\alpha$. Then
	$\bigcap_{\alpha < \mu} X_\alpha$ is saturated relative to $\bigcup_{\alpha
	< \mu} U_\alpha$.
\end{lemma}
\begin{proof}
	Let $U = \bigcup_{\alpha < \mu} U_\alpha$ and $X = \bigcap_{\alpha < \mu}
	X_\alpha$ and suppose $\restr{I}$ belongs to $\restr{X}$. Then there is some
	$J \in X$ such that $\restr{I} = \restr{J}$. This means that for every
	ground atom $p$ with $\preds{p} \subseteq U$,
	\[
		p \in I \mlequiv p \in J \enspace.
	\]
	We need to show that $I$ belongs to $X$. Take some $\beta < \mu$ and any
	atom $p$ such that $\preds{p} \subseteq U_\beta$. Since $U_\beta$ is a
	subset of $U$, we immediately obtain
	\[
		p \in I \mlequiv p \in J \enspace.
	\]
	Furthermore, since $J \in X \subseteq X_\beta$, it follows that
	$\restr[U_\beta]{I}$ belongs to $\restr[U_\beta]{X_\beta}$. Moreover,
	$X_\beta$ is saturated relative to $U_\beta$, so we conclude that $I$
	belongs to $X_\beta$. Since the choice of $\beta$ was arbitrary, $I$ belongs
	to $X_\beta$ for all $\beta < \mu$. Thus, $I$ belongs to $X$ as well.
\end{proof}

\begin{lemma} \label{lemma:seq:sat_split}
	Let $U = \an{U_\alpha}_{\alpha < \mu}$ be a splitting sequence for a hybrid
	knowledge base $\kb$ and $X = \an{X_\alpha}_{\alpha < \mu}$ be a sequence of
	members of $\mint$ such that $X_0$ is saturated relative to $U_0$, for each
	$\alpha$ such that $\alpha + 1 < \mu$, $X_{\alpha + 1}$ is saturated
	relative to $U_{\alpha + 1} \setminus U_\alpha$, and for every limit ordinal
	$\alpha < \mu$, $X_\alpha = \foint$.
	Then for all ordinals ordinal $\beta < \alpha < \mu$ the following holds:
	\begin{itemize}
		\item $\bigcap_{\eta \leq \beta} X_\eta$ is saturated relative to
			$U_\beta$;
		\item $\bigcap_{\beta < \eta \leq \alpha} X_\eta$ is saturated relative to
			$U_\alpha \setminus U_\beta$;
		\item $\bigcap_{\alpha < \eta < \mu} X_\eta$ is saturated relative to
			$\lpre \setminus U_\alpha$.
	\end{itemize}
\end{lemma}
\begin{proof*}
	Let $\an{V_\alpha}_{\alpha < \mu}$ be a sequence of sets of atoms defined as
	follows: $V_0 = U_0$, for every $\alpha$ such that $\alpha + 1 < \mu$,
	$V_{\alpha + 1} = U_{\alpha + 1} \setminus U_\alpha$, and for every limit
	ordinal $\alpha$, $V_\alpha = \emptyset$. By the definition of $X$, it must
	hold for every ordinal $\alpha < \mu$ that $X_\alpha$ is saturated relative
	to $V_\alpha$. Furthermore, by Lemma \ref{lemma:seq:sat} we obtain that for
	every $\alpha < \mu$, $\bigcap_{\eta \leq \alpha} X_\eta$ is saturated
	relative to
	\[
		\bigcup_{\eta \leq \alpha} V_\eta = U_0 \cup \bigcup_{\eta < \alpha}
			U_{\eta + 1} \setminus U_\eta = U_\alpha \enspace.
	\]
	The same lemma implies that $\bigcap_{\beta < \eta \leq \alpha} X_\eta$ must be
	saturated relative to
	\[
		\bigcup_{\beta < \eta \leq \alpha} V_\eta
			= \bigcup_{\beta \leq \eta < \eta + 1 \leq \alpha}
			U_{\eta + 1} \setminus U_\eta = U_\alpha \setminus U_\beta
	\]
	and that $\bigcap_{\alpha < \eta < \mu} X_\eta$ must be saturated relative
	to
	\[
		\bigcup_{\alpha < \eta < \mu} V_\eta
			= \bigcup_{\alpha \leq \eta < \eta + 1 < \mu}
			U_{\eta + 1} \setminus U_\eta = \lpre \setminus U_\alpha \enspace.
			\qedhere
	\]
\end{proof*}

\begin{lemma} \label{lemma:seq:sigma}
	Let $U = \an{U_\alpha}_{\alpha < \mu}$ be a splitting sequence for a hybrid
	knowledge base $\kb$, $M$ be an MKNF interpretation and $X =
	\an{X_\alpha}_{\alpha < \mu}$ be a sequence of MKNF interpretations such
	that
	\begin{itemize}
		\item $X_0 = \sat[U_0]{M}$;
		\item for all $\alpha$ such that $\alpha + 1 < \mu$, $X_{\alpha + 1} =
			\sat[U_{\alpha + 1} \setminus U_\alpha]{M}$;
		\item for any limit ordinal $\alpha < \mu$, $X_\alpha = \foint$.
	\end{itemize}
	If $M$ is an MKNF model of $\kb$, then for every ordinal $\alpha < \mu$,
	\[
		\bigcap_{\eta \leq \alpha} X_\eta = \sat[U_\alpha]{M} \enspace.
	\]
\end{lemma}
\begin{proof*}
	We will prove by induction on $\alpha$:
	\begin{enumerate}
		\renewcommand{\labelenumi}{\arabic{enumi}$^\circ$}

		\item Suppose $\alpha = 0$. We need to show that $X_0 = \sat[U_0]{M}$,
			which follows directly from the definition of $X_0$.

		\item Suppose $\alpha$ such that $\alpha + 1 < \mu$ and by the inductive
			assumption, $\bigcap_{\eta \leq \alpha} X_\eta = \sat[U_\alpha]{M}$.
			We immediately obtain:
			\begin{align*}
				\bigcap_{\eta \leq \alpha + 1} X_\eta
					&= X_{\alpha + 1} \cap \bigcap_{\eta \leq \alpha} X_\eta
						= X_{\alpha + 1} \cap \sat[U_\alpha]{M} \\
					&= \sat[U_{\alpha + 1} \setminus U_\alpha]{M}
						\cap \sat[U_\alpha]{M} \enspace.
			\end{align*}
		It remains to show that
		\[
			\sat[U_{\alpha + 1}]{M}
				= \sat[U_{\alpha + 1} \setminus U_\alpha]{M}
				\cap \sat[U_\alpha]{M} \enspace.
		\]
		We know that $U_{\alpha + 1}$ is a splitting set for $\kb$ and that $M$ is
		an MKNF model of $\kb$, so by Corollary \ref{cor:splitting:1} it follows
		that $N = \sat[U_{\alpha + 1}]{M}$ is an MKNF model of $b_{U_{\alpha +
		1}}(\kb)$. Furthermore, it can be easily verified that $U_\alpha$ is a
		splitting set for $b_{U_{\alpha + 1}}(\kb)$, so by another application of
		Corollary \ref{cor:splitting:1} we obtain that
		\begin{equation}
			\sat[U_{\alpha + 1}]{M} = N = \sat[U_\alpha]{N}
				\cap \satc[U_\alpha]{N} \enspace.
				\label{eq:proof:seq:sigma:1}
		\end{equation}
		Moreover, Proposition \ref{prop:sigma:repeated} yields
		\begin{align}
			\begin{split}
				\sat[U_\alpha]{N}
					&= \sat[U_\alpha]{\sat[U_{\alpha + 1}]{M}}
						= \sat[U_{\alpha + 1} \cap U_\alpha]{M} \\
					&= \sat[U_\alpha]{M}
			\end{split} \label{eq:proof:seq:sigma:2} \\
			\intertext{and}
			\begin{split}
				\satc[U_\alpha]{N}
					&= \satc[U_\alpha]{\sat[U_{\alpha + 1}]{M}}
						= \sat[U_{\alpha + 1} \cap (\lpre \setminus U_\alpha]{M}) \\
					&= \sat[U_{\alpha + 1} \setminus U_\alpha]{M} \enspace.
			\end{split} \label{eq:proof:seq:sigma:3}
		\end{align}
		The desired conclusion follows from \eqref{eq:proof:seq:sigma:1},
		\eqref{eq:proof:seq:sigma:2} and \eqref{eq:proof:seq:sigma:3}.
		
		\item Suppose $\alpha < \mu$ is a limit ordinal and for all $\eta <
			\alpha$ it holds that $\bigcap_{\beta \leq \eta} X_\beta =
			\sat[U_\eta]{M}$. First note that
			\begin{align*}
				\bigcap_{\eta \leq \alpha} X_\eta
					&= X_\alpha \cap \bigcap_{\eta < \alpha} X_{\eta}
					= \foint \cap \bigcap_{\eta < \alpha}
						\bigcap_{\beta \leq \eta} X_{\beta}
					= \bigcap_{\eta < \alpha} \sat[U_{\eta}]{M} \\
					&= \bigcap_{\eta < \alpha} \Set{ I \in \foint | (\exists J \in M)
						\left( \restr[U_\eta]{J} = \restr[U_\eta]{I} \right) } \\
					&= \Set{I \in \foint | (\forall \eta < \alpha)
						(\exists J \in M)(\restr[U_\eta]{J} = \restr[U_\eta]{I}) }
			\end{align*}
			and also that
			\begin{align*}
				\sat[U_\alpha]{M}
					&= \satpar[\bigcup_{\eta < \alpha} U_\eta]{M} \\
					&= \Set{I \in \foint | (\exists J \in M)
						\left( \restr[\bigcup_{\eta < \alpha} U_\eta]{J}
						= \restr[\bigcup_{\eta < \alpha} U_\eta]{I}  \right) }
					\enspace.
			\end{align*}
			From these two identities it can be inferred that $\sat[U_\alpha]{M}$ is
			a subset of $\bigcap_{\eta \leq \alpha} X_\eta$. Indeed, if $I$ belongs
			to $\sat[U_\alpha]{M}$, then for some $J \in M$ we have
			$\restr[\bigcup_{\eta < \alpha} U_\eta]{J} = \restr[\bigcup_{\eta <
			\alpha} U_\eta]{I}$, hence for any $\eta_0 < \alpha$ and any atom $p$
			such that $\preds{p} \subseteq U_{\eta_0} \subseteq \bigcup_{\eta <
			\alpha} U_\eta$ we obtain
			\begin{align*}
				p \in J \mlequiv p \in I
			\end{align*}
			To prove that the converse inclusion holds as well, we let $Y =
			\bigcap_{\eta \leq \alpha} X_\eta$ and proceed by contradiction,
			assuming that $\sat[U_\alpha]{M}$ is a proper subset of $Y$. By
			Corollary \ref{cor:splitting:1} we know that $\sat[U_\alpha]{M}$ is an
			MKNF model of $b_{U_\alpha}(\kb)$, so there must be some formula $\phi
			\in b_{U_\alpha}(\kb)$ such that
			\[
				\an{ Y, \sat[U_\alpha]{M} } \nent \phi \enspace.
			\]
			Furthermore, since $U_\alpha = \bigcup_{\eta < \alpha} U_\eta$ and
			$\preds{\phi}$ is a finite set of predicate symbols, there must be some
			$\beta < \alpha$ such that $\preds{\phi}$ is a subset of $U_\beta$.
			Consequently, by Corollary \ref{cor:sub:ent_relevant}, we obtain
			\[
				\an{ \restr[U_\beta]{Y}, \restr[U_\beta]{\sat[U_\alpha]{M}} }
					\nent \phi \enspace.
			\]
			Let $Y_1 = \bigcap_{\eta \leq \beta} X_\eta$ and $Y_2 = \bigcap_{\beta <
			\eta \leq \alpha} X_\eta$. By Lemma \ref{lemma:seq:sat_split}, $Y_1$ is
			saturated relative to $U_\beta$ and $Y_2$ is saturated relative to
			$U_\alpha \setminus U_\beta$ and thus by Lemma
			\ref{lemma:saturated:superset} also relative to $\lpre \setminus
			U_\beta$. Furthermore, $Y = Y_1 \cap Y_2$, so by Proposition
			\ref{prop:sat_semisat}, $\restr[U_\beta]{Y} = \restr[U_\beta]{Y_1}$.
			Hence, 
			\[
				\an{ \restr[U_\beta]{Y_1}, \restr[U_\beta]{\sat[U_\alpha]{M}} }
					\nent \phi
			\]
			and the inductive assumption for $\beta$ yields
			\[
				\an{ \restr[U_\beta]{\sat[U_\beta]{M}},
					\restr[U_\beta]{\sat[U_\alpha]{M}} } \nent \phi \enspace.
			\]
			Finally, since $U_\beta$ is a subset of $U_\alpha$, Proposition
			\ref{prop:sigma:restr} implies that
			\[
				\restr[U_\beta]{\sat[U_\alpha]{M}}
					= \restr[U_\beta]{M} = \restr[U_\beta]{\sat[U_\beta]{M}}
					\enspace,
			\]
			so
			\[
				\an{ \restr[U_\beta]{\sat[U_\beta]{M}},
					\restr[U_\beta]{\sat[U_\beta]{M}} } \nent \phi \enspace.
			\]
			Corollary \ref{cor:sub:ent_relevant} now yields $\an{ \sat[U_\beta]{M},
			\sat[U_\beta]{M} } \nent \phi$. But at the same time,
			$U_\beta$ is a splitting set for $\kb$, so, by Corollary
			\ref{cor:splitting:1}, $\sat[U_\beta]{M}$ is an MKNF model of
			$b_{U_\beta}(\kb)$. However, $\phi$ belongs to $b_{U_\beta}(\kb)$, so we
			reached the desired contradiction. \qedhere
	\end{enumerate}
\end{proof*}

\begin{proposition} \label{prop:seq:1}
	Let $U = \an{U_\alpha}_{\alpha < \mu}$ be a splitting sequence for a hybrid
	knowledge base $\kb$, $M$ be an MKNF interpretation and $X =
	\an{X_\alpha}_{\alpha < \mu}$ be a sequence of MKNF interpretations such
	that
	\begin{itemize}
		\item $X_0 = \sat[U_0]{M}$;
		\item for all $\alpha$ such that $\alpha + 1 < \mu$, $X_{\alpha + 1} =
			\sat[U_{\alpha + 1} \setminus U_\alpha]{M}$;
		\item for any limit ordinal $\alpha < \mu$, $X_\alpha = \foint$.
	\end{itemize}
	If $M$ is an MKNF model of $\kb$, then $X$ is a solution to $\kb$ with
	respect to $U$.
\end{proposition}
\begin{proof}
	There are four conditions to verify.

	First, $X_0$ must be an MKNF model of $b_{U_0}(\kb)$. Since $U_0$ is a
	splitting set for $\kb$, Corollary \ref{cor:splitting:1} yields that
	$\sat[U_0]{M}$ is an MKNF model of $b_{U_0}(\kb)$. By definition, $X_0 =
	\sat[U_0]{M}$, thus this part of the proof is finished.

	Second, for any ordinal $\alpha$ such that $\alpha + 1 < \mu$ it should hold
	that $X_{\alpha + 1}$ is an MKNF model of 
	\[
		e_{U_\alpha} \left( b_{U_{\alpha+1}}(\kb),
			\textstyle \bigcap_{\eta \leq \alpha} X_\eta \right) \enspace.
	\]
	By Corollary \ref{cor:splitting:1}, $N = \sat[U_{\alpha + 1}]{M}$ is an MKNF
	model of $b_{U_{\alpha + 1}}(\kb)$. Furthermore, it can be easily seen that
	$U_\alpha$ is a splitting set for $b_{U_{\alpha + 1}}$, so by another
	application of Corollary \ref{cor:splitting:1}, we obtain that
	$\satc[U_\alpha]{N}$ is an MKNF model of $e_{U_\alpha}( b_{U_{\alpha +
	1}}(\kb), \sat[U_\alpha]{N})$. Moreover, by Proposition
	\ref{prop:sigma:repeated},
	\begin{align*}
		\sat[U_\alpha]{N}
			&= \sat[U_\alpha]{\sat[U_{\alpha + 1}]{M}}
			= \sat[U_{\alpha + 1} \cap U_\alpha]{M} = \sat[U_\alpha]{M}
		\intertext{and also}
		\begin{split}
			\satc[U_\alpha]{N}
				&= \satc[U_\alpha]{\sat[U_{\alpha + 1}]{M}}
				= \sat[U_{\alpha + 1} \cap (\lpre \setminus U_\alpha]{M}) \\
				&= \sat[U_{\alpha + 1} \setminus U_\alpha]{M} \enspace.
		\end{split}
	\end{align*}
	Since we know from Lemma \ref{lemma:seq:sigma} that $\sat[U_\alpha]{M} =
	\bigcap_{\eta \leq \alpha} X_\eta$ and by definition $\sat[U_{\alpha + 1}
	\setminus U_\alpha]{M} = X_{\alpha + 1}$, we have shown that $X_{\alpha +
	1}$ is an MKNF model of
	\[
		e_{U_\alpha} \left( b_{U_{\alpha+1}}(\kb),
			\textstyle \bigcap_{\eta \leq \alpha} X_\eta \right) \enspace.
	\]

	Third, for every limit ordinal $\alpha < \mu$, $X_\alpha = \foint$ holds by
	definition.

	Fourth, by definition of $X$, $M$ is a subset of $X_\alpha$ for every
	$\alpha < \mu$. Hence,
	\[
		\emptyset \neq M \subseteq \bigcap_{\alpha < \mu} X_\alpha \enspace,
	\]
	which finishes our proof.
\end{proof}

\begin{proposition} \label{prop:seq:2}
	Let $U = \an{U_\alpha}_{\alpha < \mu}$ be a splitting sequence for a hybrid
	knowledge base $\kb$. If $X = \an{X_\alpha}_{\alpha < \mu}$ is a solution to
	$\kb$ with respect to $U$, then for all $\alpha < \mu$, $\bigcap_{\eta \leq
	\alpha} X_\eta$ is an MKNF model of $b_{U_\alpha}(\kb)$.
\end{proposition}
\begin{proof*}
	Let $Y_\alpha = \bigcap_{\eta \leq \alpha} X_\eta$ for every $\alpha < \mu$.
	We will proceed by induction on $\alpha$:
	\begin{enumerate}
		\renewcommand{\labelenumi}{\arabic{enumi}$^\circ$}

		\item For $\alpha = 0$ we need to show that $Y_0 = X_0$ is an MKNF model
			of $b_{U_0}(\kb)$. This follows directly from the assumptions.

		\item For $\alpha$ such that $\alpha + 1 < \mu$ we need to show that
			$Y_{\alpha + 1}$ is an MKNF model of $b_{U_{\alpha + 1}}(\kb)$. By the
			inductive assumption, $Y_\alpha$ is an MKNF model of
			$b_{U_\alpha}(\kb)$. Furthermore,
			\[
				b_{U_\alpha}(b_{U_{\alpha + 1}}(\kb)) = b_{U_\alpha}(\kb)
			\]
			and since $X$ is a solution to $\kb$ with respect to $U$, $X_{\alpha +
			1}$ must an MKNF model of
			\[
				e_{U_\alpha} \left( b_{U_{\alpha+1}}(\kb), Y_\alpha \right) \enspace.
			\]
			Moreover, since $U_\alpha$ is a splitting set for $\kb$, it must also be
			a splitting set for $b_{U_{\alpha + 1}}(\kb)$. Consequently, by Theorem
			\ref{thm:splitting}, $Y_\alpha \cap X_{\alpha + 1} = Y_{\alpha + 1}$
			must be an MKNF model of $b_{U_{\alpha + 1}}(\kb)$.

		\item For a limit ordinal $\alpha < \mu$ we need to show that $Y_\alpha$
			is an MKNF model of $b_{U_\alpha}(\kb)$. First we will show that
			$Y_\alpha \ent b_{U_\alpha}(\kb)$ and then that for every $Y' \supsetneq
			Y_\alpha$ it holds that $\mstr[Y', Y_\alpha] \nent b_{U_\alpha}(\kb)$.

			Take some $\phi \in b_{U_\alpha}(\kb)$ and suppose $\beta < \alpha$ is
			some ordinal such that $\preds{\phi}$ is a subset of $U_\beta$. We know
			that $Y_\beta$ is an MKNF model of $b_{U_\beta}(\kb)$, so $Y_\beta \ent
			\phi$. Furthermore, for every $\eta$ such that $\eta < \alpha$, $X_{\eta
			+ 1}$ is an MKNF model of $e_{U_\eta}(b_{U_{\eta + 1}}(\kb), Y_\eta)$,
			so by Proposition \ref{prop:saturated:mknf}, $X_{\eta + 1}$ is saturated
			relative to $U_{\eta + 1} \setminus U_\eta$. Consequently, by Lemma
			\ref{lemma:seq:sat_split}, $Y_\beta = \bigcap_{\eta \leq \beta} X_\eta$
			is saturated relative to $U_\beta$ and $\bigcap_{\beta < \eta \leq
			\alpha} X_\eta$ is saturated relative to $U_\alpha \setminus U_\beta$
			and thus by Lemma \ref{lemma:saturated:superset} it is also saturated
			relative to $\lpre \setminus U_\beta$. Hence, by Proposition
			\ref{prop:sat_semisat}, for $Y_\alpha = Y_\beta \cap \bigcap_{\beta <
			\eta \leq \alpha} X_\eta$ it holds that $\restr[U_\beta]{Y_\alpha} =
			\restr[U_\beta]{Y_\beta}$, and so $Y_\alpha \ent \phi$ follows from
			Corollary \ref{cor:ent_relevant}.

			Now suppose $Y' \supsetneq Y_\alpha$. Then there must be some $I \in Y'
			\setminus Y_\alpha$. Take some $\beta < \alpha$ such that $I \notin
			Y_\beta$ (there must be such $\beta$, otherwise $I \in Y_\alpha$). Let
			$Y'' = Y' \cup Y_\beta$. By the inductive assumption, $Y_\beta$ is an
			MKNF model of $b_{U_\beta}(\kb)$, so there must be some $\phi \in
			b_{U_\beta}(\kb)$ such that
			\[
				\mstr[Y'', Y_\beta] \nent \phi
			\]
			Furthermore, $\restr[U_\beta]{Y_\beta} = \restr[U_\beta]{Y_\alpha}$ and
			\[
				\restr[U_\beta]{Y''} = \restr[U_\beta]{Y'} \cup \restr[U_\beta]{Y_\beta}
					= \restr[U_\beta]{Y'} \cup \restr[U_\beta]{Y_\alpha}
					= \restr[U_\beta]{(Y' \cup Y_\alpha)} = \restr[U_\beta]{Y'}
					\enspace.
			\]
			Consequently, by Corollary \ref{cor:sub:ent_relevant}, $\mstr[Y',
			Y_\alpha] \nent \phi$. \qedhere
	\end{enumerate}
\end{proof*}

\begin{lemma} \label{lemma:seq:prolonging}
	Let $U = \an{U_\alpha}_{\alpha < \mu}$ be a splitting sequence for a hybrid
	knowledge base $\kb$ and let $V = \an{V_\alpha}_{\alpha < \mu + 1}$ be a
	sequence of sets of atoms such that for every $\alpha < \mu$, $V_\alpha =
	U_\alpha$ and $V_\mu = \lpre$. Then $V$ is a splitting sequence for $\kb$. 

	Moreover, if $X = \an{X_\alpha}_{\alpha < \mu}$ is a solution to $\kb$ with
	respect to $U$, then $Y = \an{Y_\alpha}_{\alpha < \mu + 1}$, where for all
	$\alpha < \mu$, $Y_\alpha = X_\alpha$, and $Y_\mu = \foint$, is a solution to
	$\kb$ with respect to $V$. 
\end{lemma}
\begin{proof}
	It is easy to see that $V$ is monotone, continuous, that every $V_\alpha$ is
	a splitting set for $\kb$ and that $\bigcup_{\alpha < \mu + 1} V_\alpha =
	\lpre$.

	Now suppose that $X$ is a solution to $\kb$ with respect to $U$. All the
	properties of $X$ propagate to $Y$, so one only needs to check that $\mu$ is
	handled correctly. In case $\mu$ is a limit ordinal, we need to show that
	$Y_\mu = \foint$, which it does. On the other hand, if $\mu$ is a nonlimit
	ordinal, then let $\gamma$ be such that $\gamma + 1 = \mu$. From
	$\bigcup_{\alpha < \mu} U_\alpha = \lpre$ it follows that $U_\gamma =
	\lpre$, so the set
	\[
		e_{U_\gamma} \left( b_{U_\mu}(\kb),
			\textstyle \bigcap_{\eta \leq \gamma} Y_\eta \right)
	\]
	is empty. Consequently, $Y_\mu = \foint$ is its MKNF model.
\end{proof}

\begin{theorem*}[Splitting Sequence Theorem for Hybrid MKNF Knowledge Bases]{thm:seq}
	Let $U = \an{U_\alpha}_{\alpha < \mu}$ be a splitting sequence for a hybrid
	knowledge base $\kb$. Then $M$ is an MKNF model of $\kb$ if and only if $M =
	\bigcap_{\alpha < \mu} X_\alpha$ for some solution $\an{X_\alpha}_{\alpha <
	\mu}$ to $\kb$ with respect to $U$.
\end{theorem*}
\begin{proof}[Proof of Theorem \ref{thm:seq}]
	\label{proof:thm:seq}%
	If $M$ is an MKNF model of $\kb$, then it follows by Proposition
	\ref{prop:seq:1} that there is a solution $X$ to $\kb$ with respect to $U$.
	Furthermore, by Lemma \ref{lemma:seq:prolonging}, there is also a solution
	$Y$ to $\kb$ with respect to $V = \an{V_\alpha}_{\alpha < \mu + 1}$ such
	that for all $\alpha < \mu$, $V_\alpha = U_\alpha$ and $V_\mu = \lpre$.
	Consequently, by Lemma \ref{lemma:seq:sigma},
	\[
		\bigcap_{\alpha < \mu} X_\alpha
			= \bigcap_{\alpha < \mu} Y_\alpha
			= \foint \cap \bigcap_{\alpha < \mu} Y_\alpha
			= \bigcap_{\alpha < \mu + 1} Y_\alpha
			= \sat[V_\mu]{M}
			= \sat[\lpre]{M}
			= M \enspace.
	\]

	To prove the converse implication, suppose $X$ is a solution to $\kb$ with
	respect to $U$. Then, by Lemma \ref{lemma:seq:prolonging}, there is also a
	solution $Y$ to $\kb$ with respect to $V = \an{V_\alpha}_{\alpha < \mu + 1}$
	such that for all $\alpha < \mu$, $V_\alpha = U_\alpha$ and $V_\mu = \lpre$.
	Furthermore,
	\[
		\bigcap_{\eta < \mu} X_\eta = \bigcap_{\eta \leq \mu} Y_\eta
	\]
	and by Proposition \ref{prop:seq:2}, $\bigcap_{\eta \leq \mu} Y_\eta$ is an
	MKNF model of $b_{V_\mu}(\kb) = b_{\lpre}(\kb) = \kb$.
\end{proof}

\begin{corollary} \label{cor:seq:1}
	Let $U = \an{U_\alpha}_{\alpha < \mu}$ be a splitting sequence for a hybrid
	knowledge base $\kb$ and $M \in \mint$. If $M$ is an MKNF model of $\kb$,
	then the sequence $X = \an{X_\alpha}_{\alpha < \mu}$ is a solution to $\kb$
	with respect to $U$ where
	\begin{itemize}
		\item $X_0 = \sat[U_0]{M}$;
		\item for all $\alpha$ such that $\alpha + 1 < \mu$, $X_{\alpha + 1} =
			\sat[U_{\alpha + 1} \setminus U_\alpha]{M}$;
		\item for any limit ordinal $\alpha < \mu$, $X_\alpha = \foint$.
	\end{itemize}
	Furthermore, $M = \bigcap_{\alpha < \mu} X_\alpha$.
\end{corollary}
\begin{proof}
	Follows from the proof of Theorem \ref{thm:seq} and from Proposition
	\ref{prop:seq:1}.
\end{proof}

\begin{corollary} \label{cor:seq:2}
	Let $U = \an{U_\alpha}_{\alpha < \mu}$ be a splitting sequence for a hybrid
	knowledge base $\kb$, such that there exists at least one solution to $\kb$
	with respect to $U$. Then $\kb$ is MKNF satisfiable, and $M$ is an MKNF
	model of $\kb$ if and only if $M = \bigcap_{\alpha < \mu} X_\alpha$ for some
	solution $\an{X_\alpha}_{\alpha < \mu}$ to $\kb$ with respect to $U$
\end{corollary}
\begin{proof}
	Follows by Theorem \ref{thm:seq}.
\end{proof}

\begin{corollary} \label{cor:seq:3}
	Let $U = \an{U_\alpha}_{\alpha < \mu}$ be a splitting sequence for a hybrid
	knowledge base $\kb$. A formula $\phi$ is an MKNF consequence of $\kb$ if
	and only if, for every solution $\an{X_\alpha}_{\alpha < \mu}$ to $\kb$ with
	respect to $U$, $\bigcap_{\alpha < \mu} X_\alpha \ent \phi$.
\end{corollary}
\begin{proof}
	Follows by Theorem \ref{thm:seq}.
\end{proof}

\section{Splitting Theorem for Minimal Change Update Operator}

\subsection{Basic Properties of Minimal Change Update Operator}

\begin{lemma} \label{lemma:pmasplitting:pma subset}
	Let $M, N \in \mint$. If $M$ is a subset of $N$, then $M \oplus N = M$.
	Also, if $M$ is a superset of $N$, then $M \oplus N = N$.
\end{lemma}
\begin{proof}
	Suppose $M$ is a subset of $N$. It can be verified easily that for every
	Herbrand interpretation $I$, $I <_I J$ for every $J \neq I$. Thus, for all
	$I \in M$, $I \oplus N = \set{I}$ and, consequently, $M \oplus N = M$.

	Now suppose that $M$ is a superset of $N$. Then for every $I \in N$ we have
	$I \oplus N = \set{I}$, so $M \oplus N$ is a superset of $N$. At the same
	time, $M \oplus N$ is a subset of $N$ by construction. Thus, $M \oplus N =
	N$.
\end{proof}

\begin{corollary} \label{cor:pmasplitting:pma idempotent}
	Let $M \in \mint$. Then $M \oplus M = M$.
\end{corollary}
\begin{proof}
	Follows from Lemma \ref{lemma:pmasplitting:pma subset}.
\end{proof}

\begin{lemma} \label{lemma:pmasplitting:update emptiness}
	Let $M, N \in \mint$. If $M$ is empty or $N$ is empty, then $M \oplus N$ is
	empty.
\end{lemma}
\begin{proof}
	Follows from Lemma \ref{lemma:pmasplitting:pma subset}.
\end{proof}

\subsection{Splitting and Updating a Sequence of First-Order Theories}

We will be using the following, inductive definition of a minimal change
update model of a sequence of first-order theories. It is equivalent to the
definition from Sect.~\ref{sect:preliminaries}, just more precise.

\begin{definition}[Minimal Change Update Model] \label{def:seq pma}
	Let $\upd = \an{\upd_i}_{i < n}$ be a finite sequence of first-order
	theories. We define:
	\begin{align*}
		\smod(\upd) &= \bigoplus \an{\smod(\upd_i)}_{i < n} \enspace \text{, where} \\
		\bigoplus \an{~} &= \foint \enspace \text{and} \\
		\bigoplus \an{M_i}_{i < k + 1}
			&= \left( \bigoplus \an{M_i}_{i < k} \right) \oplus M_k
			\enspace \text{for any $k \geq 0$.}
	\end{align*}
	If $\smod(\upd)$ is nonempty, we say it is the \emph{minimal
	change update model of $\upd$}.
\end{definition}

\begin{remark}
	Note that the above definition is compatible with the definition of a
	minimal change update of $\thr \oplus \upd$ in the following sense:
	$M$ is a minimal change update model of $\thr \oplus \upd$ if and only if
	$M$ is a minimal change update model of $\an{\thr, \upd}$. This follows from
	the above definition and from the fact that $\foint \oplus X = X$ for any $X
	\in \mint$ (by Lemma \ref{lemma:pmasplitting:pma subset}).
\end{remark}

\begin{lemma} \label{lemma:pmasplitting:fo theory model}
	Let $\thr$ be a first-order theory. Then either $\thr$ has no MKNF model and
	$\smod(\thr) = \emptyset$, or $\thr$ has a unique MKNF model that coincides
	with $\smod(\thr)$. Moreover,
	\[
		\smod(\thr) = \bigcap_{\phi \in \thr} \smod(\phi)
	\]
	Additionally, the S5 and MKNF models of $\thr$ coincide with S5 and MKNF
	models of $\set{\mk \phi | \phi \in \thr}$.
\end{lemma}
\begin{proof}
	Follows from the definition of an S5 and MKNF models.
\end{proof}

\begin{proposition} \label{prop:pmasplitting:generalisation}
	Let $\thr$ be a first-order theory. Then $M$ is the MKNF model of $\thr$ if
	and only if $M$ is the minimal change update model of $\an{\thr}$.
\end{proposition}
\begin{proof}
	Follows by Lemma \ref{lemma:pmasplitting:fo theory model} and the fact that
	$\foint \oplus M = M$ for any $M \in \mint$.
\end{proof}

\begin{proposition} \label{prop:pmasplitting:primacy}
	Let $\upd = \an{\upd_i}_{i < n}$ be a sequence of first-order theories with
	$n > 0$ and $M$ be the minimal change update model of $\upd$.  Then $M \ent
	\upd_{n-1}$.
\end{proposition}
\begin{proof}[Proof (sketch)]
	This follows by induction on $n$, by the fact that $M \oplus N$ is a subset
	of $N$ for any $M, N \in \mint$.
\end{proof}

\begin{proposition} \label{prop:pmasplitting:model iteration}
	Let $\upd = \an{\upd_i}_{i < n + 1}$ be a nonempty sequence of first-order
	theories and $M$ be an MKNF interpretation. Then $M$ is a minimal change
	update model of $\upd$ if and only if $M = N_1 \oplus N_2$ where $N_1$ is
	the minimal change update model of $\an{\upd_i}_{i < n}$ and $N_2$ is the
	MKNF model of $\upd_n$.
\end{proposition}
\begin{proof}
	By Def.~\ref{def:seq pma}, $M$ is the minimal change update model of $\upd$
	if and only if
	\[
		M = \smod(\upd) = N_1 \oplus N_2 \enspace,
	\]
	where $N_1 = \bigoplus \an{\smod(\upd_i)}_{i < n}$ and $N_2 =
	\smod(\upd_n)$. Furthermore, since $M$ is nonempty, it follows by Lemma
	\ref{lemma:pmasplitting:update emptiness} that both $N_1$ and $N_2$ are
	nonempty. So $N_1$ is the minimal change update model of $\an{\upd_i}_{i <
	n}$ by Def.~\ref{def:seq pma} and $N_2$ is the MKNF model of $\upd_n$ by
	Lemma \ref{lemma:pmasplitting:fo theory model}.
\end{proof}

\begin{proposition} \label{prop:pmasplitting:intupdate}
	Let $S = \an{S_\alpha}_{\alpha < \mu}$ be a saturation sequence, $N \in
	\mint$ be sequence-saturated relative to $S$ and $J$ be a Herbrand
	interpretation. Then
	\[
		J \in I \oplus N \mlequiv (\forall \alpha < \mu)
			(\restr[S_\alpha]{J} \in \restr[S_\alpha]{I} \oplus \restr[S_\alpha]{N})
			\enspace.
	\]
\end{proposition}
\begin{proof}
	Suppose $J$ does not belong to $I \oplus N$. If $J$ does not belong to $N$,
	then since $N$ is sequence-saturated relative to $S$, there is some $\alpha
	< \mu$ such that $\restr[S_\alpha]{J}$ does not belong to
	$\restr[S_\alpha]{N}$.  But then $\restr[S_\alpha]{J}$ also does not belong
	to $\restr[S_\alpha]{I} \oplus \restr[S_\alpha]{N}$, so we reached the
	desired conclusion.
	
	In the principal case, when $J$ belongs to $N$, we know there exists some
	$J' \in N$ such that $J' <_I J$. This means that for every predicate symbol
	$P \in \lpre$,
	\begin{equation}
		\label{eq:proof:pmasplitting:intupdate:1}
		\mathit{diff}(P, I, J') \subseteq \mathit{diff}(P, I, J)
	\end{equation}
	and for some predicate symbol $P_0 \in \lpre$,
	\begin{equation}
		\label{eq:proof:pmasplitting:intupdate:2}
		\mathit{diff}(P_0, I, J') \subsetneq \mathit{diff}(P_0, I, J) \enspace.
	\end{equation}
	Since $S$ is a saturation sequence, there is a unique ordinal $\alpha$ such
	that $P_0$ belongs to $S_\alpha$. It follows from
	\eqref{eq:proof:pmasplitting:intupdate:2} that
	\[
		\mathit{diff}(P_0, \restr[S_\alpha]{I}, \restr[S_\alpha]{J'})
			= \mathit{diff}(P_0, I, J')
			\subsetneq \mathit{diff}(P_0, I, J)
			= \mathit{diff}(P_0, \restr[S_\alpha]{I}, \restr[S_\alpha]{J})
	\]
	Furthermore, for any predicate symbol $P \in S_\alpha$ it follows from
	\eqref{eq:proof:pmasplitting:intupdate:1} that
	\[
		\mathit{diff}(P, \restr[S_\alpha]{I}, \restr[S_\alpha]{J'})
			= \mathit{diff}(P, I, J')
			\subseteq \mathit{diff}(P, I, J)
			= \mathit{diff}(P, \restr[S_\alpha]{I}, \restr[S_\alpha]{J}) \enspace.
	\]
	Finally, for any predicate symbol $P$ that does not belong to $S_\alpha$,
	$\restr[S_\alpha]{I}$, $\restr[S_\alpha]{J}$ and $\restr[S_\alpha]{J'}$ do
	not contain any atoms with the predicate symbol $P$, so
	\[
		\mathit{diff}(P, \restr[S_\alpha]{I}, \restr[S_\alpha]{J'}) = \emptyset
			= \mathit{diff}(P, \restr[S_\alpha]{I}, \restr[S_\alpha]{J}) \enspace.
	\]
	Thus, we can conclude that
	\[
		\restr[S_\alpha]{J'} <_{\restr[S_\alpha]{I}} \restr[S_\alpha]{J} \enspace,
	\]
	so $\restr[S_\alpha]{J}$ does not belong to $\restr[S_\alpha]{I} \oplus
	\restr[S_\alpha]{N}$, which finishes this part of the proof.

	For the converse implication, suppose that for some ordinal $\alpha < \mu$,
	$\restr[S_\alpha]{J}$ does not belong to $\restr[S_\alpha]{I} \oplus
	\restr[S_\alpha]{N}$. If $\restr[S_\alpha]{J}$ does not belong to
	$\restr[S_\alpha]{N}$, we immediately obtain that $J$ does not belong to
	$N$. Consequently, $J$ cannot belong to $I \oplus N$.

	It remains to consider the principal case when $\restr[S_\alpha]{J}$ belongs
	to $\restr[S_\alpha]{N}$. Then there must be some interpretation $J' \in
	\restr[S_\alpha]{N}$ such that $J' <_{\restr[S_\alpha]{I}} \restr[S_\alpha]{J}$.
	Thus, for all predicate symbols $P \in \lpre$ we know that
	\begin{equation}
		\label{eq:proof:pmasplitting:intupdate:3}
		\mathit{diff}(P, \restr[S_\alpha]{I}, J')
			\subseteq \mathit{diff}(P, \restr[S_\alpha]{I}, \restr[S_\alpha]{J})
			\enspace.
	\end{equation}
	We also know that there is some predicate symbol $P_0 \in \lpre$ such that
	\begin{equation}
		\label{eq:proof:pmasplitting:intupdate:4}
		\mathit{diff}(P_0, \restr[S_\alpha]{I}, J')
			\subsetneq \mathit{diff}(P_0, \restr[S_\alpha]{I}, \restr[S_\alpha]{J})
			\enspace.
	\end{equation}
	Additionally, for every predicate symbol $P$ from $\lpre \setminus S_\alpha$
	it holds that $\restr[S_\alpha]{I}$, $\restr[S_\alpha]{J}$ and $J'$ contain
	no atoms with the predicate symbol $P$, so that
	\[
		\mathit{diff}(P, \restr[S_\alpha]{I}, J') = \emptyset
			= \mathit{diff}(P, \restr[S_\alpha]{I}, \restr[S_\alpha]{J}) \enspace.
	\]
	Consequently, $P_0$ must belong to $S_\alpha$. Now let $J'' = J' \cup
	\restrc[S_\alpha]{J}$. It is easy to see that $\restr[S_\alpha]{J''} = J'
	\in \restr[S_\alpha]{N}$ and for every ordinal $\beta < \mu$ such that
	$\beta \neq \alpha$, $\restr[S_\beta]{J''} = \restr[S_\beta]{J} \in
	\restr[S_\beta]{N}$, so since $N$ is sequence-saturated relative to $S$,
	$J''$ belongs to $N$. Now take some predicate symbol $P \in \lpre$ and
	consider the following two cases:
	\begin{enumerate}
		\renewcommand{\labelenumi}{\alph{enumi})}
		\item If $P$ belongs to $S_\alpha$, then from
			\eqref{eq:proof:pmasplitting:intupdate:3} we obtain
			\[
				\mathit{diff}(P, I, J'')
					= \mathit{diff}(P, \restr[S_\alpha]{I}, J')
					\subseteq \mathit{diff}(P, \restr[S_\alpha]{I}, \restr[S_\alpha]{J})
					\enspace.
			\]
		\item If $P$ belongs to $\lpre \setminus S_\alpha$, then since
			$\restrc[S_\alpha]{J''} = \restrc[S_\alpha]{J}$,
			\[
				\mathit{diff}(P, I, J'') = \mathit{diff}(P, I, J) \enspace.
			\]
	\end{enumerate}
	In both cases we see that $\mathit{diff}(P, I, J'')$ is a subset of
	$\mathit{diff}(P, I, J)$. Moreover, from
	\eqref{eq:proof:pmasplitting:intupdate:4} we obtain
	\[
		\mathit{diff}(P_0, I, J'')
			= \mathit{diff}(P_0, \restr[S_\alpha]{I}, J')
			\subsetneq \mathit{diff}(P_0, \restr[S_\alpha]{I}, \restr[S_\alpha]{J})
			\enspace.
	\]
	It follows from the above considerations that $J'' <_I J$. Consequently, $J$
	does not belong to $I \oplus N$.
\end{proof}

\begin{proposition} \label{prop:pmasplitting:mintupdate}
	Let $S = \an{S_\alpha}_{\alpha < \mu}$ be a saturation sequence, $M, N \in
	\mint$ be both sequence-saturated relative to $S$ and $J$ be a Herbrand
	interpretation. Then
	\[
		J \in M \oplus N \mlequiv (\forall \alpha < \mu)
			(\restr[S_\alpha]{J} \in \restr[S_\alpha]{M} \oplus \restr[S_\alpha]{N})
			\enspace.
	\]
\end{proposition}
\begin{proof}
	We know that $J$ belongs to $M \oplus N$ if and only if for some $I \in M$,
	$J$ belongs to $I \oplus N$. By Proposition
	\ref{prop:pmasplitting:intupdate}, this holds if and only if
	\begin{equation}
		\label{eq:proof:pmasplitting:mintupdate:1}
		(\forall \alpha < \mu)
			(\restr[S_\alpha]{J} \in \restr[S_\alpha]{I} \oplus \restr[S_\alpha]{N})
			\enspace.
	\end{equation}
	At the same time, the right hand side of our equivalence is true if and only
	for some sequence of Herbrand interpretations $\an{I_\alpha}_{\alpha <
	\mu}$, the following holds:
	\begin{equation}
		\label{eq:proof:pmasplitting:mintupdate:2}
		(\forall \alpha < \mu)
			(\restr[S_\alpha]{J} \in
			\restr[S_\alpha]{I_\alpha} \oplus \restr[S_\alpha]{N})
			\enspace.
	\end{equation}
	It remains to show that \eqref{eq:proof:pmasplitting:mintupdate:1} is
	equivalent to \eqref{eq:proof:pmasplitting:mintupdate:2}. Indeed, it is
	easy to see that \eqref{eq:proof:pmasplitting:mintupdate:1} implies
	\eqref{eq:proof:pmasplitting:mintupdate:2} by putting $I_\alpha = I$ for
	all $\alpha < \mu$. Now suppose that
	\eqref{eq:proof:pmasplitting:mintupdate:2} holds and put
	\[
		I = \bigcup_{\alpha < \mu} \restr[S_\alpha]{I_\alpha} \enspace.
	\]
	Then it holds for every $\alpha < \mu$ that $\restr[S_\alpha]{I} =
	\restr[S_\alpha]{I_\alpha} \in \restr[S_\alpha]{M}$. Since $M$ is
	sequence-saturated relative to $S$, this implies that $I$ belongs to $M$.
	Moreover, we can also conclude that $\restr[S_\alpha]{J}$ belongs to
	$\restr[S_\alpha]{I} \oplus \restr[S_\alpha]{N}$. As a consequence,
	\eqref{eq:proof:pmasplitting:mintupdate:1} is satisfied and our proof is
	finished.
\end{proof}

\begin{proposition} \label{prop:pmasplitting:equivalence1}
	Let $S = \an{S_\alpha}_{\alpha < \mu}$ be a saturation sequence, $M, N \in
	\mint$ be both sequence-saturated relative to $S$ and $J$ be a Herbrand
	interpretation. If $M \oplus N$ is nonempty, then 
	\[
		\restr[S_\alpha]{(M \oplus N)}
			= \restr[S_\alpha]{M} \oplus \restr[S_\alpha]{N} \enspace.
	\]
\end{proposition}
\begin{proof}
	Suppose that $J$ belongs to $\restr[S_\alpha]{(M \oplus N)}$. Then $M \oplus
	N$ contains some interpretation $J'$ such that $\restr[S_\alpha]{J'} = J$.
	Also, by Proposition \ref{prop:pmasplitting:mintupdate} it follows that
	$\restr[S_\alpha]{J'}$ belongs to $\restr[S_\alpha]{M} \oplus
	\restr[S_\alpha]{N}$, so $J$ also belongs there.

	For the converse implication, suppose that $J$ belongs to
	$\restr[S_\alpha]{M} \oplus \restr[S_\alpha]{N}$, take some $J' \in M \oplus
	N$ and put $J'' = J \cup \restrc[S_\alpha]{J'}$. By Proposition
	\ref{prop:pmasplitting:mintupdate} it follows that for every $\beta < \mu$,
	$\restr[S_\beta]{J'}$ belongs to $\restr[S_\beta]{M} \oplus
	\restr[S_\beta]{N}$. Also, whenever $\beta \neq \alpha$,
	$\restr[S_\beta]{J''} = \restr[S_\beta]{J'}$, so $\restr[S_\beta]{J''}$
	belongs to $\restr[S_\beta]{M} \oplus \restr[S_\beta]{N}$. Moreover,
	$\restr[S_\alpha]{J''} = J$, so $\restr[S_\alpha]{J''}$ belongs to
	$\restr[S_\alpha]{M} \oplus \restr[S_\alpha]{N}$, and by using Proposition
	\ref{prop:pmasplitting:mintupdate} again we obtain that $J''$ belongs to $M
	\oplus N$. As a consequence, $J$ belongs to $\restr[S_\alpha]{(M \oplus N)}$
	because $J = \restr[S_\alpha]{J''}$.
\end{proof}

\begin{proposition} \label{prop:pmasplitting:equivalence2}
	Let $S = \an{S_\alpha}_{\alpha < \mu}$ be a saturation sequence, $M, N \in
	\mint$ be both nonempty and sequence-saturated relative to $S$ and $J$ be a
	Herbrand interpretation. Then
	\[
		\restr[S_\alpha]{J} \in \restr[S_\alpha]{M} \oplus \restr[S_\alpha]{N}
			\mlequiv J \in \sat[S_\alpha]{M} \oplus \sat[S_\alpha]{N} \enspace.
	\]
\end{proposition}
\begin{proof}
	By applying Proposition \ref{prop:pmasplitting:mintupdate} on
	$\sat[S_\alpha]{M}$ and $\sat[S_\alpha]{N}$ it follows that for any Herbrand
	interpretation $J$,
	\begin{equation} \label{eq:proof:pmasplitting:equivalence2:1}
		J \in \sat[S_\alpha]{M} \oplus \sat[S_\alpha]{N} \mlequiv
			(\forall \beta < \mu) \left(
				\restr[S_\beta]{J} \in
				\restr[S_\beta]{\sat[S_\alpha]{M}} \oplus
				\restr[S_\beta]{\sat[S_\alpha]{N}}
			\right) \enspace.
	\end{equation}
	By Lemma \ref{lemma:pmasplitting:restr sat disjoint} it follows that
	whenever $\beta \neq \alpha$, $\restr[S_\beta]{\sat[S_\alpha]{M}} =
	\restr[S_\beta]{\foint}$ and $\restr[S_\beta]{\sat[S_\alpha]{M}} =
	\restr[S_\beta]{\foint}$, so that by Corollary \ref{cor:pmasplitting:pma
	idempotent} we obtain that $\restr[S_\beta]{\sat[S_\alpha]{M}} \oplus
	\restr[S_\beta]{\sat[S_\alpha]{N}} = \restr[S_\beta]{\foint}$. Thus,
	condition \eqref{eq:proof:pmasplitting:equivalence2:1} gets simplified to
	\[
		J \in \sat[S_\alpha]{M} \oplus \sat[S_\alpha]{N} \mlequiv
			\restr[S_\alpha]{J} \in
			\restr[S_\alpha]{\sat[S_\alpha]{M}} \oplus
			\restr[S_\alpha]{\sat[S_\alpha]{N}}
			\enspace.
	\]
	Furthermore, by Proposition \ref{prop:sigma},
	$\restr[S_\alpha]{\sat[S_\alpha]{M}} = \restr[S_\alpha]{M}$ and
	s$\restr[S_\alpha]{\sat[S_\alpha]{N}} = \restr[S_\alpha]{N}$, so we obtain 
	\[
		J \in \sat[S_\alpha]{M} \oplus \sat[S_\alpha]{N} \mlequiv
			\restr[S_\alpha]{J} \in
			\restr[S_\alpha]{M} \oplus
			\restr[S_\alpha]{N}
			\enspace.
	\]
	This completes our proof.
\end{proof}

\begin{corollary} \label{cor:pmasplitting:equivalence}
	Let $S = \an{S_\alpha}_{\alpha < \mu}$ be a saturation sequence, $M, N \in
	\mint$ be both sequence-saturated relative to $S$ and $J$ be a Herbrand
	interpretation. If $M \oplus N$ is nonempty, then 
	\[
		\sat[S_\alpha]{M \oplus N}
			= \sat[S_\alpha]{M} \oplus \sat[S_\alpha]{N} \enspace.
	\]
\end{corollary}
\begin{proof}
	Lemma \ref{lemma:pmasplitting:update emptiness} implies that since $M \oplus
	N$ is nonempty, both $M$ and $N$ must also be nonempty. Furthermore, an
	interpretation $J$ belongs to $\sat[S_\alpha]{M \oplus N}$ if and only if
	$\restr[S_\alpha]{J}$ belongs to $\restr[S_\alpha]{(M \oplus N)}$. By
	Proposition \ref{prop:pmasplitting:equivalence1} this holds if and only if
	$\restr[S_\alpha]{J}$ belongs to $\restr[S_\alpha]{M} \oplus
	\restr[S_\alpha]{N}$. Finally, by Proposition
	\ref{prop:pmasplitting:equivalence2}, this holds if and only if $J$ belongs
	to $\sat[S_\alpha]{M} \oplus \sat[S_\alpha]{N}$, which finishes our proof.
\end{proof}

\begin{definition}[Saturation Sequence Induced by a Splitting Sequence]
	Let $U = \an{U_\alpha}_{\alpha < \mu}$ be a splitting sequence. The
	\emph{saturation sequence induced by $U$} is the sequence
	$\an{S_\alpha}_{\alpha < \mu}$ where
	\begin{itemize}
		\item $S_0 = U_0$;
		\item for any ordinal $\alpha$ such that $\alpha + 1 < \mu$, $S_{\alpha +
			1} = U_{\alpha + 1} \setminus U_\alpha$;
		\item for any limit ordinal $\alpha$, $S_\alpha = \emptyset$.
	\end{itemize}
\end{definition}

\begin{proposition} \label{prop:pmasplitting:fo splitting}
	Let $\thr$ be a first order theory, $U$ be a splitting sequence for $\thr$
	and $S = \an{S_\alpha}_{\alpha < \mu}$ be the saturation sequence induced by
	$U$. If $M$ is the MKNF model of $\thr$, then $M$ is sequence-saturated
	relative to $S$ and for every ordinal $\alpha$ such that $\alpha < \mu$,
	$\sat[S_\alpha]{M}$ is the MKNF model of $b_{S_\alpha}(\thr)$.
\end{proposition}
\begin{proof}
	By Lemma \ref{lemma:pmasplitting:fo theory model} we know that $M =
	\smod(\thr) = \bigcap_{\phi \in \thr} \smod(\phi)$. Also, since $U$ is a
	splitting sequence for $\thr$, for every formula $\phi \in \thr$ there
	exists a unique $\alpha < \mu$ such that $\phi$ belongs to
	$b_{S_\alpha}(\thr)$. Thus, we obtain
	\[
		M = \smod(\thr) = \bigcap_{\phi \in \thr} \smod(\phi)
			= \bigcap_{\alpha < \mu}
				\bigcap_{\phi \in b_{S_\alpha}(\thr)} \smod(\phi)
			= \bigcap_{\alpha < \mu} \smod(b_{S_\alpha}(\thr)) \enspace.
	\]
	Let $X_\alpha = \smod(b_{S_\alpha}(\thr))$, so that
	\[
		M = \bigcap_{\alpha < \mu} X_\alpha \enspace.
	\]
	We can use Lemma \ref{lemma:pmasplitting:fo theory model} to conclude that
	$X_\alpha$ is the MKNF model of $b_{S_\alpha}(\thr)$. Furthermore, since
	$S_\alpha$ includes $\preds{b_{S_\alpha}(\thr)}$, it follows by Proposition
	\ref{prop:saturated:mknf} that $X_\alpha$ is saturated relative to
	$S_\alpha$. Thus, we can now use Proposition \ref{prop:seqsat:int sigma} to
	conclude that $X_\alpha = \sat[S_\alpha]{M}$. It also follows by Proposition
	\ref{prop:seqsat:equivalence} that $M$ is sequence-saturated relative to
	$S$.
\end{proof}

\begin{proposition} \label{prop:pmasplitting:fo splitting 2}
	Let $\thr$ be a first-order theory, $U$ be a splitting sequence for $\thr$
	and $S$ be the saturation sequence induced by $U$. If for every ordinal
	$\alpha < \mu$, $X_\alpha$ is the MKNF model of $b_{S_\alpha}(\thr)$, then
	$\bigcap_{\alpha < \mu} X_\alpha$ is the MKNF model of $\thr$.
\end{proposition}
\begin{proof*}
	By Lemma \ref{lemma:pmasplitting:fo theory model}, it holds for every
	$\alpha < \mu$ that $X_\alpha = \smod(b_{S_\alpha}(\thr))$. Also, since
	$U$ is a splitting sequence for $\thr$, for every formula $\phi \in \thr$
	there exists a unique $\alpha < \mu$ such that $\phi$ belongs to
	$b_{S_\alpha}(\thr)$. Hence,
	\[
		\bigcap_{\alpha < \mu} X_\alpha
			= \bigcap_{\alpha < \mu} \smod(b_{S_\alpha}(\thr))
			= \bigcap_{\alpha < \mu}
				\bigcap_{\phi \in b_{S_\alpha}(\thr)} \smod(\phi)
			= \bigcap_{\phi \in \thr} \smod(\phi)
			= \smod(\thr) \enspace.
	\]
	Furthermore, since $S_\alpha$ includes  $\preds{b_{S_\alpha}(\thr)}$, it
	follows by Proposition \ref{prop:saturated:mknf} that $X_\alpha$ is
	saturated relative to $S_\alpha$. By Proposition \ref{prop:seqsat:int sigma}
	we now obtain that $\bigcap_{\alpha < \mu} X_\alpha$ is nonempty.
	Consequently, by Lemma \ref{lemma:pmasplitting:fo theory model},
	$\bigcap_{\alpha < \mu} X_\alpha = \smod(\thr)$ is the unique MKNF model of
	$\thr$.
\end{proof*}

\begin{lemma} \label{lemma:pmasplitting:int sat}
	Let $U$ be a set of predicate symbols, $N \in \mint$ be saturated relative
	to $U$ and $I, J$ be Herbrand interpretations. If $J$ belongs to $I \oplus
	N$, then $I$ coincides with $J$ on $\lpre \setminus U$.
\end{lemma}
\begin{proof}
	If $J$ belongs to $I \oplus N$, then $J$ also belongs to $N$. Put $J' =
	\restr{J} \cup \restrc{I}$. Then $\restr{J'} = \restr{J} \in N$, so since
	$N$ is saturated relative to $U$, $J'$ belongs to $N$. Furthermore, for any
	predicate symbol $P \in U$, $\mathit{diff}(P, I, J') = \mathit{diff}(P, I,
	J)$ and for any predicate symbol $P \in \lpre \setminus U$,
	$\mathit{diff}(P, I, J') = \emptyset \subseteq \mathit{diff}(P, I, J)$. If
	this inclusion was proper for some predicate symbol $P$, then we would
	obtain that $J' <_I J$ holds, contrary to the assumption that $J$ belongs to
	$I \oplus N$. Thus, for all predicate symbols $P \in \lpre \setminus U$,
	$\mathit{diff}(P, I, J)$ must be equal to $\emptyset$. It follows that
	$\restrc{J} = \restrc{J'} = \restrc{I}$, which is the desired result.
\end{proof}

\begin{lemma} \label{lemma:pmasplitting:pma sat}
	Let $U$ be a set of predicate symbols and $M, N \in \mint$ be both saturated
	relative to $U$. Then $M \oplus N$ is also saturated relative to $U$.
\end{lemma}
\begin{proof}
	Suppose that $J$ is a Herbrand interpretation such that $\restr{J}$ belongs
	to $\restr{(M \oplus N)}$ but $J$ does not belong to $M \oplus N$. Then
	there exists some interpretation $J'$ from $M \oplus N$ such that
	$\restr{J'} = \restr{J}$. This also implies that $J'$ belongs to $N$ and
	since $N$ is saturated relative to $U$, $J$ also belongs to $N$.
	Furthermore, there must exist some interpretation $I \in M$ such that $J'$
	belongs to $I \oplus N$. By Lemma \ref{lemma:pmasplitting:int sat} we obtain
	that $\restr{I} = \restr{J'} = \restr{J}$. Let $I' = \restr{I} \cup
	\restrc{J}$. Then $\restr{I'} = \restr{I}$ and $\restrc{I'} = \restrc{J}$
	and since $M$ is saturated relative to $U$, $I'$ belongs to $M$. Since $J$
	does not belong to $M \oplus N$, there must exist some interpretation $J''
	\in N$ such that $J'' <_{I'} J$.  This means that for any predicate symbol
	$P \in U$ we have
	\[
		\mathit{diff}(P, I', J'') \subseteq \mathit{diff}(P, I', J)
	\]
	and for every predicate symbol $P \in \lpre \setminus U$ we have
	\[
		\mathit{diff}(P, I', J'') = \mathit{diff}(P, I', J) = \emptyset
	\]
	because $I'$ coincides with $J$ on $\lpre \setminus U$. There must also
	exist some predicate symbol $P_0$ such that
	\[
		\mathit{diff}(P_0, I', J'') \subsetneq \mathit{diff}(P_0, I', J) \enspace.
	\]
	Since this is impossible if $P_0$ belongs to $\lpre \setminus U$, $P_0$ must
	belong to $U$. Now let $J''' = \restr{J''} \cup \restrc{J'}$. For predicate
	symbols $P \in U$ we have
	\[
		\mathit{diff}(P, I, J''') = \mathit{diff}(P, I', J'') \subseteq
		\mathit{diff}(P, I', J) = \mathit{diff}(P, I, J')
	\]
	and for predicate symbols $P \in \lpre \setminus U$ we have
	\[
		\mathit{diff}(P, I, J''') = \mathit{diff}(P, I, J') = \emptyset \enspace.
	\]
	Also, for $P_0$ we obtain
	\[
		\mathit{diff}(P_0, I, J''') = \mathit{diff}(P_0, I', J'') \subsetneq
		\mathit{diff}(P_0, I', J) = \mathit{diff}(P_0, I, J')
	\]
	As a consequence, $J''' <_I J'$, which is in conflict with the assumption
	that $J'$ belongs to $I \oplus N$.
\end{proof}

\begin{proposition} \label{prop:pmasplitting:seq pma sat}
	Let $U$ be a set of predicate symbols and $\upd = \an{\upd_i}_{i < n}$ be a
	finite sequence of first-order theories such that for every $i < n$,
	$\preds{\upd_i}$ is included in $U$. Then the minimal change update model of
	$\upd$ is saturated relative to $U$.
\end{proposition}
\begin{proof}
	Follows using induction on $i$ by Proposition \ref{prop:saturated:mknf} and by
	Lemma \ref{lemma:pmasplitting:pma sat}.
\end{proof}

\begin{proposition} \label{prop:pmasplitting:1}
	Let $\upd = \an{\upd_i}_{i < n}$ be a finite sequence of first-order
	theories, $U$ be a splitting sequence for $\upd$ and $S =
	\an{S_\alpha}_{\alpha < \mu}$ be the saturation sequence induced by $U$. If
	$M$ is the minimal change update model of $\upd$, then $M$ is
	sequence-saturated relative to $S$ and for every ordinal $\alpha < \mu$,
	$\sat[S_\alpha]{M}$ is the minimal change update model of
	$b_{S_\alpha}(\upd)$.
\end{proposition}
\begin{proof*}
	We prove by induction on $n$.
	\begin{enumerate}
		\renewcommand{\labelenumi}{\arabic{enumi}$^\circ$}
		\item If $n = 0$, then $\upd = \an{~}$, so $M = \foint$ by definition.
			Thus, $M$ is trivially sequence-saturated relative to $S$. Moreover, for
			every $\alpha < \mu$, $b_{S_\alpha}(\upd) = \an{~}$, and
			$\sat[S_\alpha]{M} = \sat[S_\alpha]{\foint} = \foint$, so
			$\sat[S_\alpha]{M}$ is the minimal change update model of
			$b_{S_\alpha}(\upd)$.
		
		\item We assume the claim holds for $n$ and prove it for $n + 1$.  Suppose
			$M'$ is the minimal change update model of $\upd' = \an{\upd_i}_{i < n +
			1}$. It follows from Proposition \ref{prop:pmasplitting:model iteration}
			that $M' = M \oplus N$ where $M$ is the minimal change update model of
			$\upd = \an{\upd_i}_{i < n}$ and $N$ is the MKNF model of $\upd_n$.  We
			need to prove that for every $\alpha < \mu$, $\sat[S_\alpha]{M'}$ is the
			minimal change update model of $b_{S_\alpha}(\upd')$.

			Take some arbitrary but fixed $\alpha < \mu$. By the inductive
			assumption we obtain that $M$ is sequence-saturated relative to $S$ and
			$\sat[S_\alpha]{M}$ is the minimal change update model of
			$b_{S_\alpha}(\upd)$. Also, by Proposition \ref{prop:pmasplitting:fo
			splitting}, $N$ is sequence-saturated relative to $S$ and
			$\sat[S_\alpha]{N}$ is the MKNF model of $b_{S_\alpha}(\upd_n)$.
			Furthermore, by Corollary \ref{cor:pmasplitting:equivalence},
			\[
				\sat[S_\alpha]{M'} = \sat[S_\alpha]{M \oplus N}
					= \sat[S_\alpha]{M} \oplus \sat[S_\alpha]{N} \enspace,
			\]
			and by another application of Proposition \ref{prop:pmasplitting:model
			iteration} we obtain that $\sat[S_\alpha]{M'}$ is the minimal change
			update model of $b_{S_\alpha}(\upd')$.

			It remains to show that $M'$ is sequence-saturated relative to $S$.
			Suppose that $J$ is a Herbrand interpretation such that
			$\restr[S_\alpha]{J}$ belongs to $\restr[S_\alpha]{M'}$. We conclude by
			Proposition \ref{prop:pmasplitting:equivalence1} that for any $\alpha <
			\mu$, $\restr[S_\alpha]{J}$ belongs to $\restr[S_\alpha]{M} \oplus
			\restr[S_\alpha]{N}$. Thus, by Proposition
			\ref{prop:pmasplitting:mintupdate}, $J$ belongs to $M \oplus N = M'$ as
			desired. \qedhere
	\end{enumerate}
\end{proof*}

\begin{proposition} \label{prop:pmasplitting:2}
	Let $\upd = \an{\upd_i}$ be a finite sequence of first-order theories, $U$
	be a splitting sequence for $\upd$ and $S = \an{S_\alpha}_{\alpha < \mu}$ be
	the saturation sequence induced by $U$. If for every ordinal $\alpha < \mu$,
	$X_\alpha$ is the minimal change update model of $b_{S_\alpha}(\upd)$, then
	$\bigcap_{\alpha < \mu} X_\alpha$ is the minimal change update model of
	$\upd$.
\end{proposition}
\begin{proof*}
	We prove by induction on $n$.
	\begin{enumerate}
		\renewcommand{\labelenumi}{\arabic{enumi}$^\circ$}
		\item If $n = 0$, then $\upd = \an{~}$ and for every $\alpha < \mu$,
			$b_{S_\alpha}(\upd) = \an{~}$, so $X_\alpha = \foint$ and $M = \bigcap_{\alpha <
			\mu} X_\alpha = \foint$, so $M$ is indeed the minimal change update
			model of $\upd$.

		\item We assume the claim holds for $n$ and prove it for $n + 1$. Suppose
			that for every ordinal $\alpha < \mu$, $X_\alpha'$ is the minimal change
			update model of $b_{S_\alpha}(\upd')$ where $\upd' = \an{\upd_i}_{i < n
			+ 1}$. We need to show that $\bigcap_{\alpha < \mu} X_\alpha'$ is the
			minimal change update model of $\upd'$.

			It follows from Proposition \ref{prop:pmasplitting:model iteration} that
			$X_\alpha' = X_\alpha \oplus Y_\alpha$ where $X_\alpha$ is the minimal
			change update model of $b_{S_\alpha}(\upd)$, where $\upd =
			\an{\upd_i}_{i < n}$, and $Y_\alpha$ is the MKNF model of
			$b_{S_\alpha}(\upd_n)$. Thus, by the inductive assumption, $M =
			\bigcap_{\alpha < \mu} X_\alpha$ is the minimal change update model of
			$\upd$. Also, by Proposition \ref{prop:pmasplitting:fo splitting 2},
			$N = \bigcap_{\alpha < \mu} Y_\alpha$ is the MKNF model of $\upd_n$.
			Moreover, we know from Proposition \ref{prop:pmasplitting:seq pma sat}
			that both $X_\alpha$ and $Y_\alpha$ are saturated relative to
			$S_\alpha$. Thus, we can use Proposition \ref{prop:seqsat:int sigma} to
			conclude that $X_\alpha = \sat[S_\alpha]{M}$ and $Y_\alpha =
			\sat[S_\alpha]{N}$ and that both $M$ and $N$ are nonempty and by
			Proposition \ref{prop:seqsat:equivalence} they are also
			sequence-saturated relative to $S$. We can thus apply Propositions
			\ref{prop:pmasplitting:mintupdate} and
			\ref{prop:pmasplitting:equivalence2} to obtain
			\[
				M \oplus N = \bigcap_{\alpha < \mu} \sat[S_\alpha]{M} \oplus
				\sat[S_\alpha]{N} = \bigcap_{\alpha < \mu} X_\alpha \oplus Y_\alpha =
				\bigcap_{\alpha < \mu} X_\alpha' \enspace.
			\]
			Thus, by Proposition \ref{prop:pmasplitting:model iteration} it follows
			that $\bigcap_{\alpha < \mu} X_\alpha'$ is the minimal change update
			model of $\upd'$. \qedhere
	\end{enumerate}
\end{proof*}

\section{Splitting Theorem for Dynamic Logic Programs}

\begin{proposition}[Positive Support] \label{prop:dlp support}
	Let $I$ be a dynamic stable model of a dynamic logic program $\prog$. Then
	for every $p \in I$ there exists a rule $r$ from $\dynall{\prog}$ such that
	$H(r) = p$ and $I \ent B(r)$.
\end{proposition}
\begin{proof}
	Follows by definition and the fact that the least model of a definite logic
	program satisfies support.
\end{proof}

\begin{proposition}[Generalisation of Stable Models]
	\label{prop:dlp generalisation}
	Let $\prog$ be a logic program. Then $I$ is a stable model of $\prog$ if and
	only if $I$ is a dynamic stable model of $\prog$.
\end{proposition}
\begin{proof}
	See \cite{Leite2003} for a proof using a slightly different definition of a
	dynamic stable model. The proof for the semantics we use is analogical.
\end{proof}

\begin{proposition} \label{prop:dlp primacy}
	Let $\prog = \an{\prog_i}_{i < n}$ be a dynamic logic program with $n > 0$
	and $I$ be a dynamic stable model of $\prog$. Then $I \ent \prog_{n-1}$.
\end{proposition}
\begin{proof}[Proof (sketch)]
	Follows from the fact that rules in $\prog_{n-1}$ cannot be rejected by
	rules in preceding programs and in case they reject each other, no dynamic
	stable model exists.
\end{proof}

We will use the terms ``splitting set'' and ``splitting sequence'' for
(dynamic) logic programs without defining them formally, assuming they are the
natural specializations of the notions defined for (dynamic) hybrid knowledge
bases.

\begin{definition}
	Let $U = \an{U_\alpha}_{\alpha < \mu}$ be a splitting sequence for a dynamic
	logic program $\prog$. A \emph{solution to $\prog$ with respect to $U$} is a
	sequence $\an{I_\alpha}_{\alpha < \mu}$ of Herbrand interpretations such
	that
	\begin{enumerate}
		\item $I_0$ is a dynamic stable model of $b_{U_0}(\prog)$;
		\item For any ordinal $\alpha$ such that $\alpha + 1 < \mu$, $I_{\alpha +
			1}$ is a dynamic stable model of
			\[
				e_{U_\alpha} \left(
					b_{U_{\alpha+1}}(\prog),
					\textstyle \bigcup_{\eta \leq \alpha} I_\eta
				\right) \enspace;
			\]
		\item For any limit ordinal $\alpha < \mu$, $I_\alpha = \emptyset$;
	\end{enumerate}
\end{definition}

\begin{proposition} \label{prop:dlp splitting}
	Let $U = \an{U_\alpha}_{\alpha < \mu}$ be a splitting sequence for a dynamic
	logic program $\prog$. Then $I$ is a dynamic stable model of $\prog$ if and
	only if $I = \bigcup_{\alpha < \mu} I_\alpha$ for some solution
	$\an{I_\alpha}_{\alpha < \mu}$ to $\prog$ with respect to $U$.
\end{proposition}
\begin{proof}[Proof (sketch)]
	We need to prove that $I$ is a dynamic stable model of $\prog$ if and only
	if $I = \bigcup_{\alpha < \mu} I_\alpha$ where
	\begin{itemize}
		\item $I_0$ is a dynamic stable model of $\prog_0 = b_{U_0}(\prog)$;
		\item for any ordinal $\alpha$ such that $\alpha + 1 < \mu$, $I_{\alpha +
			1}$ is a dynamic stable model of $\prog_{\alpha + 1} =
			e_{U_\alpha}(b_{U_{\alpha + 1}}(\prog), \bigcup_{\eta \leq \alpha}
			I_\alpha)$;
		\item for any limit ordinal $\alpha < \mu$, $I_\alpha = \emptyset$.
	\end{itemize}
	This basically follows from the splitting sequence theorem for logic
	programs \cite{Lifschitz1994a}. However, the unconstrained set of default
	assumptions $\dyndef{\prog, I}$ is troublesome here. In order to overcome
	the problems associated with it, we need to introduce the set
	\[
		\dyndef{\prog, I, S} =
			\Set{ \lpnot p | \preds{p} \subseteq S \land
			(\lnot \exists r \in \dynall)(H(r) = p \land I \ent B(r)) } \enspace,
	\]
	and prove that as long as $S$ includes $\preds{\prog}$, $\dyndef{\prog, I}$
	can be replaced by $\dyndef{\prog, I, S}$ in the definition of a dynamic
	stable model, if accompanied by a suitable restriction in the definition of
	$I'$ as well, i.e. $I' = I \cup \set{\mathit{not}\_p | p \text{ is an atom,
	} \preds{p} \subseteq S \text{ and } p \notin I}$. Now let $S =
	\an{S_\alpha}_{\alpha < \mu}$ be the saturation sequence induced by $U$ and
	let
	\begin{align*}
		\prog[Q]
			&= [ \rho(\prog) \setminus \dynrej{\prog, I} ]
				\cup \dyndef{\prog, I, \lpre} \enspace, \\
		\prog[Q]_\alpha
			&= [ \rho(\prog_\alpha) \setminus \dynrej{\prog_\alpha, I_\alpha} ]
				\cup \dyndef{\prog_\alpha, I_\alpha, S_\alpha} \enspace,
	\end{align*}
	where, depending on the part of the equivalence we are proving, $I_\alpha$
	is either defined to be $\restr[S_\alpha]{I}$, or it is the dynamic stable
	model of $\prog_\alpha$ and $I$ is the union of all $I_\alpha$.

	It then needs to be shown that $\prog[Q]_0 = b_{U_0}(\prog[Q])$ and that for
	every ordinal $\alpha$ such that $\alpha + 1 < \mu$, $\prog[Q]_\alpha =
	e_{U_\alpha}(b_{U_{\alpha + 1}}(\prog[Q]), \bigcup_{\eta \leq \alpha}
	I_\eta)$. Note that the rejection happens the same way in $\prog[Q]$ and in
	$\prog[Q]_\alpha$ because rules with the same predicate symbol in the head
	are always together in the same $\prog[Q]_\alpha$. Once this is done,
	replacing every default literal $\lpnot p$ by a new atom $\mathit{not}\_p$
	will not change anything about it. Thus, by the splitting sequence theorem
	for logic programs,
	\[
		\least{\prog[Q]} = \bigcup_{\alpha < \mu} \least{\prog[Q]_\alpha}
	\]
	and the desired result follows from this by definition of a dynamic stable
	model.
\end{proof}

\section{Proofs of Properties of Hybrid Update Operator}

\label{app:hybrid update operator}

\begin{definition}
	Let $U$ be a splitting sequence for a basic dynamic hybrid knowledge base
	$\kb$. A \emph{solution to $\kb$ with respect to $U$} is a sequence of MKNF
	interpretations $\an{X_\alpha}_{\alpha < \mu}$ such that
	\begin{enumerate}
		\item $X_0$ is a dynamic MKNF model of $b_{U_0}(\kb)$;
		\item For any ordinal $\alpha$ such that $\alpha + 1 < \mu$, $X_{\alpha +
			1}$ is a dynamic MKNF model of
			\[
				e_{U_\alpha} \left(
					b_{U_{\alpha+1}}(\kb),
					\textstyle \bigcap_{\eta \leq \alpha} X_\eta
				\right) \enspace;
			\]
		\item For any limit ordinal $\alpha < \mu$, $X_\alpha = \foint$.
	\end{enumerate}
\end{definition}

\begin{proposition*}
	[Layers of an Updatable Dynamic Hybrid Knowledge Base are Basic]
	{prop:update-enabling implies basic}%
	Let $U$ be an update-enabling splitting sequence for a dynamic hybrid
	knowledge base $\kb$ and $X \in \mint$. Then $b_{U_0}(\kb)$ is a basic
	dynamic hybrid knowledge base and for any ordinal $\alpha$ such that $\alpha
	+ 1 < \mu$, $e_{U_\alpha}(b_{U_{\alpha+1}}(\kb), X)$ is also a basic dynamic
	hybrid knowledge base.
\end{proposition*}%
\vspace{-1em}
\begin{proof}[Proof of Proposition \ref{prop:update-enabling implies basic}]
	\label{proof:prop:update-enabling implies basic}%
	We know $b_{U_0}(\kb)$ is reducible relative to $\emptyset$, which means
	that it either contains no ontology axioms, or all rules inside it are
	facts. Thus, it is basic.

	Now pick some ordinal $\alpha$ such that $\alpha + 1 < \mu$. Since
	$t_{U_\alpha}(b_{U_{\alpha + 1}}(\kb))$ is reducible relative to $U_\alpha$,
	either it contains no ontology axioms and so is basic, or all rules in it
	are positive and their bodies contain only predicate symbols from
	$U_\alpha$. This implies that in $e_{U_\alpha}(b_{U_{\alpha+1}}(\kb), X)$,
	all rules are positive facts, so it is basic.
\end{proof}

\begin{proposition} \label{prop:basic splitting}
	Let $\kb$ be a basic dynamic hybrid knowledge base and $U$ be a splitting
	sequence for $\kb$. Then $M$ is a dynamic MKNF model of $\kb$ if and only if
	$M = \bigcap_{\alpha < \mu} X_\alpha$ for some solution to $\kb$ with
	respect to $U$.
\end{proposition}
\begin{proof}
	Let $\kb = \an{\kb_i}_{i < n}$, where $\kb_i = \an{\ont_i, \prog_i}$.

	First suppose that $\kb$ is $\ont$-based and let $\upd = \an{\upd_i}_{i <
	n}$ be a sequence of first-order theories where $\upd_i = \dlfo{\ont_i} \cup
	\prog_i$. By definition, $M$ is a dynamic MKNF model of $\kb$ if and only if
	$M$ is the minimal change update model of $\upd$. Let $S =
	\an{S_\alpha}_{\alpha < \mu}$ be the saturation sequence induced by $U$.
	Then for any ordinal $\alpha$ such that $\alpha + 1 < \mu$ and any $X \in
	\mint$, the following holds:
	\begin{align*}
		b_{S_0}(\upd) &= b_{U_0}(\upd) \enspace, \\
		b_{S_{\alpha + 1}}(\upd)
			&= b_{U_{\alpha + 1} \setminus U_\alpha}(\upd)
			= t_{U_\alpha}(b_{U_\alpha + 1}(\upd))
			= e_{U_\alpha}(b_{U_{\alpha + 1}}(\upd), X) \enspace.
	\end{align*}
	The rest follows by Propositions \ref{prop:pmasplitting:1} and
	\ref{prop:pmasplitting:2}.

	Now suppose that $\kb$ is $\prog$-based and let $\prog = \an{\prog_i}_{i <
	n}$ be a dynamic logic program. If $M$ is a dynamic MKNF model
	of $\kb$, then $M = \set{J \in \foint | I \subseteq J}$ for some
	dynamic stable model $I$ of $\prog$. By Proposition \ref{prop:dlp
	splitting}, this implies that $I = \bigcup_{\alpha < \mu} I_\alpha$, where
	\begin{itemize}
		\item $I_0$ is a dynamic stable model of $b_{U_0}(\prog)$;
		\item for every ordinal $\alpha$ such that $\alpha + 1 < \mu$, $I_{\alpha
			+ 1}$ is a dynamic stable model of $e_{U_\alpha}(b_{U_{\alpha +
			1}}(\prog), \bigcup_{\eta \leq \alpha} I_\eta)$;
		\item for every limit ordinal $\alpha$, $I_\alpha = \emptyset$.
	\end{itemize}
	Put $M_\alpha = \set{J \in \foint | I_\alpha \subseteq J}$ for every $\alpha
	< \mu$. Also, for any $\alpha < \mu$,
	\[
		\bigcap_{\eta \leq \alpha} M_\eta = \bigcap_{\eta \leq \alpha} \Set{J \in
		\foint | I_\eta \subseteq J} = \Set{J \in \foint | \bigcup_{\eta \leq
		\alpha} I_\eta \subseteq J}
	\]
	and it can be verified easily that for every literal $L$,
	\[
		\bigcup_{\eta \leq \alpha} I_\eta \ent \pi(L)
		\enspace \text{if and only if} \enspace
		\Set{J \in \foint | \bigcup_{\eta \leq \alpha} I_\eta \subseteq J} \ent \pi(L)
		\enspace.
	\]
	Thus, the following now follows by the definition of a dynamic MKNF model of
	a basic dynamic hybrid knowledge base:
	\begin{itemize}
		\item $M_0$ is a dynamic MKNF model of $b_{U_0}(\kb)$;
		\item for every ordinal $\alpha$ such that $\alpha + 1 < \mu$, $M_{\alpha
			+ 1}$ is a dynamic MKNF model of $e_{U_\alpha}(b_{U_{\alpha +
			1}}(\kb), \bigcap_{\eta \leq \alpha} M_\eta)$;
		\item for every limit ordinal $\alpha$, $M_\alpha = \foint$.
	\end{itemize}
	In other words, $\an{M_\alpha}_{\alpha < \mu}$ is a solution to $\kb$ with
	respect to $U$ and
	\begin{align*}
		M &= \Set{J \in \foint | I \subseteq J}
			= \Set{J \in \foint | \bigcup_{\alpha < \mu} I_\alpha \subseteq J} \\
			&= \bigcap_{\alpha < \mu} \set{J \in \foint | I_\alpha \subseteq J}
			= \bigcap_{\alpha < \mu} M_\alpha \enspace.
	\end{align*}

	The converse implication can be proved analogically by reversing the steps
	in the above proof.
\end{proof}

\begin{corollary} \label{cor:basic seqsat}
	Let $\kb$ be a basic dynamic hybrid knowledge base, $U$ be a splitting
	sequence for $\kb$ and $S$ be the saturation sequence induced by $U$. If $M$
	is a dynamic MKNF model of $\kb$, then $M$ is sequence-saturated relative to
	$S$.
\end{corollary}
\begin{proof}
	Follows by Proposition \ref{prop:basic splitting}, definition of a solution
	and by Propositions \ref{prop:seqsat:int sigma} and
	\ref{prop:seqsat:equivalence}.
\end{proof}

\begin{proposition} \label{prop:basic saturated}
	Let $U$ be a set of predicate symbols and $\kb$ be a basic dynamic hybrid
	knowledge base such that $\preds{\kb}$ is included in $U$. Then every
	dynamic MKNF model of $\kb$ is saturated relative to $U$.
\end{proposition}
\begin{proof}[Proof (sketch)]
	If $\kb$ is $\ont$-based, then this follows from Proposition
	\ref{prop:pmasplitting:seq pma sat}. If $\kb$ is $\prog$-based, then $M$ is
	a dynamic MKNF model of $\prog$ only if $M = \set{J \in \in | I \subseteq
	J}$ where $I$, by Proposition \ref{prop:dlp support}, contains only atoms
	with predicate symbols from $U$. This implies that whenever some $\restr{J}$
	belongs to $\restr{M}$, $I$ is a subset of $\restr{J}$ which is a subset of
	$J$, and so $J$ belongs to $M$. Thus, $M$ is saturated relative to $U$.
\end{proof}

\begin{lemma}
	\label{lemma:bottom repeated}
	Let $\kb$ be a (dynamic) hybrid knowledge base and $U, V$ be sets of
	predicate symbols. Then,
	\[
		b_U(b_V(\kb)) = b_{U \cap V}(\kb) \enspace.
	\]
\end{lemma}
\begin{proof}
	Follows directly by definition.
\end{proof}

\begin{lemma}
	\label{lemma:bottom reduct}
	Let $\kb$ be a (dynamic) hybrid knowledge base, $X \in \mint$, $U$ be a set
	of predicate symbls and $V$ be a splitting set for $\kb$. Then,
	\[
		e_U(b_V(\kb), X) = b_V(e_U(\kb, X)) \enspace.
	\]
\end{lemma}
\begin{proof}[Proof (sketch)]
	Since $V$ is a splitting set for $\kb$, all rules from $\kb$ whose head atom
	has a predicate symbol from $V$ must also have all body literals with
	predicate symbols from $V$. Thus, body atoms discarded in $e_U(\kb, X)$ for
	rules with a head predicate symbol from $V$ cannot be a reason for the rule
	being thrown away by application of $b_V(\cdot)$.
\end{proof}

\begin{lemma} \label{lemma:coinciding intersection}
	Let $U, V$ be sets of predicate symbols and $X, Y$ be MKNF interpretations
	such that $X$ is saturated relative to $U$, $Y$ is saturated relative to
	$V$ and $X$ coincides with $Y$ on $U \cap V$. Then,
	\[
		\restr{(X \cap Y)} = \restr{X}
		\enspace \text{and} \enspace
		\restr[V]{(X \cap Y)} = \restr[V]{Y}
		\enspace.
	\]
\end{lemma}
\begin{proof}[Proof (sketch)]
	The left to right inclusions are obvious. If $\restr{I}$ belongs to
	$\restr{X}$, then we can construct an interpretation $I' = \restr{I} \cup
	\restr{J}$ where $J$ is some interpretation from $Y$ that coincides with $I$
	on $U \cap V$. Because of the assumptions, $I'$ will belong to both $X$ and
	$Y$. Thus, $\restr{I} = \restr{I'} \in \restr{(X \cap Y)}$. The case with
	$V$ is symmetric.
\end{proof}

\begin{lemma} \label{lemma:reduct repeated}
	Let $U, V$ be splitting sets for a (dynamic) hybrid knowledge base $\kb$ and
	$X, Y$ be MKNF interpretations such that $X$ is saturated relative to $U$,
	$Y$ is saturated relative to $V$ and $X$ coincides with $Y$ on $U \cap V$.
	Then,
	\[
		e_U(e_V(\kb, Y), X) = e_{U \cup V}(\kb, X \cap Y) \enspace.
	\]
\end{lemma}
\begin{proof}[Proof (sketch)]
	For the ontology part of the hybrid knowledge base this holds because
	\begin{align*}
		e_U(e_V(\ont, Y), X)
			&= t_U(t_V(\ont))
			= b_{\lpre \setminus U}(b_{\lpre \setminus V}(\ont))
			= b_{(\lpre \setminus U) \cap (\lpre \setminus V)}(\ont) \\
			&= b_{\lpre \setminus (U \cup V)}(\ont)
			= t_{U \cup V}(\ont)
			= e_{U \cup V}(\ont, X \cap Y) \enspace.
	\end{align*}
	For the rule part, we additionally need to observe that on the right hand side, all
	body atoms with predicate symbol from $U$ are interpreted under $X$ and all
	body atoms with predicate symbol from $V$ are interpreted under $Y$, and
	Lemma \ref{lemma:coinciding intersection} guarantees that $X \cap Y$
	coincides with $X$ on $U$ and with $Y$ on $V$.
\end{proof}

\begin{lemma} \label{lemma:reduct context}
	Let $U, V$ be sets of predicate symbols, $X \in \mint$ and let $\kb$ be a
	dynamic hybrid knowledge base such that $\preds{\kb}$ is a a subset of $V$.
	Then,
	\[
		e_U(\kb, X) = e_U(\kb, \sat[V]{X})
	\]
\end{lemma}
\begin{proof}[Proof (sketch)]
	The second argument of $e_U(\cdot)$ is used only to interpret body atoms of
	rules from $\kb$, which by the assumption contain only predicate symbols
	from $V$, and by Proposition \ref{prop:sigma:restr}, $\restr[V]{\sat[V]{X}}
	= \restr[V]{X}$.
\end{proof}

\begin{lemma} \label{lemma:sigma seq sat}
	Let $S$ be a saturation sequence, $M \in \mint$ be sequence-saturated
	relative to $S$ and $U$ be a set of predicate symbols. Then $\sat{M}$ is
	also sequence-saturated relative to $S$.
\end{lemma}
\begin{proof}[Proof (sketch)]
	Let $\restr[S_\alpha]{I} \in \restr[S_\alpha]{\sat{M}}$ for all $\alpha$.
	Then there is some $J_\alpha \in \sat{M}$ such that $\restr[S_\alpha]{I} =
	\restr[S_\alpha]{J_\alpha}$ and some $K_\alpha \in M$ such that
	$\restr{J_\alpha} = \restr{K_\alpha}$. From these $K_\alpha$'s we can
	construct a $K \in M$ such that $\restr[S_\alpha \cap U]{K} =
	\restr[S_\alpha \cap U]{I}$, from which it follows that $\restr{K} =
	\restr{I}$. Hence, $I \in \sat{M}$.
\end{proof}

\begin{lemma} \label{lemma:basic ambiguity}
	Let $\kb = \an{\kb_i}_{i < n}$, where $\kb_i = \an{\ont_i, \prog_i}$, be a
	dynamic hybrid knowledge base that is both $\ont$-based and $\prog$-based.
	Then the minimal change update model of $\an{\dlfo{\ont_i} \cup \prog_i}_{i
	< n}$ with the unique dynamic stable model of $\an{\prog_i}_{i < n}$.
\end{lemma}
\begin{proof}
	This can be seen easily since in this case $\ont_i$ is empty and $\prog_i$
	contains only positive facts, so that the dynamic stable model of
	$\an{\prog_i}_{i < n}$ coincides with the set of all atoms appearing as
	heads of rules in the programs, and this also coincides with its minimal
	change update model.
\end{proof}

\begin{lemma} \label{lemma:splitting saturation sigma}
	Let $U = \an{U_\alpha}_{\alpha < \mu}$ be a splitting sequence, $S$ be the
	saturation sequence induced by $U$ and $M \in \mint$ be sequence-saturated
	relative to $S$.  Then for any ordinal $\alpha < \mu$ the following holds:
	\[
		\bigcap_{\eta \leq \alpha} \sat[S_\eta]{M} = \satpar[\bigcup_{\eta \leq
		\alpha} S_\eta]{M} = \sat[U_\alpha]{M} \enspace.
	\]
\end{lemma}
\begin{proof}
	It can be shown by induction that $U_\alpha = \bigcup_{\eta \leq \alpha}
	S_\eta$. If $M$ is empty, then the lemma trivially follows. Suppose that $M$
	contains some interpretation $J_0$. Let $I \in \bigcap_{\eta \leq \alpha}
	\sat[S_\eta]{M}$. Then for every ordinal $\eta \leq \alpha$ there must exist
	some interpretation $I_\eta$ from $M$ such that $\restr[S_\eta]{I_\eta} =
	\restr[S_\eta]{I}$. Let $J = \bigcup_{\eta \leq \alpha} \restr[S_\eta]{I}
	\cup \restrc[U_\alpha]{J_0}$. It is not difficult to see that $J$ belongs to
	$M$ due to the fact that $M$ is sequence-saturated relative to $U$. Also,
	$J$ coincides with $I$ on $U_\alpha$. Thus, $I$ belongs to
	$\restr[U_\alpha]{M}$.

	As for the other inclusion, if $I$ belongs to $\sat[U_\alpha]{M}$, then
	there is some $J \in M$ such that $J$ coincides with $I$ on $U_\alpha$. But
	then $J$ also coincides with $I$ on $S_\eta$ for every $\eta \leq \alpha$.
	Thus, $I$ belongs to $\bigcap_{\eta \leq \alpha} \sat[S_\eta]{M}$.
\end{proof}

\begin{proposition*}
	[Solution Independence]
	{prop:solution independence}%
	Let $U, V$ be update-enabling splitting sequences for a dynamic hybrid knowledge base
	$\kb$. Then $M$ is a dynamic MKNF model of $\kb$ with respect to $U$ if and only if
	$M$ is a dynamic MKNF model of $\kb$ with respect to $V$.
\end{proposition*}
\begin{proof}[Proof of Proposition \ref{prop:solution independence}]
	\label{proof:prop:solution independence}%
	Suppose $M$ is a dynamic MKNF model of $\kb$ with respect to $U =
	\an{U_\alpha}_{\alpha < \mu}$. Then $M = \bigcap_{\alpha < \mu} X_\alpha$
	for some solution $\an{X_\alpha}_{\alpha < \mu}$ to $\kb$ with respect to
	$U$. This means that:
	\begin{itemize}
		\item $X_0$ is a dynamic MKNF model of $\kb_0 = b_{U_0}(\kb)$;
		\item for any ordinal $\alpha$ such that $\alpha + 1 < \mu$, $X_{\alpha +
			1}$ is a dynamic MKNF model of $\kb_{\alpha + 1} =
			e_{U_\alpha}(b_{U_{\alpha + 1}}(\kb), \bigcap_{\eta \leq \alpha}
			X_\eta)$;
		\item for any limit ordinal $\alpha < \mu$, $X_\alpha = \foint$ and thus it is
			a dynamic MKNF model of $\kb_\alpha = \an{~}$.
	\end{itemize}
	We also know from Proposition \ref{prop:update-enabling implies basic} that
	$\kb_\alpha$ is a basic dynamic hybrid knowledge base for every ordinal
	$\alpha < \mu$. Let $S = \an{S_\alpha}_{\alpha < \mu}$ be the saturation
	sequence induced by $U$. We know that for every $\alpha < \mu$, $\kb_\alpha$
	contains only predicate symbols from $S_\alpha$, so by Proposition
	\ref{prop:basic saturated}, $X_\alpha$ is saturated relative to
	$S_\alpha$. Thus, by Proposition \ref{prop:seqsat:int sigma},
	\[
		X_\alpha = \sat[S_\alpha]{M} \enspace.
	\]
	Moreover, by Lemma \ref{lemma:splitting saturation sigma}, $\bigcap_{\eta
	\leq \alpha} X_\eta = \bigcap_{\eta \leq \alpha} \sat[S_\alpha]{M} =
	\sat[U_\alpha]{M}$, and so
	\[
		\kb_{\alpha + 1} =
			e_{U_\alpha}(b_{U_{\alpha + 1}}(\kb), \sat[U_\alpha]{M}) \enspace.
	\]

	Now pick some arbitrary but fixed $\alpha < \mu$ and suppose that $V =
	\an{V_\beta}_{\beta < \nu}$. Since $V$ is a splitting sequence for $\kb$, it
	is also a splitting sequence for $\kb_\alpha$. Thus, by Proposition
	\ref{prop:basic splitting} we know that $X_\alpha = \bigcap_{\beta < \nu}
	Y_{\alpha, \beta}$ for some solution $\an{Y_{\alpha, \beta}}_{\beta < \nu}$
	to $\kb_\alpha$ with respect to $V$. This means that:
	\begin{itemize}
		\item $Y_{\alpha, 0}$ is a dynamic MKNF model of  $\kb_{\alpha, 0} =
			b_{V_0}(\kb_\alpha)$;
		\item for any ordinal $\beta$ such that $\beta + 1 < \nu$, $Y_{\alpha,
			\beta + 1}$ is a dynamic MKNF model of $\kb_{\alpha, \beta + 1} =
			e_{V_\beta}(b_{V_{\beta + 1}}(\kb_\alpha), \bigcap_{\eta \leq \beta}
			Y_{\alpha, \eta})$;
		\item for any limit ordinal $\beta < \nu$, $Y_{\alpha, \beta} = \foint$ and
			thus it is a dynamic MKNF model of $\kb_{\alpha, \beta} = \an{~}$.
	\end{itemize}
	Since $\kb_\alpha$ is a basic dynamic hybrid knowledge base, $\kb_{\alpha,
	\beta}$ must also be a basic dynamic hybrid knowledge base. Let $T =
	\an{T_\beta}_{\beta < \nu}$ be the saturation sequence induced by $V$. We
	know that for every $\beta < \nu$, $\kb_{\alpha, \beta}$ contains only
	predicate symbols from $T_\beta$, so by Proposition \ref{prop:basic
	saturated}, $Y_{\alpha, \beta}$ is saturated relative to $T_\beta$.  Thus,
	by Propositions \ref{prop:seqsat:int sigma} and \ref{prop:sigma:repeated},
	\[
		Y_{\alpha, \beta}
			= \sat[T_\beta]{X_\alpha}
			= \sat[T_\beta]{\sat[S_\alpha]{M}} = \sat[S_\alpha \cap T_\beta]{M}
			\enspace.
	\]

	Let the sequence of knowledge bases $\kb' = \an{\kb_\beta'}_{\beta < \nu}$
	be defined as follows:
	\begin{itemize}
		\item $\kb_0' = b_{V_0}(\kb)$;
		\item for any ordinal $\beta$ such that $\beta + 1 < \nu$, $\kb_{\beta
			+ 1}' = e_{V_\beta}(b_{V_{\beta + 1}}(\kb), \sat[V_\beta]{M})$;
		\item for any limit ordinal $\beta < \nu$, $\kb_\beta' = \an{~}$.
	\end{itemize}
	In the following we prove that for any ordinal $\beta < \nu$ and any ordinal
	$\alpha$ such that $\alpha + 1 < \mu$,
	\begin{align}
		\kb_{0, \beta} &= b_{V_0}(\kb_\beta')
			\label{eq:proof:solution independence:1} \\
		\kb_{\alpha + 1, \beta}
			&= e_{U_\alpha}(b_{U_{\alpha + 1}}(\kb_{\beta}'), \sat[U_\alpha]{M})
			\label{eq:proof:solution independence:2} 
	\end{align}
	Suppose first that $\beta = 0$. Then we can use Lemma \ref{lemma:bottom
	repeated} to obtain
	\[
		\kb_{0, 0} = b_{U_0}(\kb_0) = b_{V_0}(b_{U_0}(\kb))
			= b_{U_0 \cap V_0}(\kb) = b_{U_0}(b_{V_0}(\kb)) = b_{U_0}(\kb_0')
	\]
	and for any ordinal $\alpha$ such that $\alpha + 1 < \mu$ we can apply
	Lemmas \ref{lemma:bottom repeated} and \ref{lemma:bottom reduct}, achieving
	the following result:
	\begin{align*}
		\kb_{\alpha + 1, 0} &= b_{V_0}(\kb_{\alpha + 1}) \\
			&= b_{V_0}(e_{U_\alpha}(b_{U_{\alpha + 1}}(\kb), \sat[U_\alpha]{M})) \\
			&= e_{U_\alpha}(b_{V_0}(b_{U_{\alpha + 1}}(\kb)), \sat[U_\alpha]{M}) \\
			&= e_{U_\alpha}(b_{U_{\alpha + 1} \cap V_0}(\kb), \sat[U_\alpha]{M}) \\
			&= e_{U_\alpha}(b_{U_{\alpha + 1}}(b_{V_0}(\kb)), \sat[U_\alpha]{M}) \\
			&= e_{U_\alpha}(b_{U_{\alpha + 1}}(\kb_0'), \sat[U_\alpha]{M})
			\enspace.
	\end{align*}
	Now suppose that $\beta$ is an ordinal such that $\beta + 1 < \nu$. Using
	Lemmas \ref{lemma:bottom repeated} and \ref{lemma:bottom reduct} we obtain:
	\begin{align*}
		\kb_{0, \beta + 1}
			&= e_{V_\beta}(b_{V_{\beta + 1}}(\kb_0), \sat[V_\beta]{M}) \\
			&= e_{V_\beta}(b_{V_{\beta + 1}}(b_{U_0}(\kb)), \sat[V_\beta]{M}) \\
			&= e_{V_\beta}(b_{U_0 \cap V_{\beta + 1}}(\kb), \sat[V_\beta]{M}) \\
			&= e_{V_\beta}(b_{U_0}(b_{V_{\beta + 1}}(\kb)), \sat[V_\beta]{M}) \\
			&= b_{U_0}(e_{V_\beta}(b_{V_{\beta + 1}}(\kb), \sat[V_\beta]{M})) \\
			&= b_{U_0}(\kb_{\beta + 1}') \enspace.
	\end{align*}
	Finally, for any ordinal $\alpha$ such that $\alpha + 1 < \mu$, Lemmas
	\ref{lemma:bottom repeated}, \ref{lemma:bottom reduct} and \ref{lemma:reduct
	repeated} imply the following:
	\begin{align*}
		\kb_{\alpha + 1, \beta + 1}
			&= e_{V_\beta}(b_{V_{\beta + 1}}(\kb_{\alpha + 1}), \sat[V_\beta]{M}) \\
			&= e_{V_\beta}(
				b_{V_{\beta + 1}}(
					e_{U_\alpha}(b_{U_{\alpha + 1}}(\kb), \sat[U_\alpha]{M})
				),
				\sat[V_\beta]{M}
			) \\
			&= e_{V_\beta}(
				e_{U_\alpha}(
					b_{V_{\beta + 1}}(b_{U_{\alpha + 1}}(\kb)),
					\sat[U_\alpha]{M}
				),
				\sat[V_\beta]{M}
			) \\
			&= e_{U_\alpha \cup V_\beta}(
				b_{U_{\alpha + 1} \cap V_{\beta + 1}}(\kb),
				\sat[U_\alpha]{M} \cap \sat[V_\beta]{M}
			) \\
			&= e_{U_\alpha}(
				e_{V_\beta}(
					b_{U_{\alpha + 1}}(b_{V_{\beta + 1}}(\kb)),
					\sat[V_\beta]{M}
				),
				\sat[U_\alpha]{M}
			) \\
			&= e_{U_\alpha}(
				b_{U_{\alpha + 1}}(
					e_{V_\beta}(b_{V_{\beta + 1}}(\kb), \sat[V_\beta]{M})
				),
				\sat[U_\alpha]{M}
			) \\
			&= e_{U_\alpha}(
				b_{U_{\alpha + 1}}(
					\kb_{\beta + 1}'
				),
				\sat[U_\alpha]{M}
			) \enspace.
	\end{align*}

	Now since $\kb_{\beta}'$ is saturated relative to $T_\beta$, we can use
	Lemma \ref{lemma:reduct context} to replace $\sat[U_\alpha]{M}$ in
	\eqref{eq:proof:solution independence:2} by $\sat[U_\alpha \cap
	T_\beta]{M}$. Furthermore, by consecutively using Proposition
	\ref{prop:sigma:repeated}, Lemmas \ref{lemma:sigma seq sat} and
	\ref{lemma:splitting saturation sigma} and Proposition
	\ref{prop:sigma:repeated} again, we can see that
	\begin{align*}
		\sat[U_\alpha \cap T_\beta]{M}
			&= \sat[U_\alpha]{\sat[T_\beta]{M}}
			= \bigcap_{\eta \leq \alpha} \sat[S_\eta]{\sat[T_\beta]{M}} \\
			&= \bigcap_{\eta \leq \alpha} \sat[S_\eta \cap T_\beta]{M}
			= \bigcap_{\eta \leq \alpha} Y_{\eta, \beta} \enspace.
	\end{align*}
	Hence, \eqref{eq:proof:solution independence:2} can be rewritten as:
	\[
		\kb_{\alpha + 1, \beta}
			= e_{U_\alpha} \left(
				b_{U_{\alpha + 1}} (\kb_{\beta}'),
				\bigcap_{\eta \leq \alpha} Y_{\eta, \beta}
			\right)
	\]
	We can now use Proposition \ref{prop:basic splitting} and conclude that
	\[
		\bigcap_{\alpha < \mu} Y_{\alpha, \beta}
			= \bigcap_{\alpha < \mu} \sat[S_\alpha \cap T_\beta]{M}
			= \bigcap_{\alpha < \mu} \sat[S_\alpha]{\sat[T_\beta]{M}}
			= \sat[T_\beta]{M}
	\]
	is a dynamic MKNF model of $\kb_\beta'$. One of the last steps in the
	proof is to show that $M$ is sequence-saturated relative to $T$. We know
	from Corrolary \ref{cor:basic seqsat} that $X_\alpha$ is sequence-saturated
	relative to $T$, so we obtain the following:
	\begin{align*}
		\bigcap_{\beta < \nu} \sat[T_\beta]{M}
			&= \bigcap_{\beta < \nu} \bigcap_{\alpha < \mu}
				\sat[S_\alpha]{\sat[T_\beta]{M}} \\
			&= \bigcap_{\alpha < \mu} \bigcap_{\beta < \nu}
				\sat[T_\beta]{\sat[S_\alpha]{M}} \\
			&= \bigcap_{\alpha < \mu} \bigcap_{\beta < \nu}
				\sat[T_\beta]{X_\alpha} \\
			&= \bigcap_{\alpha < \mu} X_\alpha
			= M \enspace,
	\end{align*}
	which implies that $M$ is sequence-saturated relative to $T$. Thus, for any
	$\beta < \nu$, Lemma \ref{lemma:splitting saturation sigma} implies that
	\[
		\sat[V_\beta]{M} = \bigcap_{\eta \leq \beta} \sat[T_\eta]{M} \enspace.
	\]
	To sum up, define the sequence of interpretations $Z = \an{Z_\beta}_{\beta <
	\nu}$ by $Z_\beta = \sat[T_\beta]{M}$. We know the following:
	\begin{itemize}
		\item $Z_0 = \sat[T_0]{M}$ is a dynamic MKNF model of $\kb_0' = b_{V_0}(\kb)$;
		
		\item for any ordinal $\beta$ such that $\beta + 1 < \nu$, $Z_{\beta + 1}
			= \sat[T_{\beta + 1}]{M}$ is a dynamic MKNF model of $\kb_{\beta + 1}' =
			e_{V_\beta}(b_{V_{\beta + 1}}(\kb), \bigcap_{\eta \leq \beta}
			\sat[T_\eta]{M}) = e_{V_\beta}(b_{V_{\beta + 1}}(\kb), \bigcap_{\eta
			\leq \beta} Z_\eta)$;

		\item for any limit ordinal $\beta < \nu$, put $Z_\beta = \sat[T_\beta]{M}
			= \sat[\emptyset]{M} = \foint$.
	\end{itemize}
	Thus, $Z$ is a solution to $\kb$ with respect to $V$. Moreover, since $M$ is
	sequence-saturated relative to $T$, it follows by Proposition
	\ref{prop:seqsat:equivalence} that
	\[
		M = \bigcap_{\beta < \nu} \sat[T_\beta]{M}
			= \bigcap_{\beta < \nu} Z_\beta \enspace.
	\]
	So $M$ is a dynamic MKNF model of $\kb$ with respect to $V$.
	
	Proof of the converse implication is symmetric.
\end{proof}

\begin{corollary*}
	[Compatibility with Def.~\ref{def:basic dynamic hybrid knowledge base}]
	{cor:compatibility with basic}%
	Let $\kb$ be a basic dynamic hybrid knowledge base and $U$ be a splitting
	sequence for $\kb$. Then $M$ is a dynamic	MKNF model of $\kb$ if and only if
	$M$ is a dynamic MKNF model of $\kb$ with respect to $U$.
\end{corollary*}
\begin{proof}[Proof of Corollary \ref{cor:compatibility with basic}]
	\label{proof:cor:compatibility with basic}%
	Since $\kb$ is basic, $\an{\lpre}$ is an update-enabling sequence for $\kb$.
	Also, by the definition, $M$ is a dynamic MKNF model of $\kb$ if and only if
	$\kb$ is a dynamic MKNF model of $\kb$ with respect to $\an{\lpre}$.
	Finally, by Proposition \ref{prop:solution independence} this holds if and
	only if $M$ is a dynamic MKNF model of $\kb$ with respect to $U$.
\end{proof}

\begin{theorem*}
	[Generalisation of MKNF Models]
	{thm:generalisation of mknf}%
	Let $\kb$ be an updatable hybrid knowledge base and $M$ be an MKNF
	interpretation. Then $M$ is a dynamic MKNF model of $\an{\kb}$ if and only
	if $M$ is an MKNF model of $\kb$.
\end{theorem*}
\begin{proof}[Proof of Theorem \ref{thm:generalisation of mknf}]
	\label{proof:thm:generalisation of mknf}%
	This follows by Theorem \ref{thm:seq} and Propositions \ref{prop:dlp
	generalisation} and \ref{prop:pmasplitting:generalisation}.
\end{proof}

\begin{theorem*}
	[Generalisation of Minimal Change Update Semantics]
	{thm:generalisation of pma}%
	Let $\an{\kb_i}_{i < n}$, where $\kb_i = \an{\ont_i, \prog_i}$, be a dynamic
	hybrid knowledge base such that $\prog_i$ is empty for all $i < n$. Then $M$
	is a dynamic MKNF model of $\an{\kb_i}_{i < n}$ if and only if $M$ is the
	minimal change update model of $\an{\dlfo{\ont_i}}_{i < n}$.
\end{theorem*}
\begin{proof}[Proof of Theorem \ref{thm:generalisation of pma}]
	\label{proof:thm:generalisation of pma}%
	This follows by Corollary \ref{cor:compatibility with basic} and Lemma
	\ref{lemma:basic ambiguity}.
\end{proof}

\begin{theorem*}
	[Generalisation of Dynamic Stable Model Semantics]
	{thm:generalisation of dlp}%
	Let $\kb = \an{\kb_i}_{i < n}$, where $\kb_i = \an{\ont_i, \prog_i}$, be a
	dynamic hybrid knowledge base such that $\ont_i$ is empty for all $i < n$.
	Then $M$ is a dynamic MKNF model of $\kb$ if and only if $M = \set{J \in
	\foint | I \subseteq J}$ for some dynamic stable model $I$ of
	$\an{\prog_i}_{i < n}$.
\end{theorem*}
\begin{proof}[Proof of Theorem \ref{thm:generalisation of dlp}]
	\label{proof:thm:generalisation of dlp}%
	This follows by Corollary \ref{cor:compatibility with basic} and Lemma
	\ref{lemma:basic ambiguity}.
\end{proof}

\begin{theorem*}
	[Principle of Primacy of New Information]
	{thm:primacy of new information}%
	Let $\kb = \an{\kb_i}_{i < n}$ be an updatable dynamic hybrid knowledge base
	with $n > 0$ and $M$ be a dynamic MKNF model of $\kb$. Then $M \ent
	\pi(\kb_{n-1})$.
\end{theorem*}
\begin{proof}[Sketch of proof of Theorem \ref{thm:primacy of new information}]
	\label{proof:thm:primacy of new information}%
	If $M$ is a dynamic MKNF model of $\kb$, then for some update-enabling
	sequence $U = \an{U_\alpha}_{\alpha < \mu}$, $M = \bigcap_{\alpha < \mu}
	X_\alpha$ for some solution to $\kb$ with respect to $U$. This means that
	\begin{itemize}
		\item $X_0$ is a dynamic MKNF model of $b_{U_0}(\kb)$;
		\item for any ordinal $\alpha$ such that $\alpha + 1 < \mu$, $X_{\alpha +
			1}$ a dynamic MKNF model of $e_{U_\alpha}(b_{U_{\alpha + 1}}(\kb),
			\bigcap_{\eta \leq \alpha} X_\eta)$;
		\item for any limit ordinal $\alpha$, $X_\alpha = \foint$.
	\end{itemize}
	Let $S = \an{S_\alpha}_{\alpha < \mu}$ be the saturation sequence induced by
	$U$. It follows from Propositions \ref{prop:pmasplitting:primacy},
	\ref{prop:dlp primacy} and \ref{prop:basic saturated} that
	\begin{itemize}
		\item $X_0$ is saturated relative to $S_0$ and $X_0 \ent b_{U_0}(\kb_{n-1})$;
		\item for any ordinal $\alpha$ such that $\alpha + 1 < \mu$, $X_{\alpha +
			1}$ is saturated relative to $S_{\alpha + 1}$ and $X_{\alpha + 1} \ent
			e_{U_\alpha}(b_{U_{\alpha + 1}}(\kb_{n-1}), \bigcap_{\eta \leq \alpha}
			X_\eta)$;
		\item for any limit ordinal $\alpha$, $X_\alpha = \foint$ is saturated
			relative to $S_\alpha = \emptyset$.
	\end{itemize}

	Thus, by Proposition \ref{prop:seqsat:equivalence}, $M$ is
	sequence-saturated relative to $S$, by Proposition \ref{prop:seqsat:int
	sigma}, $X_\alpha = \sat[S_\alpha]{M}$, and by Lemma \ref{lemma:splitting
	saturation sigma}, $\bigcap_{\eta \leq \alpha} X_\eta = \sat[U_\alpha]{M}$.

	Now let $\phi$ be some formula from $\pi(\kb_{n-1})$. If $\phi$ is of the
	form $\mk \psi$ where $\psi$ is a first-order formula, then there must exist
	a unique set $S_\alpha$ that includes $\preds{\phi}$. Due to the above
	considerations, we can then conclude that $X_\alpha \ent \phi$. Furthermore,
	\[
		X_\alpha \ent \phi \mlequiv \sat[S_\alpha]{M} \ent \phi \mlequiv
		\restr[S_\alpha]{\sat[S_\alpha]{M}} \ent \phi \mlequiv \restr[S_\alpha]{M} \ent \phi
		\mlequiv M \ent \phi \enspace.
	\]
	On the other hand, if $\phi = \pi(r)$ for some rule $r$, then the there
	exists a unique nonlimit ordinal $\alpha$ such that $\preds{H(r)} \subseteq
	S_\alpha$ and the body of the rule can be divided in two parts, $B_1$ and
	$B_2$, such that $\preds{B_1} \subseteq S_\alpha$ and $B_2 = \emptyset$ if
	$\alpha = 0$ and $\preds{B_2} \subseteq U_{\alpha - 1}$ if $\alpha > 0$. The
	case when $\alpha = 0$ is can be derived from the case when $\alpha > 0$, so
	in the following we only consider the latter case. We have:
	\begin{align*}
		M \ent B_2 &\mlequiv \sat[U_{\alpha - 1}]{M} \ent B_2 \\
		M \ent B_1 &\mlequiv \sat[S_\alpha]{M} \ent B_1 \\
		M \ent H(r) &\mlequiv \sat[S_\alpha]{M} \ent H(r)
	\end{align*}
	Hence, if $M \nent B_2$, then $M \ent r$ and we are finished. On the other
	hand, if $M \ent B_2$, then there is a rule $r'$ in
	$e_{U_\alpha - 1}(b_{U_{\alpha}}(\kb), \sat[U_\alpha - 1]{M})$ such that $H(r')
	= H(r)$ and $B(r') = B_1$. If $M \ent B_1$, then since $X_\alpha =
	\sat[U_\alpha - 1]{M} \ent B_1$ and since $X_\alpha \ent e_{U_\alpha -
	1}(b_{U_{\alpha}}(\kb), \sat[U_\alpha - 1]{M})$, we obtain that $X_\alpha
	\ent H(r')$. As a consequence, $M \ent H(r)$, so that $M \ent r$.
\end{proof}

\end{extended}

\end{document}